\documentclass[sigconf,nonacm]{acmart} 

\usepackage{graphicx}            
\usepackage{subcaption}          

\usepackage{algorithm}
\usepackage[noend]{algpseudocode}
\algrenewcommand\algorithmicrequire{\textbf{Input:}}
\algrenewcommand\algorithmicensure{\textbf{Output:}}

\algblockdefx{Switch}{EndSwitch}[1]{\textbf{switch} #1 \textbf{do}}{}
\algblockdefx{Case}{EndCase}[1]{\textbf{case} #1:}{}
\algtext*{EndSwitch}
\algtext*{EndCase}

\usepackage{bbm}
\usepackage{array}
\usepackage{booktabs} 
\usepackage{stfloats}
\usepackage{enumitem}

\newcolumntype{C}[1]{>{\centering\arraybackslash}m{#1}}
\newcolumntype{L}[1]{>{\raggedright\arraybackslash}p{#1}}

\AtBeginDocument{%
  }

\usepackage{ifthen}
\newboolean{fullpaper}
\setboolean{fullpaper}{true} 
\providecommand{\arxivnotice}{}

\ifthenelse{\boolean{fullpaper}}{
    \settopmatter{printacmref=false, printfolios=true}
    \setcopyright{none}
    \renewcommand\footnotetextcopyrightpermission[1]{}
    \copyrightyear{}
    \acmYear{}
    \acmConference[]{}{}{}
    \acmBooktitle{}
    \acmPrice{}
    \acmDOI{}
    \acmISBN{}
    \renewcommand{\arxivnotice}{%
        \noindent\textbf{Full version.} This document contains additional results and appendices omitted from the KDD '26 version due to space limits.
        The Version of Record will appear in the ACM Digital Library.
        DOI: \url{https://doi.org/10.1145/3770854.3780167}\par\medskip
    }
}
{
    \copyrightyear{2026}
    \acmYear{2026}
    \setcopyright{cc}
    \setcctype{by}
    \acmConference[KDD '26]{Proceedings of the 32nd ACM SIGKDD Conference on Knowledge Discovery and Data Mining V.1}{August 09--13, 2026}{Jeju Island, Republic of Korea}
    \acmBooktitle{Proceedings of the 32nd ACM SIGKDD Conference on Knowledge Discovery and Data Mining V.1 (KDD '26), August 09--13, 2026, Jeju Island, Republic of Korea}
    \acmPrice{}
    \acmDOI{10.1145/3770854.3780167}
    \acmISBN{979-8-4007-2258-5/2026/08}
    \acmSubmissionID{89}
}

\usepackage{xr-hyper}
\ifthenelse{\boolean{fullpaper}}{}{%
  \externaldocument{body_appendix}
}

\DeclareRobustCommand{\refappendix}[1]{%
  ~\ref{#1}%
  \ifthenelse{\boolean{fullpaper}}{}{ of our full paper\,\citep{park2024fair}}%
  \unskip
}

\begin{document}

\newcommand{\method}{FSW}

\title{Fair Class-Incremental Learning using Sample Weighting}

\author{Jaeyoung Park}
\authornote{Equal contribution.}
\affiliation{%
  \institution{Korea Advanced Institute of Science and Technology}
  \city{Daejeon}
  \country{Republic of Korea}
}
\email{jypark@kaist.ac.kr}

\author{Minsu Kim}
\authornotemark[1]
\affiliation{%
  \institution{Korea Advanced Institute of Science and Technology}
  \city{Daejeon}
  \country{Republic of Korea}
}
\email{ms716@kaist.ac.kr}

\author{Steven Euijong Whang}
\authornote{Corresponding author.}
\affiliation{%
  \institution{Korea Advanced Institute of Science and Technology}
  \city{Daejeon}
  \country{Republic of Korea}
}
\email{swhang@kaist.ac.kr}


\begin{abstract}
Model fairness is becoming important in class-incremental learning for Trustworthy AI. While accuracy has been a central focus in class-incremental learning, fairness has been relatively understudied. However, na\"ively using all the samples of the current task for training results in {\it unfair catastrophic forgetting} for certain sensitive groups including classes. We theoretically analyze that forgetting occurs if the average gradient vector of the current task data is in an ``opposite direction'' compared to the average gradient vector of a sensitive group, which means their inner products are negative. We then propose a {\it fair class-incremental learning} framework that adjusts the training weights of current task samples to change the direction of the average gradient vector and thus reduce the forgetting of underperforming groups and achieve fairness. For various group fairness measures, we formulate optimization problems to minimize the overall losses of sensitive groups while minimizing the disparities among them. We also show the problems can be solved with linear programming and propose an efficient Fairness-aware Sample Weighting (\method{}) algorithm. Experiments show that \method{} achieves better accuracy-fairness tradeoff results than state-of-the-art approaches on real datasets. The source code is released at https://github.com/jyparkkr/FSW.
\end{abstract}



\maketitle
\arxivnotice

\section{Introduction}
\label{sec:intro}
Trustworthy AI is becoming critical in various continual learning applications including autonomous vehicles, personalized recommendations, healthcare monitoring, and more\,\citep{DBLP:journals/corr/abs-2107-06641, DBLP:journals/csur/KaurURD23}. 
Improving both model fairness and accuracy is crucial, as unfair predictions can significantly compromise trust and safety, particularly in human-centered automated systems, especially as observed frequently in the context of continual learning. 
In this paper, we focus on class-incremental learning, where the objective is to incrementally learn new classes as they appear. 

The main challenge of class-incremental learning is to learn new classes of data, while not forgetting previously-learned classes\,\citep{DBLP:journals/nn/BelouadahPK21, DBLP:journals/pami/LangeAMPJLST22}. If we simply fine-tune the model on the new classes, the model will gradually forget about the previously-learned classes. This phenomenon called catastrophic forgetting\,\citep{MCCLOSKEY1989109, DBLP:journals/corr/KirkpatrickPRVD16, DBLP:journals/nn/ParisiKPKW19} may easily occur in real-world scenarios where the model needs to continuously learn new classes. 
Therefore, we need to have a balance between learning new information and retaining previously-learned knowledge, which is called the stability-plasticity dilemma\,\citep{ABRAHAM200573, 10.3389/fpsyg.2013.00504, DBLP:conf/cvpr/KimH23}.

In this paper, we solve the problem of {\it fair class-incremental learning} where the goal is to satisfy various notions of fairness among sensitive groups including classes in addition to classifying accurately. 
In continual learning, unfair forgetting may occur if the current task data has similar characteristics to previous data, but belongs to different sensitive groups including classes, which negatively affects the performance on the previous data during training. Despite the importance of the problem, the existing research\,\citep{DBLP:conf/aaai/ChowdhuryC23, DBLP:conf/nips/TruongNRL23} is still nascent and has limitations in terms of technique or scope (see Sec.~\ref{sec:relatedwork}). In comparison, we support fairness more generally in class-incremental learning by satisfying various notions of group fairness for sensitive groups including classes.

\begin{figure*}[t]
    \begin{subfigure}[t]{0.25\textwidth}
        \includegraphics[width=\linewidth]{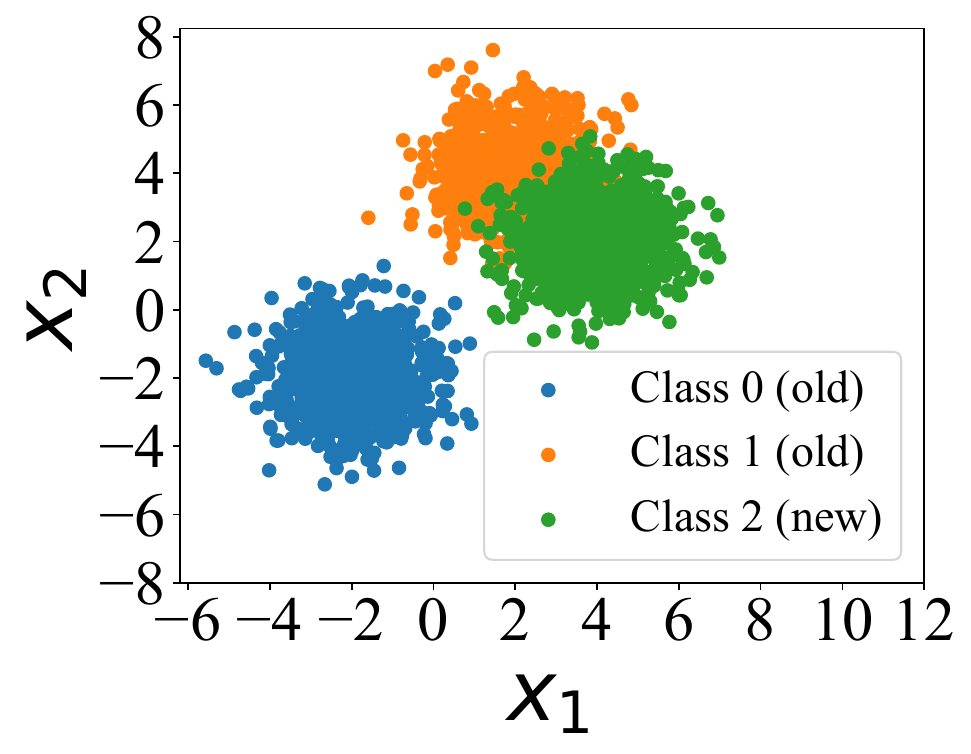}
        \caption{}
        \label{fig:synthetic_dataset}
    \end{subfigure}
    \begin{subfigure}[t]{0.25\textwidth}
        \includegraphics[width=\linewidth]{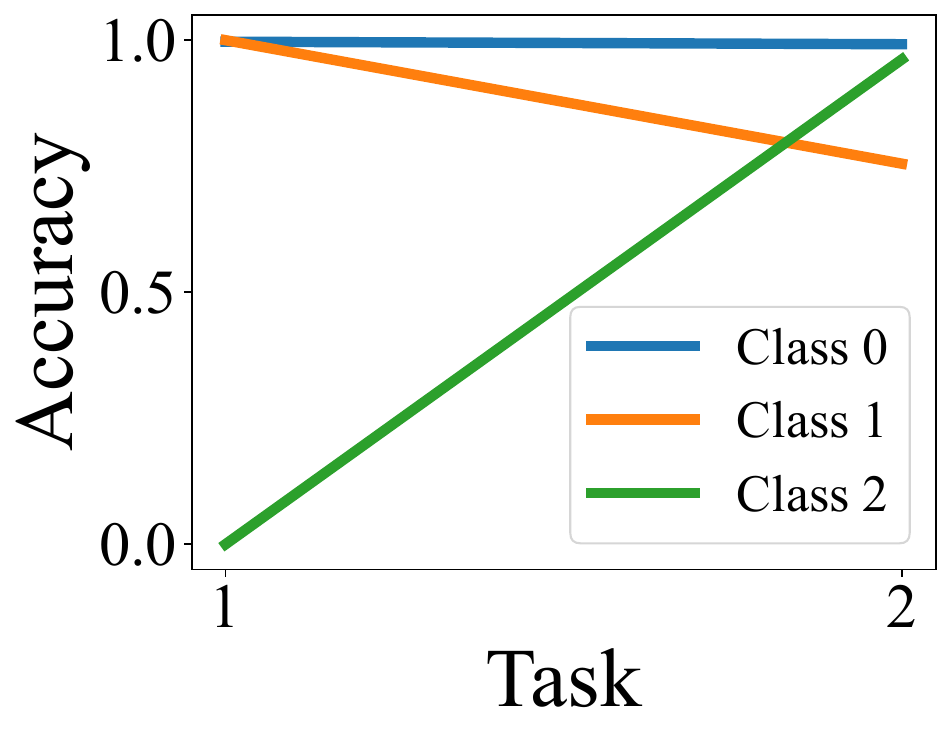}
        \caption{}
        \label{fig:synthetic_forgetting}
    \end{subfigure}
    \hspace{0.01\textwidth}
    \begin{subfigure}[t]{0.205\textwidth}
        \includegraphics[width=\linewidth]{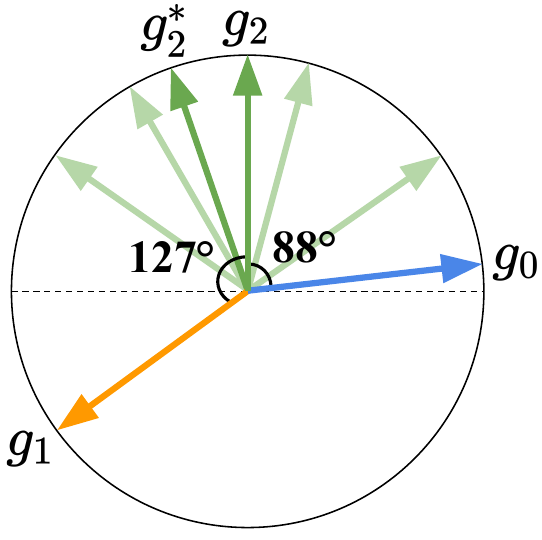}
        \caption{}
        \label{fig:gradient_vector}
    \end{subfigure}
    \begin{subfigure}[t]{0.25\textwidth}
        \includegraphics[width=\linewidth]{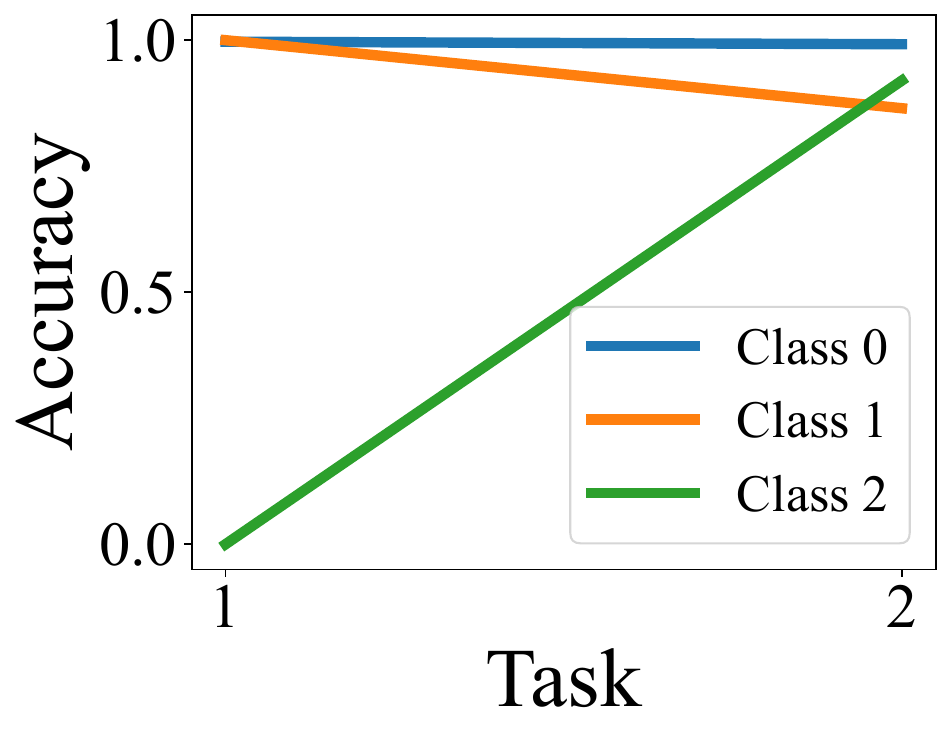}
        \caption{}
        \label{fig:synthetic_mitigation}
    \end{subfigure}
    \caption{(a) A synthetic dataset for class-incremental learning. (b) Training on Class 2 results in unfair forgetting on Class 1 only. (c) The average gradient vector of Class 2, $g_2$, is more than $90^{\circ}$ apart from Class 1's $g_1$, which means the model is being trained in an opposite direction. Our method adjusts $g_2$ to $g_2^*$ through sample weighting to be closer to $g_1$, but not too far from the original $g_2$. (d) As a result, the unfair forgetting is mitigated while minimally sacrificing accuracy for Class 2.}
    \label{fig:toy_exp}
\end{figure*}

We demonstrate how unfair forgetting can occur on a synthetic dataset with two attributes $(x_1, x_2)$, and one true label $y$ as shown in Fig.~\ref{fig:synthetic_dataset}. We sample data for each class from three normal distributions: $(x_1, x_2)|y=0 \sim \mathcal{N}([-2;-2], [1;1])$, $(x_1, x_2)|y=1 \sim \mathcal{N}([2;4], [1;1])$, and $(x_1, x_2)|y=2 \sim \mathcal{N}([4;2], [1;1])$. 
To simulate class-incremental learning, we introduce data for Class 0 (blue) and Class 1 (orange) in Task 1, followed by Class 2 (green) data in Task 2, where Class 2's data is similar to Class 1's data. We observe that this setting frequently occurs in real datasets, where different classes of data exhibit similar features or characteristics, as shown in Sec.~\refappendix{subsec:tsne}. We assume a data replay setting in which only a limited amount of data from Classes 0 and 1 is stored and jointly used during training on Class 2. After completing training on Task 1, we observe how the model’s accuracy on the three classes changes during Task 2 training in Fig.~\ref{fig:synthetic_forgetting}. As the accuracy on Class 2 improves, there is a catastrophic forgetting of Class 1 only, which leads to unfairness. 

Beyond the label, unfairness can also arise when forgetting occurs among groups defined by both label $y$ and a sensitive attribute $z$ such as age, gender, or race. For instance, we can consider Class 0 as $(y=0, z=0)$, Class 1 as $(y=0, z=1)$, and Class 2 as $(y=1, z=1)$. Introducing Class 2 in Task 2 can lead to an accuracy drop on Class 1 $(z=1)$, while Class 0 $(z=0)$ remains unaffected. Likewise, unfairness can arise when groups are defined by both class labels and sensitive attributes.

To analytically understand the unfair forgetting, we project the average gradient vector for each class data on a 2-dimensional space in Fig.~\ref{fig:gradient_vector}. Here $g_0$, $g_1$, and $g_2$ represent the average gradient vectors of the samples of Classes 0, 1, and 2, respectively. We observe that $g_2$ is $127^{\circ}$ apart from $g_1$, but $88^{\circ}$ from $g_0$, which means that the inner products $\langle g_2, g_1 \rangle$ and $\langle g_2, g_0 \rangle$ are negative and close to 0, respectively. In Sec.~\ref{sec:unfairforgetting}, we theoretically show that a negative inner product between average gradient vectors of current and previous data results in higher loss for the previous data as the model is being updated in an opposite direction and identify a sufficient condition for unfair forgetting. As a result, Class 1's accuracy decreases, while Class 0's accuracy remains stable.

Our solution to mitigate unfair forgetting is to adjust the average gradient vector of the current task data by weighting its samples. The light-green vectors in Fig.~\ref{fig:gradient_vector} are the gradient vectors of individual samples from Class 2, and by weighting them we can adjust $g_2$ to $g_2^*$ to make the inner product with $g_1$ less negative. At the same time, we do not want $g_2^*$ to be too different from $g_2$ and lose accuracy. In Sec.~\ref{sec:sampleweighting}, we formalize this idea using the weighted average gradient vector of the current task data. We then optimize the sample weights such that unfair forgetting and accuracy reduction over sensitive groups including classes are both minimized. We show this optimization can be solved with linear programming and propose our efficient Fairness-aware Sample Weighting (FSW) algorithm. Fig.~\ref{fig:synthetic_mitigation} shows how using \method{} mitigates the unfair forgetting between Classes 0 and 1 without harming Class 2's accuracy much. Our framework supports the group fairness measures equal error rate\,\citep{DBLP:conf/pods/Venkatasubramanian19}, equalized odds\,\citep{DBLP:conf/nips/HardtPNS16}, and demographic parity\,\citep{DBLP:conf/kdd/FeldmanFMSV15} and can be potentially extended to other measures.

In our experiments, we show that \method{} achieves better fairness and competitive accuracy compared to state-of-the-art baselines on various image, text, and tabular datasets. The benefits come from assigning different training weights to the current task samples with accuracy and fairness in mind.

\textbf{Summary of Contributions:} (1) We theoretically analyze how unfair catastrophic forgetting can occur in class-incremental learning; (2) We formulate optimization problems for mitigating the unfairness for various group fairness measures and propose an efficient fairness-aware sample weighting algorithm, \method{}; (3) We demonstrate how \method{} outperforms state-of-the-art baselines in terms of fairness with comparable accuracy on various datasets.

\section{Related Work}
\label{sec:relatedwork}

Class-incremental learning is a challenging type of continual learning where a model continuously learns new tasks, each composed of new disjoint classes, and the goal is to minimize catastrophic forgetting\,\citep{DBLP:journals/ijon/MaiLJQKS22, DBLP:journals/pami/MasanaLTMBW23}. Data replay techniques\,\citep{DBLP:conf/nips/Lopez-PazR17, DBLP:conf/cvpr/RebuffiKSL17, DBLP:journals/corr/abs-1902-10486} store a small portion of previous data in a buffer to utilize for training and are widely used with other techniques
(see more details in Sec.~\refappendix{appendix:related_work}). 
There are also more advanced gradient-based sample selection techniques like GSS\,\citep{DBLP:conf/nips/AljundiLGB19} and OCS\,\citep{DBLP:conf/iclr/YoonMYH22} that manage buffer data to have samples with diverse and representative gradient vectors. 
All these works do not consider fairness and simply assume that the entire incoming data is used for model training, which may result in unfair forgetting as we show in our experiments.

Model fairness research mitigates bias by ensuring that a model's performance is equitable across different sensitive groups, thereby preventing discrimination based on race, gender, age, or other sensitive attributes\,\citep{DBLP:journals/csur/MehrabiMSLG21}
(see more details in Sec.~\refappendix{appendix:related_work}). 
However, most of these techniques assume that the training data is given all at once, which may not be realistic. 
There are techniques for fairness-aware active learning\,\citep{anahideh2022fair, pang2024fairness, tae2024falcon}, in which the training data evolves with the acquisition of samples. However, these techniques store all labeled data and use them for training, which is impractical in continual learning settings.

A recent study addresses model fairness in class-incremental learning, where models may suffer from imbalanced forgetting of previously learned sensitive groups including classes, leading to unfairness across different groups\,\citep{he2024gradient, xu2024defying, DBLP:conf/aaai/ChowdhuryC23, DBLP:conf/nips/TruongNRL23} (see more details in Sec.~\refappendix{appendix:related_work}). 
FaIRL\,\citep{DBLP:conf/aaai/ChowdhuryC23} supports group fairness metrics like demographic parity for continual learning, but proposes a representation learning method that does not directly optimize the given fairness measure and thus has limitations in improving fairness as we show in experiments. 
FairCL\,\citep{DBLP:conf/nips/TruongNRL23} also addresses fairness in a continual learning setup, but only focuses on resolving the imbalanced class distribution based on the number of pixels of each class in an image, and is specific to semantic segmentation tasks. 
In comparison, we can mitigate bias from both data imbalance and intrinsic or acquired characteristics within the data by satisfying multiple notions of group fairness for sensitive groups including classes.

\section{Framework}
\label{sec:framework}

In this section, we first theoretically analyze unfair forgetting using gradient vectors of sensitive groups and the current task data. Next, we propose a sample weighting algorithm to mitigate unfairness by adjusting the average gradient vector of the current task data.

\paragraph{Notations}\label{para:notations}
In class-incremental learning, a model incrementally learns new current task data along with previous buffer data using data replay. Suppose we train a model to incrementally learn $L$ tasks $\{T_1, T_2, \ldots, T_L\}$ over time, and there are $N$ classes in each task 
with no overlapping classes between different tasks.
After learning the $l^{th}$ task $T_l$, we would like the model to remember all $(l-1) \cdot N$ previous task classes and an additional $N$ current task classes. We assume the buffer has a fixed size of $M$ samples. For $L$ tasks, we allocate $m = M/L$ samples of buffer data per task. If each task consists of $N$ classes, then we allocate $m/N = M/(L \cdot N)$ samples of buffer data per class\,\citep{DBLP:conf/iclr/ChaudhryRRE19, DBLP:conf/nips/MirzadehFPG20, DBLP:conf/aaai/ChaudhryGDTL21}. Each task $T_l = \{d_i=(X_i, y_i)\}$ is composed of feature-label pairs.
We also use $\mathcal{M}_l = \{d_j=(X_j, y_j)\}_{j=1}^{m}$ to represent the buffer data for each previous $l^{th}$ task $T_l$. We assume the buffer data per task is small, i.e., $m \ll |T_l|$\,\citep{DBLP:journals/corr/abs-1902-10486}. 

When defining fairness for class-incremental learning, we utilize sensitive groups including classes. According to the fairness literature, sensitive groups are divided by sensitive attributes like gender and race (e.g., Male and Female).
Similarly, the classes can also be viewed as sensitive groups, where the class label serves as the sensitive attribute.
To support either case in a class-incremental setting,
we use the following unifying notations: (1) if the sensitive groups are classes, then they form the set $G_y = \{(X, \text{y}) \in \mathcal{D}: \text{y} = y, y \in \mathbb{Y}\}$ where $\mathcal{D}$ is a dataset, $\text{y}$ is a class attribute, and $\mathbb{Y}$ is the set of $\text{y}$; (2) if we are using sensitive attributes beyond classes, we can further divide the classes into the set $G_{y, z} = \{(X, \text{y}, \text{z}) \in \mathcal{D}: \text{y} = y, \text{z} = z, y \in \mathbb{Y}, z \in \mathbb{Z}\}$ where $\text{z}$ is a sensitive attribute and $\mathbb{Z}$ is the set of $\text{z}$. 
See Sec.~\refappendix{appendix:notations} for a summary of notations used.

\subsection{Unfair Forgetting}
\label{sec:unfairforgetting}

We take inspiration from GEM\,\citep{DBLP:conf/nips/Lopez-PazR17}, which theoretically analyzes catastrophic forgetting by utilizing the angle between gradient vectors of data. If the inner products of gradient vectors for previous tasks and the current task are negative (i.e., $90^{\circ} <$ angle $\leq 180^{\circ}$), the loss of previous tasks increases after learning the current task. Catastrophic forgetting thus occurs when the gradient vectors of different tasks point in opposite directions. Intuitively, the opposite gradient vectors update the model parameters in conflicting directions, leading to forgetting while learning.

Using the notion of catastrophic forgetting, we propose theoretical results for unfair forgetting:

\begin{lemma}
\label{lem:CF}
    Denote $G$ as a sensitive group containing features $X$ and true labels $y$. Also, denote $f_{\theta}^{l-1}$ as a previous model and $f_{\theta}$ as the updated model after training on the current task $T_l$. Let $\ell$ be any differentiable loss function (e.g., cross-entropy loss), and $\eta$ be a learning rate. Then, the loss of the sensitive group of data after training with a current task sample $d_i \in T_l$ is approximated as: 
    \begin{equation} \label{eq:loss_update}
    \tilde{\ell}(f_{\theta}, G) = \ell(f_{\theta}^{l-1}, G) - \eta \nabla_{\theta} \ell(f_{\theta}^{l-1}, G)^\top \nabla_{\theta} \ell(f_{\theta}^{l-1}, d_i),
    \end{equation}
    where $\tilde{\ell}(f_{\theta}, G)$ is the approximated average loss between model predictions $f_{\theta}(X)$ and true labels $y$, whereas $\ell(f_{\theta}^{l-1}, G)$ is the exact average loss, $\nabla_{\theta} \ell(f_{\theta}^{l-1}, G)$ is the average gradient vector for the samples in the group $G$, and $\nabla_{\theta} \ell(f_{\theta}^{l-1}, d_i)$ is the gradient vector for a sample $d_i$, each with respect to the previous model $f_{\theta}^{l-1}$.
\end{lemma}

The proof is in Sec.~\refappendix{appendix:linear-approximation}. We employ first-order Taylor series approximation for the proof, which is widely used in the continual learning literature, by assuming that the loss function is locally linear in small optimization steps and considering the first-order term as the cause of catastrophic forgetting\,\citep{DBLP:conf/nips/Lopez-PazR17, DBLP:conf/nips/AljundiLGB19, DBLP:conf/nips/LeeXSBNSP19}. We empirically find that the approximation error is large when a new task begins because new samples with unseen classes are introduced. However, the error gradually becomes quite small as the number of epochs increases while training a model for the task, as in Sec.~\refappendix{appendix:apx_error}. 
Additionally, while higher-order terms can be included in principle, doing so makes the optimization non-convex and NP-hard. 
A detailed proof is provided in Sec.~\refappendix{appendix:higher_order_approximation}.

To define fairness in class-incremental learning with the approximated loss, we adopt the definition of approximate fairness that considers a model to be fair if it has approximately the same loss on the positive class, independent of the group membership\,\citep{DBLP:conf/nips/DoniniOBSP18}. In this paper, we compute fairness measures based on the disparity between approximated cross-entropy losses, which are derived from Lemma~\ref{lem:CF} using gradients.
The following proposition shows how using the cross-entropy loss disparity can effectively approximate common group fairness metrics such as equalized odds and demographic parity.
We choose cross-entropy loss specifically because it is widely validated as a strong proxy for fairness disparity metrics, though our method is robust to alternative loss choices
(see Sec.~\ref{appendix:CE-as-approximator} and~\refappendix{appendix:alternative_loss_design} for more justification of the loss function and an alternative, respectively). 

\begin{proposition}
\label{prop1}
(From\,\citep{DBLP:conf/iclr/Roh0WS21, roh2023drfairness, DBLP:conf/naacl/ShenHCBF22}) Using the cross-entropy loss disparity to measure fairness is empirically verified to provide reasonable proxies for common group fairness metrics.
\end{proposition}

Based on Lemma~\ref{lem:CF} and Proposition~\ref{prop1}, Theorem~\ref{thm:CFcondition} provides a sufficient condition for unfair forgetting. Intuitively, if a sample’s gradient opposes only the underperforming group’s average gradient, the training amplifies unfairness between the two groups.

\begin{theorem} \label{thm:CFcondition}
    Let $\ell$ be the cross-entropy loss and $G_1$ and $G_2$ the overperforming and underperforming sensitive groups of data, respectively. Also let $d_i$ be a training sample that satisfy the following conditions: $\ell(f_\theta^{l-1}, G_1) < \ell(f_\theta^{l-1}, G_2)$ while $\nabla_\theta \ell(f_\theta^{l-1}, G_1)^\top \nabla_\theta \ell(f_\theta^{l-1}, d_i) > 0$ and $\nabla_\theta \ell(f_\theta^{l-1}, G_2)^\top \nabla_\theta \ell(f_\theta^{l-1}, d_i) < 0$. Then $|\tilde{\ell}(f_\theta, G_1) - \tilde{\ell}(f_\theta, G_2)| > |\ell(f_\theta^{l-1}, G_1) - \ell(f_\theta^{l-1}, G_2)|$.
\end{theorem}

The proof is in Sec.~\refappendix{appendix:linear-approximation}. The result shows that the loss disparity between the two groups could become larger after training on the current task sample, which leads to worse fairness. This theorem can be extended to when we have a set of current task samples $T_l = \{d_i=(X_i, y_i)\}$ where we can replace $\nabla_\theta \ell(f_\theta^{l-1}, d_i)$ with $\frac{1}{|T_l|} \sum_{d_i \in T_l} \nabla_\theta \ell(f_\theta^{l-1}, d_i)$.
If the average gradient vector of the current task data satisfies the derived sufficient condition, training with all of the current task samples using equal weights could thus result in unfair catastrophic forgetting.

\subsection{Sample Weighting for Unfairness Mitigation}
\label{sec:sampleweighting}

To mitigate unfairness, we propose sample weighting to suppress samples that harm fairness and promote samples that help. 
Given training weights $\mathbf{w}_l \in [0, 1]^{|T_l|}$ for the samples in the current task data, the approximated loss of a group $G$ after training is:
\[
\tilde{\ell}(f_{\theta}, G) = \ell(f_{\theta}^{l-1}, G) - \eta \nabla_{\theta} \ell(f_{\theta}^{l-1}, G)^\top \biggl(\frac{1}{|T_l|} \sum_{d_i \in T_l} \mathbf{w}_l^i \nabla_{\theta} \ell(f_{\theta}^{l-1}, d_i)\biggr),
\]
where $\mathbf{w}_l^i$ is a training weight for the current task sample $d_i$. We then formulate an optimization problem to find the weights such that both loss and unfairness are minimized. Here we define $\mathbb{Y}$ as the set of all classes and $\mathbb{Y}_c$ as the set of classes in the current task. We represent accuracy as the average loss over the current task data and minimize the cost function $\mathcal{L}_{acc} = \tilde{\ell}(f_{\theta}, G_{\mathbb{Y}_c}) = \frac{1}{|\mathbb{Y}_c|} \sum_{y \in \mathbb{Y}_c} \tilde{\ell}(f_{\theta}, G_y) = \frac{1}{|\mathbb{Y}_c| |\mathbb{Z}|} \sum_{y \in \mathbb{Y}_c, z \in \mathbb{Z}} \tilde{\ell}(f_{\theta}, G_{y, z})$. For fairness, the cost function $\mathcal{L}_{fair}$ depends on the group fairness measure as we explain below. We then minimize $\mathcal{L}_{fair} + \lambda \mathcal{L}_{acc}$ where $\lambda$ is a hyperparameter that balances fairness and accuracy. 

\paragraph{Equal Error Rate (EER)}

This measure \citep{DBLP:conf/pods/Venkatasubramanian19} is defined as $\Pr(\hat{\text{y}} \neq y_1 | \text{y} = y_1) = \Pr(\hat{\text{y}} \neq y_2 | \text{y} = y_2)$ for $y_{1}, y_{2} \in \mathbb{Y}$, where $\hat{\text{y}}$ is the predicted class, and $\text{y}$ is the true class. 
We define the cost function for EER as the average absolute deviation of the class loss, consistent with the definition of group fairness metrics: 
$\mathcal{L}_{EER} = \frac{1}{|\mathbb{Y}|} \sum_{y \in \mathbb{Y}} \allowbreak | \tilde{\ell}(f_{\theta}, G_y) - \tilde{\ell}(f_{\theta}, G_\mathbb{Y}) |$. The entire optimization problem is: 
\begin{equation} \label{eq:eer}
\begin{gathered}
\min_{\mathbf{w}_l} \frac{1}{|\mathbb{Y}|} \sum_{y \in \mathbb{Y}} | \tilde{\ell}(f_{\theta}, G_y) - \tilde{\ell}(f_{\theta}, G_\mathbb{Y}) | \ + \lambda \frac{1}{|\mathbb{Y}_c|} \sum_{y \in \mathbb{Y}_c} \tilde{\ell}(f_{\theta}, G_y), \\
\text{where} \ \ \tilde{\ell}(f_{\theta}, G_y) = \ell(f_{\theta}^{l-1}, G_y) \\ -  
\eta \nabla_{\theta} \ell(f_{\theta}^{l-1}, G_y)^\top \biggl(\frac{1}{|T_l|} \sum_{d_i \in T_l} \mathbf{w}_l^i \nabla_{\theta} \ell(f_{\theta}^{l-1}, d_i)\biggr).
\end{gathered}
\end{equation}

\paragraph{Equalized Odds (EO)}
\label{EO_description}
This measure\,\citep{DBLP:conf/nips/HardtPNS16} is satisfied when sensitive groups have the same accuracy, i.e., $\Pr(\hat{\text{y}}=y|\text{y}=y, \text{z}=z_1) = \Pr(\hat{\text{y}}=y|\text{y}=y, \text{z}=z_2)$ for $y \in \mathbb{Y}$ and $z_{1}, z_{2} \in \mathbb{Z}$. We design the cost function for EO as $\mathcal{L}_{EO} = \frac{1}{|\mathbb{Y}| |\mathbb{Z}|} \sum_{y \in \mathbb{Y}, z \in \mathbb{Z}} |\tilde{\ell}(f_{\theta}, G_{y, z}) - \tilde{\ell}(f_{\theta}, G_{y})|$, and the entire optimization problem is:
\begin{equation} \label{eq:eo}
\begin{gathered}
\min_{\mathbf{w}_l} \frac{1}{|\mathbb{Y}| |\mathbb{Z}|} \sum_{y \in \mathbb{Y}, z \in \mathbb{Z}} |\tilde{\ell}(f_{\theta}, G_{y, z}) - \tilde{\ell}(f_{\theta}, G_{y})| \\ + 
\lambda\frac{1}{|\mathbb{Y}_c| |\mathbb{Z}|} \sum_{y \in \mathbb{Y}_c, z \in \mathbb{Z}} \tilde{\ell}(f_{\theta}, G_{y, z}), \\
\text{where} \ \ \tilde{\ell}(f_{\theta}, G_{y, z}) = \ell(f_{\theta}^{l-1}, G_{y, z}) \\ - 
\eta \nabla_{\theta} \ell(f_{\theta}^{l-1}, G_{y, z})^\top \biggl(\frac{1}{|T_l|} \sum_{d_i \in T_l} \mathbf{w}_l^i \nabla_{\theta} \ell(f_{\theta}^{l-1}, d_i)\biggr).
\end{gathered}
\end{equation}

\paragraph{Demographic Parity (DP)} \label{para:DP}
This measure\,\citep{DBLP:conf/kdd/FeldmanFMSV15} is satisfied by minimizing the difference in positive prediction rates between sensitive groups. Here, we extend the notion of demographic parity to the multi-class setting\,\citep{DBLP:conf/nips/AlabdulmohsinSK22, denis2023fairness}, i.e., $\Pr(\hat{\text{y}}=y|\text{z}=z_1) = \Pr(\hat{\text{y}}=y|\text{z}=z_2)$ for $y \in \mathbb{Y}$ and $z_{1}, z_{2} \in \mathbb{Z}$. In the binary setting of $\mathbb{Y} = \mathbb{Z} = \{0, 1\}$, a sufficient condition for demographic parity is suggested using the loss multiplied by the ratios of sensitive groups\,\citep{DBLP:conf/iclr/Roh0WS21}.
By extending to multi-class, we derive a sufficient condition as follows: $\frac{m_{y,z_1}}{m_{*,z_1}} \ell(f_{\theta}, G_{y, z_1}) = \frac{m_{y,z_2}}{m_{*,z_2}} \ell(f_{\theta}, G_{y, z_2})$ where $m_{y,z}:= |\{i: \text{y}_i = y, \text{z}_i = z\}|$ and $m_{*,z}:= |\{i: \text{z}_i = z\}|$. 
The proof is in Sec.~\ref{appendix:multi-dp}. Let us define $\ell'(f_{\theta}, G_{y, z}) = \frac{m_{y,z}}{m_{*,z}} \ell(f_{\theta}, G_{y, z})$ and $\ell'(f_{\theta}, G_{y}) = \frac{1}{|\mathbb{Z}|} \sum_{n=1}^{|\mathbb{Z}|} \frac{m_{y,z_n}}{m_{*,z_n}} \ell(f_{\theta}, G_{y, z_n})$. We then define the cost function for DP as $\mathcal{L}_{DP} = \frac{1}{|\mathbb{Y}| |\mathbb{Z}|} \sum_{y \in \mathbb{Y}, z \in \mathbb{Z}} |\tilde{\ell}'(f_\theta, G_{y, z}) - \tilde{\ell}'(f_\theta, G_y)|$. The entire optimization problem is:
\begin{equation}  \label{eq:dp}
\begin{gathered}
\min_{\mathbf{w}_l} \frac{1}{|\mathbb{Y}| |\mathbb{Z}|} \sum_{y \in \mathbb{Y}, z \in \mathbb{Z}} |\tilde{\ell}'(f_\theta, G_{y, z}) - \tilde{\ell}'(f_\theta, G_y)| \\ + 
\lambda \frac{1}{|\mathbb{Y}_c| |\mathbb{Z}|} \sum_{y \in \mathbb{Y}_c, z \in \mathbb{Z}} \tilde{\ell}(f_{\theta}, G_{y, z}), \\
\text{where} \ \ \tilde{\ell}(f_{\theta}, G_{y, z}) = \ell(f_{\theta}^{l-1}, G_{y, z}) \\ - 
\eta \nabla_{\theta} \ell(f_{\theta}^{l-1}, G_{y, z})^\top \biggl(\frac{1}{|T_l|} \sum_{d_i \in T_l} \mathbf{w}_l^i \nabla_{\theta} \ell(f_{\theta}^{l-1}, d_i)\biggr).
\end{gathered}
\end{equation}

To find the optimal sample weights for the current task data considering both model accuracy and fairness, we first transform the defined optimization problems of Eqs.~\ref{eq:eer},~\ref{eq:eo}, and~\ref{eq:dp} into the form of linear programming (LP) problems.

\begin{theorem}\label{thm:OptimToLP}
The fairness-aware optimization problems (Eqs.~\ref{eq:eer},~\ref{eq:eo}, and~\ref{eq:dp}) can be formulated as linear programming (LP) problems.
\end{theorem}
\label{para:optim_to_lp}
The loss of each group can be approximated as a linear function, as described in Lemma~\ref{lem:CF}. This result implies that the optimization problems, consisting of the loss of each group, can likewise be formulated as LP problems. The proof is in Sec.~\refappendix{appendix:Fairness_LP}. We solve the LP problems using linear optimization solvers (e.g., CPLEX\,\citep{cplex2009v12}). As we add the average loss of the current task data in Eqs.~\ref{eq:eer},~\ref{eq:eo}, and~\ref{eq:dp} as a regularization term, the optimal sample weights do not indicate a severely shifted distribution.

\subsection{Algorithm}
\label{subsec:alg}

\begin{algorithm}[t]
\small
\caption{Fair Class-Incremental Learning}
\label{alg:cil}
\begin{algorithmic}[1]
  \Require Current task data $T_l$, previous buffer data $\mathcal{M}=\{\mathcal{M}_1,\dots,\mathcal{M}_{l-1}\}$, previous model $f_{\theta}^{l-1}$, loss function $\ell$, learning rate $\eta$, hyperparameters $\{\alpha,\lambda,\tau\}$, and fairness measure $F$.
  
  \For{each epoch}
    \State $\mathbf{w}_l^* \gets \textbf{\textit{FSW}}\bigl(T_l,\mathcal{M},f_{\theta}^{l-1},\ell,\alpha,\lambda,F\bigr)$
    \State $g_{\mathrm{curr}} \gets \frac{1}{|T_l|}\sum_{d_i\in T_l}\mathbf{w}_l^{*i}\,\nabla_{\theta}\,\ell\bigl(f_{\theta}^{l-1},d_i\bigr)$
    \State $g_{\mathrm{prev}} \gets \nabla_{\theta}\,\ell\bigl(f_{\theta}^{l-1},\mathcal{M}\bigr)$
    \State $\theta \gets \theta - \eta\,\bigl(g_{\mathrm{curr}} + \tau\,g_{\mathrm{prev}}\bigr)$
  \EndFor
  
  \State $\mathcal{M}_l \gets \textit{Buffer Sample Selection}(T_l)$
  \State $\mathcal{M} \gets \mathcal{M} \cup \mathcal{M}_l$
\end{algorithmic}
\end{algorithm}

We describe the overall process of fair class-incremental learning in Alg.~\ref{alg:cil}. For the recently arrived current task data, we first perform our fairness-aware sample weighting (\method{}) to assign training weights that can help learn new knowledge of the current task while retaining accurate and fair memories of previous tasks (Step 2). Next, we train the model using the current task data with its corresponding weights and stored buffer data of previous tasks (Steps 3--5), where $\eta$ is a learning rate, and $\tau$ is a hyperparameter to balance between them during training. The sample weights are computed once at the beginning of each epoch, and they are applied to all batches for computational efficiency\,\citep{DBLP:conf/icml/KillamsettySRDI21, DBLP:conf/aaai/KillamsettySRI21}. This procedure is repeated until the model converges (Steps 1--5). Before moving on to the next task, we employ buffer sample selection to store a small data subset for the current task (Steps 6--7). Buffer sample selection can also be done with consideration for fairness, but our experimental observations indicate that selecting representative and diverse samples for the buffer, as previous studies have shown, results in better accuracy and fairness. We thus employ a simple random sampling technique for the buffer sample selection in our framework (see Sec.~\refappendix{appendix:buffer_strategy} for more details). 

\begin{algorithm}[t]
\small
\caption{Fairness-aware Sample Weighting (FSW)}
\label{alg:fsw}
\begin{algorithmic}[1]
  \Require Current task data $T_l = \{\dots, d_i, \dots\}$, previous buffer data, $\mathcal{M} = \bigcup_{y\in\mathbb{Y}-\mathbb{Y}_c,\,z\in\mathbb{Z}}G_{y,z}$, previous model $f_{\theta}^{l-1}$, loss $\ell$, hyperparameters $\{\alpha,\lambda\}$, and fairness measure $F$.  
  \Ensure Optimal training weights $\mathbf{w}_l^*$ for current task data

  \State $\ell_G \gets [\ell(f_{\theta}^{l-1},G_{1,1}), \dots, \ell(f_{\theta}^{l-1},G_{|\mathbb{Y}|,|\mathbb{Z}|})]$
  \State $g_G \gets [\nabla_{\theta}\ell(f_{\theta}^{l-1},G_{1,1}), \dots, \nabla_{\theta}\ell(f_{\theta}^{l-1},G_{|\mathbb{Y}|,|\mathbb{Z}|})]$
  \State $g_d \gets [\dots, \nabla_{\theta}\ell(f_{\theta}^{l-1},d_i), \dots]$

  \State \textbf{switch} $F$ \textbf{do}
  \State \quad \textbf{case} EER: $\mathbf{w}_l^* \gets \text{Solve Eq.~\ref{eq:eer}}$
  \State \quad \textbf{case} EO:  $\mathbf{w}_l^* \gets \text{Solve Eq.~\ref{eq:eo}}$
  \State \quad \textbf{case} DP:  $\mathbf{w}_l^* \gets \text{Solve Eq.~\ref{eq:dp}}$

  \State \Return $\mathbf{w}_l^*$
\end{algorithmic}
\end{algorithm}

Alg.~\ref{alg:fsw} shows the fairness-aware sample weighting (FSW) algorithm for the current task data. We first divide both the previous buffer data and the current task data into groups based on each class and sensitive attribute. Next, we compute the average loss and gradient vectors for each group (Steps 1--2), and individual gradient vectors for the current task data (Step 3). To compute gradient vectors, we use the last layer approximation, which only considers the gradients of the model's last layer, that is efficient and known to be reasonable\,\citep{DBLP:conf/icml/KatharopoulosF18, DBLP:conf/iclr/AshZK0A20, DBLP:conf/icml/MirzasoleimanBL20, DBLP:conf/aaai/KillamsettySRI21, DBLP:conf/icml/KillamsettySRDI21, DBLP:conf/aaai/KimHW24}
(see Sec.~\refappendix{appendix:gradient_last_layer} for details).
We then solve linear programming to find the optimal sample weights for a user-defined target fairness measure such as EER (Step 5), EO (Step 6), and DP (Step 7). 
Since the gradient norm of the current task data is significantly larger than that of the buffer data, we utilize normalized gradients to update the loss of each group and replace the learning rate $\eta$ with a hyperparameter $\alpha$ in the equations. Finally, we return the current task’s sample weights for training (Step 8). 

Training with \method{} theoretically guarantees model convergence under the assumptions that the training loss is Lipschitz continuous and strongly convex, and that a proper learning rate is used\,\citep{DBLP:conf/icml/KillamsettySRDI21, chai2022fairness, DBLP:journals/tsp/LuTHC20}. The computational complexity of \method{} is quadratic to the number of current task samples, as CPLEX, which uses simplex-based algorithms, typically exhibits quadratic complexity with respect to the number of variables when solving LP problems in practice\,\citep{DBLP:journals/ior/Bixby02}. Our empirical results show that for about twelve thousand current task samples, the time to solve an LP problem is a few seconds, which leads to a few minutes of overall runtime for MNIST datasets (see Sec.~\refappendix{subsec:runtime} for details). Since we focus on continual offline training of large batches or separate tasks, rather than online learning, the overhead is manageable enough to deploy updated models in real-world applications. If the task size becomes too large, clustering similar samples and assigning weights to the clusters, rather than samples, could be a solution to reduce the computational overhead.

\section{Experiments}
\label{sec:exp}

In this section, we construct experiments on \method{} and address the following research questions: {\bf RQ1} How well can \method{} mitigate the unfair forgetting that occurs in class-incremental learning with better accuracy-fairness tradeoff? {\bf RQ2} How does \method{} weight the samples? 
{\bf RQ3} Can \method{} be further integrated with fair post-processing techniques?

\subsection{Experiment Settings}
\label{subsec:expsettings}

\subsubsection{Metrics} \label{subsubsec:metrics}
We follow the prior fair continual learning literature to evaluate accuracy and fairness\,\citep{DBLP:conf/aaai/ChowdhuryC23, DBLP:conf/nips/TruongNRL23}. 

\paragraph{Average Accuracy} 
We denote $A_l = \frac{1}{l}\sum_{t=1}^{l} a_{l, t}$ as the accuracy at the $l^{th}$ task, where $a_{l, t}$ is the accuracy of the $t^{th}$ task after learning the $l^{th}$ task. We average $A_l$ over all tasks to obtain the final average accuracy, $\overline{A_l} = \frac{1}{L}\sum_{l=1}^{L} A_l$ where $L$ is the total number of tasks.

\paragraph{Average Fairness}
We measure fairness for each task and then average these values to obtain the overall fairness score. Depending on the evaluation scenario, we select one of three disparity measures. In all cases, lower disparity indicates better fairness.

\subparagraph{\it (1) EER  } computes the average difference in test error rates among classes: $\frac{1}{|\mathbb{Y}|} \sum_{y \in \mathbb{Y}} |\Pr(\hat{\text{y}} \neq y | \text{y} = y) - \Pr(\hat{\text{y}} \neq \text{y})|$. 
EER targets fairness in predictive quality, aiming for equitable accuracy or error across different prediction outcomes. EER is used as a target metric where a sensitive attribute ($z$) is not explicitly provided. This scenario includes commonly used class-incremental learning benchmarks where only the class label ($y$) is available. In such cases, EER serves as a practical proxy for fairness under task imbalance, focusing on equal error across classes.

On the other hand, when $z$ is available, fairness can also be evaluated using group-based metrics like EO and DP, which use both $y$ and $z$ to cover broader social norms.

\subparagraph{\it (2) EO disparity} computes the average difference in accuracy among sensitive groups for all classes: $\frac{1}{|\mathbb{Y}| |\mathbb{Z}|} \sum_{y \in \mathbb{Y}, z \in \mathbb{Z}} |\Pr(\hat{\text{y}}=y|\text{y}=y, \text{z}=z) - \Pr(\hat{\text{y}}=y|\text{y}=y)|$. EO is preferred when ensuring equal model performance across groups or when true/false positives are equally important.

\subparagraph{\it (3) DP disparity} computes the average difference in class prediction ratios among sensitive groups for all classes: $\frac{1}{|\mathbb{Y}| |\mathbb{Z}|} \sum_{y \in \mathbb{Y}, z \in \mathbb{Z}} \allowbreak |\Pr(\hat{\text{y}}=y|\text{z}=z) - \Pr(\hat{\text{y}}=y)|$. DP is preferred when mitigating data bias or achieving ‘blind fairness’ by equalizing output distributions.

\begin{table}[t]
  \setlength{\tabcolsep}{6.65pt}
  \caption{Experimental settings for the five datasets.}
  \centering
  \begin{tabular}{lcccc}
    \toprule
    {\bf Dataset} & {\bf Size} & {\bf \#Features} & {\bf \#Classes} & {\bf \#Tasks} \\ 
    \midrule
    {\sf MNIST} & 60K & 28$\times$28 & 10 & 5  \\
    {\sf FMNIST} & 60K & 28$\times$28 & 10 & 5  \\
    {\sf Biased MNIST} & 60K & 3$\times$28$\times$28 & 10 & 5 \\
    {\sf DRUG} & 1.3K & 12 & 6 & 3 \\
    {\sf BiasBios} & 253K & 128$\times$768 & 25 & 5 \\
    \bottomrule
  \end{tabular}
  \label{tbl:datasets}
\end{table}

\begin{table*}[t]
  \setlength{\tabcolsep}{6.4pt}
  \caption{Accuracy and fairness results with respect to (1) EER disparity, where class is considered the sensitive attribute for MNIST and FMNIST datasets, and (2) EO disparity, where background color and gender are the sensitive attributes for the Biased MNIST and DRUG datasets, respectively (see DP disparity and the BiasBios dataset results in Sec.~\refappendix{appendix:results-main}). 
  We mark the best and second-best results with \textbf{bold} and \underline{underline}, respectively, excluding the na\"ive methods. 
  }
  \centering
  \begin{tabular}{l|cccccccccc}
  \toprule
    {Methods} & \multicolumn{2}{c}{\sf MNIST} & \multicolumn{2}{c}{\sf FMNIST} & \multicolumn{2}{c}{\sf Biased MNIST} & \multicolumn{2}{c}{\sf DRUG}  & \multicolumn{2}{c}{\sf BiasBios} \\
    \cmidrule{1-11}
    {} & {Acc.} & {EER Disp.} & {Acc.} & {EER Disp.} & {Acc.} & {EO Disp.} & {Acc.} & {EO Disp.}  & {Acc.} & {EO Disp.} \\
    \midrule
    {Joint Training} & .989\tiny{$\pm$.000} & .003\tiny{$\pm$.000} & .921\tiny{$\pm$.002} & .024\tiny{$\pm$.002} & .944\tiny{$\pm$.002} & .108\tiny{$\pm$.003} & .442\tiny{$\pm$.015} & .179\tiny{$\pm$.052} & .823\tiny{$\pm$.002} & .076\tiny{$\pm$.001} \\
    {Fine Tuning} & .455\tiny{$\pm$.000} & .326\tiny{$\pm$.000} & .451\tiny{$\pm$.000} & .325\tiny{$\pm$.000} & .449\tiny{$\pm$.001} & .016\tiny{$\pm$.002} & .357\tiny{$\pm$.009} & .125\tiny{$\pm$.034} & .420\tiny{$\pm$.001} & .028\tiny{$\pm$.002} \\
    \cmidrule{1-11}
    {iCaRL} & .918\tiny{$\pm$.005} & .048\tiny{$\pm$.003} & \textbf{.852\tiny{$\pm$.002}} & \underline{.047\tiny{$\pm$.001}} & .802\tiny{$\pm$.008} & .365\tiny{$\pm$.021} & \textbf{.444\tiny{$\pm$.025}} & .190\tiny{$\pm$.017} & \textbf{.829\tiny{$\pm$.002}} & .084\tiny{$\pm$.003} \\
    {WA} & .911\tiny{$\pm$.007} & .052\tiny{$\pm$.006} & .809\tiny{$\pm$.005} & .088\tiny{$\pm$.003} & \textbf{.916\tiny{$\pm$.002}} & .140\tiny{$\pm$.004} & .408\tiny{$\pm$.022} & .134\tiny{$\pm$.029} & .796\tiny{$\pm$.003} & .076\tiny{$\pm$.001}\\
    {CLAD} & .835\tiny{$\pm$}.016 & .099\tiny{$\pm$}.016 & .782\tiny{$\pm$}.018 & .118\tiny{$\pm$}.022 & .871\tiny{$\pm$}.012 & .198\tiny{$\pm$}.022 & .410\tiny{$\pm$}.026 & .114\tiny{$\pm$}.043 & .799\tiny{$\pm$.003} & .074\tiny{$\pm$.002}\\
    \cmidrule{1-11}
    {GSS} & .889\tiny{$\pm$.010} & .080\tiny{$\pm$.009} & .732\tiny{$\pm$.021} & .149\tiny{$\pm$.019} & .809\tiny{$\pm$.005} & .325\tiny{$\pm$.017} & \underline{.426\tiny{$\pm$.010}} & .167\tiny{$\pm$.038} & \underline{.808\tiny{$\pm$.003}} & .081\tiny{$\pm$.002} \\
    {OCS} & \textbf{.929\tiny{$\pm$.002}} & \underline{.040\tiny{$\pm$.003}} & .799\tiny{$\pm$.008} & .109\tiny{$\pm$.007} & .824\tiny{$\pm$.007} & .331\tiny{$\pm$.013} & .406\tiny{$\pm$.024} & .142\tiny{$\pm$.003} & -- & -- \\
    \cmidrule{1-11}
    {FaIRL} & .558\tiny{$\pm$.060} & .273\tiny{$\pm$.018} & .531\tiny{$\pm$.032} & .289\tiny{$\pm$.019} & .411\tiny{$\pm$.012} & \textbf{.118\tiny{$\pm$.011}} & .354\tiny{$\pm$.011} & \textbf{.060\tiny{$\pm$.021}} & .400\tiny{$\pm$.060} & \textbf{.055\tiny{$\pm$.020}}\\
    \cmidrule{1-11}
    {\bf \method{}} & \underline{.925\tiny{$\pm$.004}} & \textbf{.032\tiny{$\pm$.005}} & \underline{.824\tiny{$\pm$.006}} & \textbf{.039\tiny{$\pm$.006}} & \underline{.909\tiny{$\pm$.004}} & \underline{.119\tiny{$\pm$.007}} & .406\tiny{$\pm$.014} & \underline{.077\tiny{$\pm$.010}} & \underline{.808\tiny{$\pm$.002}} & \underline{.072\tiny{$\pm$.001}} \\
    \bottomrule
  \end{tabular}
  \label{tbl:performance_all}
\end{table*}

\subsubsection{Datasets} \label{subsubsec:dataset}
We use a total of five datasets as shown in Table~\ref{tbl:datasets}. We first utilize commonly used benchmarks for continual image classification tasks, which include MNIST and Fashion-MNIST (FMNIST). Here we regard the class as the sensitive attribute and evaluate fairness with EER disparity. We also use multi-class fairness benchmark datasets that have sensitive attributes\,\citep{DBLP:conf/eccv/XuWKG20, putzel2022blackbox, DBLP:journals/taffco/ChuramaniKG23, denis2023fairness}: Biased MNIST, Drug Consumption (DRUG), and BiasBios. We consider background color as the sensitive attribute for Biased MNIST, and gender for DRUG and BiasBios. We then use EO and DP disparity to evaluate fairness. We also consider using other datasets in the fairness field, but they are unsuitable for class-incremental learning experiments because either there are only two classes, or it is difficult to apply group fairness metrics (see Sec.~\refappendix{appendix:dataset} for details).

\subsubsection{Models and Hyperparameters} \label{subsubsec:models_hyparparmas}
Following the experimental setups of \citep{DBLP:conf/iclr/ChaudhryRRE19, DBLP:conf/nips/MirzadehFPG20}, we use a two-layer MLP with each 256 neurons for the MNIST, FMNIST, Biased MNIST, and DRUG datasets. For the BiasBios dataset, we use a pre-trained BERT language model\,\citep{DBLP:conf/naacl/DevlinCLT19}. 
We employ single-head evaluation where a final layer of the model is shared for all tasks\,\citep{DBLP:journals/corr/abs-1805-09733, DBLP:conf/eccv/ChaudhryDAT18}. 
To solve the fairness-aware optimization problems and find optimal sample weights, we use CPLEX, a high-performance optimization solver developed by IBM that specializes in solving linear programming (LP). See Sec.~\ref{appendix:optimization_solver} for experiments comparing CPLEX with another LP solver, HiGHS\,\citep{huangfu2018parallelizing}. 
For our buffer storage, we store 32 samples per sensitive group for all experiments.  For the hyperparameters $\alpha$, $\lambda$, and $\tau$ used in our algorithms, we perform cross-validation with a sequential grid search to find their optimal parameters.
Model details and hyperparameter set are provided in Sec.~\refappendix{appendix:models_hyparparmas}.

\subsubsection{Baselines} \label{subsubsec:baselines}
Following prior work\,\citep{DBLP:conf/nips/AljundiLGB19, DBLP:conf/iclr/YoonMYH22}, all baselines are continual learning methods, as other approaches do not fit this scenario. We compare \method{} with algorithms from four categories:

\paragraph{Na\"ive methods} {\it Joint Training} assumes access to all the data from previous classes for training, thereby serving as an upper bound in terms of performance; {\it Fine Tuning} trains a model only on data from new classes, without access to previous data, and thus serves as a lower bound.

\paragraph{State-of-the-art methods} {\it iCaRL}\,\citep{DBLP:conf/cvpr/RebuffiKSL17} performs herding-based buffer selection and representation learning using knowledge distillation loss; {\it WA}\,\citep{DBLP:conf/cvpr/ZhaoXGZX20} is a model rectification method designed to correct the bias in the model's final fully-connected layer. {\it WA} uses weight aligning to equalize the norms of the weight vectors over classes; {\it CLAD}\,\citep{xu2024defying} is a representation learning method that disentangles the representation interference between old and new classes. 

\paragraph{Sample selection methods} {\it GSS}\,\citep{DBLP:conf/nips/AljundiLGB19} selects a buffer with diverse gradients of samples; {\it OCS}\,\citep{DBLP:conf/iclr/YoonMYH22} uses gradient-based similarity, diversity, and affinity scores to rank and select samples for both current and buffer data in a unified manner.

\paragraph{Fairness-aware methods} {\it FaIRL}\,\citep{DBLP:conf/aaai/ChowdhuryC23} performs fair representation learning by controlling the rate-distortion function. 

Our choice of baselines was driven by practical considerations of computational cost and spatial complexity. For instance, while {\it L2P}\,\citep{wang2022learning} achieves strong performance, it leverages a pretrained vision transformer that already embeds extensive prior information, making it less directly comparable to \method{}. Similarly, {\it DER}\,\citep{yan2021dynamically} trains a separate feature extractor for each task, which incurs additional storage overhead and complexity. In contrast, the baselines like {\it iCaRL} and {\it WA} are widely recognized as strong and representative baselines in many studies\,\citep{zhou2024class, he2024gradient, zhuang2024class}, offering a fair and practical comparison. {\it FairCL}\,\citep{DBLP:conf/nips/TruongNRL23} addresses fairness in semantic segmentation tasks arising from the imbalanced class distribution of pixels, but we consider this problem to be unrelated from ours to add the method as a baseline.


\subsection{Accuracy and Fairness Results}
\label{subsec:expresults}

To answer {\bf RQ1}, we compare \method{} with other baselines in terms of accuracy and corresponding fairness metrics as shown in Table~\ref{tbl:performance_all} and Sec.~\refappendix{appendix:results-main}. 
The accuracy-fairness Pareto front is shown in Fig.~\ref{fig:tradeoff_curve} and Sec.~\refappendix{appendix:tradeoff}.
Due to the excessive time required to run {\it OCS} on BiasBios, we are not able to measure the results.
Detailed sequential performance results and additional analyses on varying buffer sizes are provided in Sec.~\ref{appendix:results-seq} and~\refappendix{appendix:buffer_size}, respectively. 

Overall, \method{} achieves better accuracy-fairness tradeoff results compared to the baselines across all the datasets. Although \method{} does not achieve the best performance in accuracy, \method{} shows the best fairness results among the baselines with similar accuracies (e.g., {\it iCaRL}, {\it WA}, {\it CLAD}, {\it GSS}, and {\it OCS}) and thus has the best accuracy-fairness tradeoff. We observe that \method{} also improves model accuracy while enhancing the performance of underperforming groups for fairness. 

The state-of-the-art method, {\it iCaRL}, generally achieves high accuracy with low EER disparity results. However, {\it iCaRL}'s nearest-mean-of-exemplars approach for its classification model, the predictions are significantly affected by sensitive attribute values, resulting in high disparities for EO and DP. 
Although {\it WA} also performs well, the method adjusts the model weights for the current task classes as a whole, which leads to an unfair forgetting of sensitive groups and unstable results. 
The closest work to \method{} is {\it CLAD}, which disentangles the representations of new classes and a fixed proportion of conflicting old classes to mitigate imbalanced forgetting across classes. However, the proportion of conflicts may vary by task in practice, limiting {\it CLAD}'s ability to achieve group fairness. 
While the two sample selection methods {\it GSS} and {\it OCS} store diverse and representative samples in the buffer, these methods sometimes result in an imbalance in the number of buffer samples across sensitive groups. 
The fairness-aware method {\it FaIRL} leverages an adversarial debiasing framework combined with a rate-distortion function. While this method theoretically promises improved fairness, it suffers from significant accuracy degradation due to the instability of jointly training the feature encoder and discriminator.

\begin{figure}[t!]
    \centering
    \begin{subfigure}[t]{0.48\textwidth}
    \centering
    \includegraphics[width=0.84\linewidth]{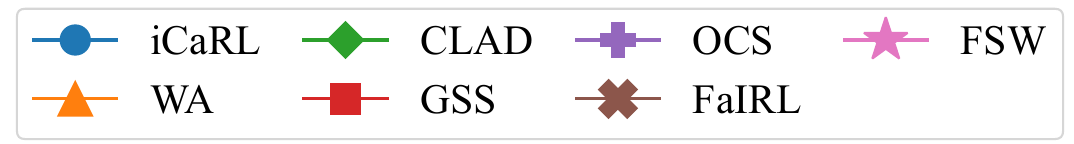}
    \vspace{0.2cm}
    \end{subfigure}
    \hfill
    \begin{subfigure}[t]{0.240\textwidth}
        \includegraphics[width=\linewidth]{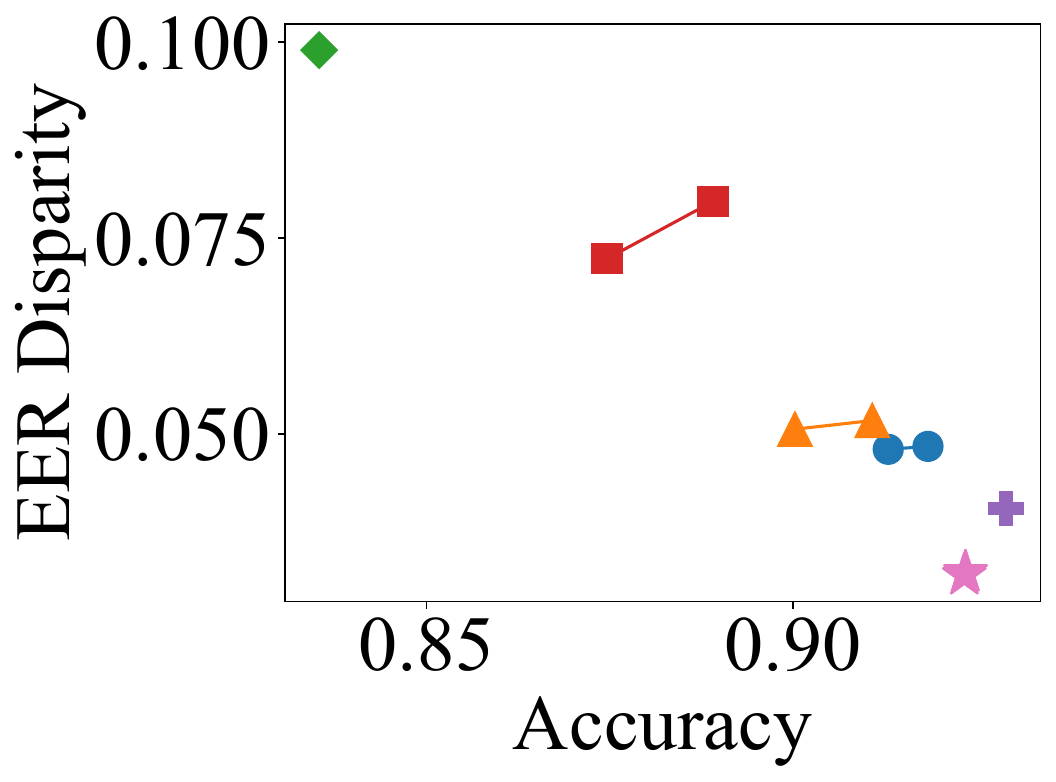}
        \caption{MNIST (EER).}
    \end{subfigure}
    \begin{subfigure}[t]{0.222\textwidth}
        \includegraphics[width=\linewidth]{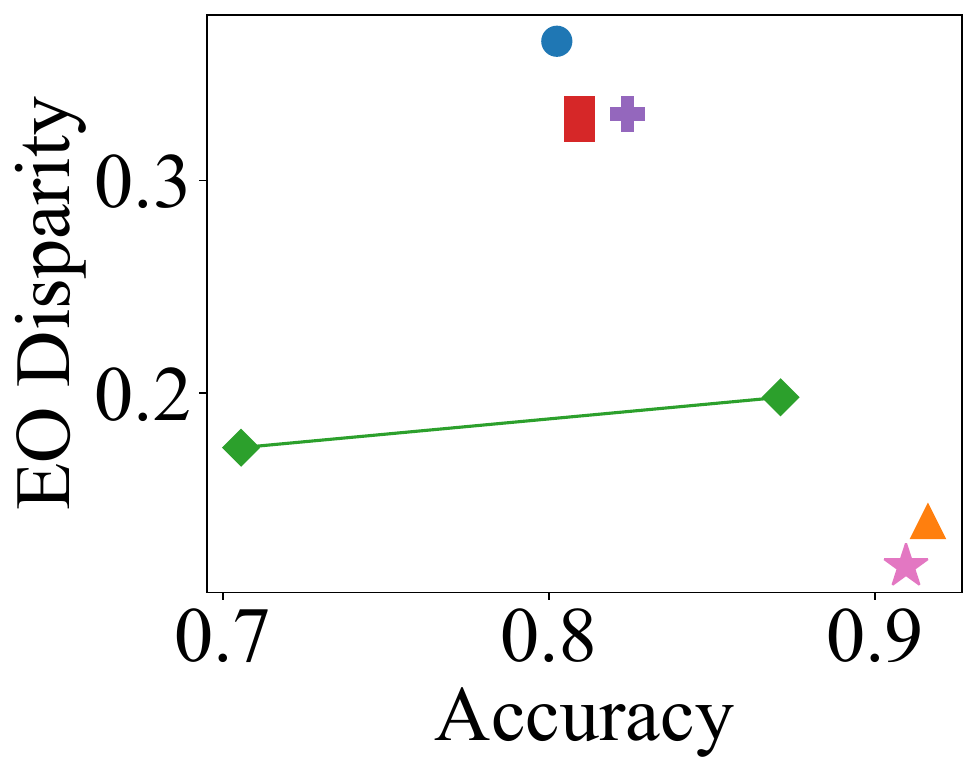}
        \caption{Biased MNIST (EO).}
    \end{subfigure}
    \caption{Tradeoff results between accuracy and fairness on the MNIST and Biased MNIST datasets. \method{} positioned in the lower right corner of the graph, indicating better accuracy-fairness tradeoff results compared to other baselines.}
    \label{fig:tradeoff_curve}
\end{figure}

In comparison, \method{} explicitly utilizes approximated loss and fairness measures to adjust the training weights for the current task samples, which leads to better model accuracy and fairness. As shown in Fig.~\ref{fig:tradeoff_curve} and detailed in Sec.~\refappendix{appendix:tradeoff}, \method{} consistently dominates the Pareto frontier, occupying the optimal bottom-right region across all evaluated tasks. These results demonstrate that \method{} consistently outperforms other baselines and achieves better accuracy-fairness tradeoff.


\begin{figure}[t]
    \begin{subfigure}[t]{0.243\textwidth}
        \includegraphics[width=\linewidth]{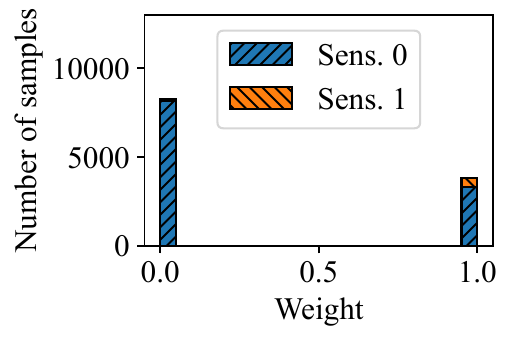}
        \caption{Task 2.}
    \end{subfigure}
    \begin{subfigure}[t]{0.225\textwidth}
        \includegraphics[width=\linewidth]{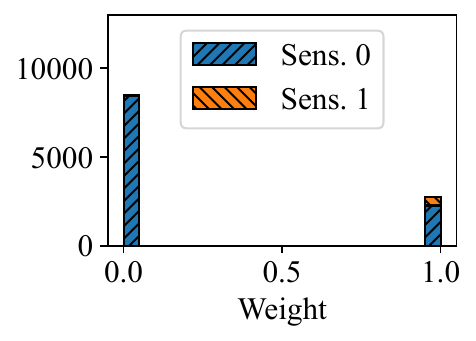}
        \caption{Task 3.}
    \end{subfigure}
    \begin{subfigure}[t]{0.243\textwidth}
        \includegraphics[width=\linewidth]{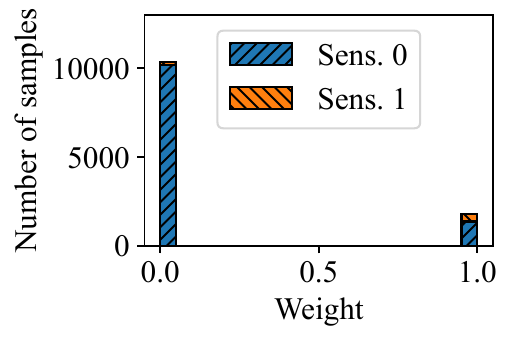}
        \caption{Task 4.}
    \end{subfigure}
    \begin{subfigure}[t]{0.225\textwidth}
        \includegraphics[width=\linewidth]{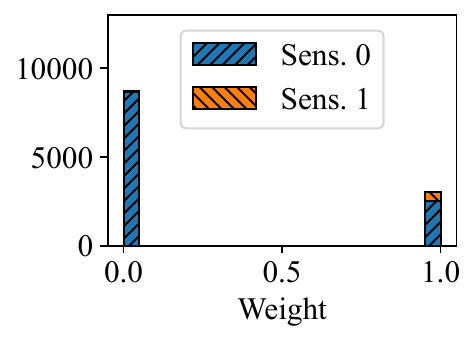}
        \caption{Task 5.}
    \end{subfigure}
    \caption{Distribution of sample weights for EO in sequential tasks of the Biased MNIST dataset.}\label{fig:weights_exp2}
\end{figure}

\begin{table}[t]
    \setlength{\tabcolsep}{2.05pt}
    \caption{Average counts of binary (0 or 1) and non-binary (not 0 or 1) sample weights for each optimization task. Since we take averages of all tasks except the first, the sum of (\# Binary) and (\# Non-binary) is not necessarily an integer.}
    \vspace{-0.cm}
    \centering
    \begin{tabular}{l|C{1.2cm}C{1.2cm}C{1.6cm}C{1.6cm}}
    \toprule
        {Dataset (Metric)} & {\sf MNIST (EER)} & {\sf FMNIST (EER)} & {\sf Biased MNIST (EO)} & {\sf Biased MNIST (DP)} \\
        \midrule
        {\# Binary} & 11830.8 & 11997.4 & 11831.9 & 11831.6 \\
        {\# Non-binary} & 3.0 & 2.6 & 1.9 & 2.1 \\
    \bottomrule
    \end{tabular}
    
    \vspace{0.5em} 
    
    \begin{tabular}{l|C{1.2cm}C{1.2cm}C{1.6cm}C{1.6cm}}
    \toprule
        {Dataset (Metric)} & {\sf Drug (EO)} & {\sf Drug (DP)} & {\sf BiasBios (EO)} & {\sf BiasBios (DP)} \\
        \midrule
        {\# Binary} & 446.5 & 446.1 & 25274.7 & 25273.2 \\
        {\# Non-binary} & 0.0 & 0.4 & 2.0 & 3.5 \\
    \bottomrule
    \end{tabular}
    \label{tbl:binaryanalysis_main}
\end{table}

\subsection{Sample Weighting Analysis}
\label{subsec:analysis}

To answer {\bf RQ2}, we analyze how our \method{} algorithm weights the current task samples at each task
in Fig.~\ref{fig:weights_exp2} and Sec.~\refappendix{subsec:sampleweighting}.
As the acquired sample weights may change with epochs during training, we show the average weight distribution of sensitive groups over all epochs. Since \method{} is not applied to the first task, where the model is trained with only the current task data, we present results starting from the second task. 

The result shows that \method{} assigns higher weights on average to the underperforming group (Sensitive group 1 in Fig.~\ref{fig:weights_exp2}) compared to the overperforming group (Sensitive group 0 in Fig.~\ref{fig:weights_exp2}). 
The weights are computed by considering complex forgetting relationships between sensitive groups, which differs from simply assigning higher weights to underperforming groups. 
Note that the acquired sample weights are mostly close to 0 or 1, which are extreme values. As optimal solutions of a linear program are obtained at the vertices of the feasible region, a convex polytope defined by $\mathbf{w}_l \in [0, 1]^{|T_l|}$ (as formulated in Sec.~\refappendix{sec:sampleweighting}), each \(w\) coincides with its extreme values. However, there are also some values that do not lie at these boundaries (0 or 1). The number of binary (0 or 1) and non-binary (not 0 or 1) samples weights are shown in Table~\ref{tbl:binaryanalysis_main}.

We also emphasize that if we limit the solution of the optimization problem to binary, which is equivalent to sample selection, the problem would transform into a mixed-integer linear programming problem, which is NP-hard and cannot be solved efficiently.
We also observe that \method{} assigns zero weight to a significant number of samples, reducing the training data used. This weighting approach provides an additional advantage in enabling efficient model training while retaining accuracy and fairness.

\begin{table}[t]
  \setlength{\tabcolsep}{7.8pt}
  \caption{Accuracy and fairness results on the MNIST and Biased MNIST datasets with or without \method{}.}
  \centering
  \begin{tabular}{l|cccc}
  \toprule
    {Methods} & \multicolumn{2}{c}{\sf MNIST} & \multicolumn{2}{c}{\sf Biased MNIST} \\
    \cmidrule{1-5}
    {} & {Acc.} & {EER Disp.} & {Acc.} & {EO Disp.} \\
    \midrule
    {W/o \method{}} & .912\tiny{$\pm$.004} & .051\tiny{$\pm$.005} & \textbf{.910\tiny{$\pm$.003}} & .126\tiny{$\pm$.005} \\
    {\bf \method{}} & \textbf{.925\tiny{$\pm$.004}} & \textbf{.032\tiny{$\pm$.005}} & .909\tiny{$\pm$.004} & \textbf{.119\tiny{$\pm$.007}}\\
    \bottomrule
  \end{tabular}
  \label{tbl:performance_ablation}
\end{table}

\begin{table}[t]
  \setlength{\tabcolsep}{3.35pt}
  \caption{Accuracy and more fairness metrics (including EO, DP, and EER Disparity) results on the Biased MNIST dataset without \method{} and \method{} with respect to EO and DP disparity.} 
  \centering
  \begin{tabular}{l|cccc}
  \toprule
    {Methods} & {Acc.} & {EO Disp.} & {DP Disp.} & {EER Disp.} \\
    \midrule
    {W/o \method{}} & \textbf{.910\tiny{$\pm$.003}} & .126\tiny{$\pm$.005}& .009\tiny{$\pm$.001}& .032\tiny{$\pm$.002} \\
    {\method{} (w.r.t. EO Disp.)} & .909\tiny{$\pm$.004} & \textbf{.119\tiny{$\pm$.007}}& \textbf{.008\tiny{$\pm$.001}}& .032\tiny{$\pm$.003}\\
    {\method{} (w.r.t. DP Disp.)} & .904\tiny{$\pm$.004} & .124\tiny{$\pm$.008}& \textbf{.008\tiny{$\pm$.001}}& \textbf{.029\tiny{$\pm$.004}}\\
    \bottomrule
  \end{tabular}
  \label{tbl:performance_ablation_extension_BMNIST_main}
\end{table}

\subsection{Ablation Study}
\label{subsec:ablationstudy}

To show the effectiveness of \method{} on accuracy and fairness, we perform an ablation study comparing the performance of using \method{} versus using all current task samples for training with equal weights. 
The results for the MNIST and Biased MNIST datasets are shown in Table~\ref{tbl:performance_ablation}. The results for DP disparity and other datasets, which are similar, are in Sec.~\refappendix{appendix:ablation}. As a result, applying sample weighting to the current task data is necessary to improve fairness while maintaining comparable accuracy. 

We also investigate how optimizing for one fairness metric influences other metrics in Table~\ref{tbl:performance_ablation_extension_BMNIST_main} and Sec.~\refappendix{appendix:ablation_extension}. Optimizing \method{} for a particular fairness metric partially mitigates disparities in other metrics for EO and DP, since both metrics utilize explicitly defined sensitive attributes, though its effectiveness is generally the highest for the targeted metric. In contrast, \method{} does not effectively reduce EER disparity, where the sensitive attribute is defined based on the class label rather than an explicitly defined sensitive attribute. 

\begin{table}[t]
  \setlength{\tabcolsep}{6.6pt}
  \caption{Accuracy and fairness (DP disparity) results when combining fair post-processing technique ($\epsilon$-fair) with continual learning methods ({\it iCaRL}, {\it WA}, {\it OCS}, and \method{}).}
  \centering
  \begin{tabular}{l|cccc}
  \toprule
    {Methods} & \multicolumn{2}{c}{\sf Biased MNIST} & \multicolumn{2}{c}{\sf DRUG} \\
    \cmidrule{1-5}
    {} & {Acc.} & {DP Disp.} & {Acc.} & {DP Disp.} \\
    \midrule
    {iCaRL} & .802\tiny{$\pm$.008} & .015\tiny{$\pm$.001} & \textbf{.444\tiny{$\pm$.025}} & .093\tiny{$\pm$.009} \\
    {WA} & .916\tiny{$\pm$.002} & .009\tiny{$\pm$.001} & .408\tiny{$\pm$.022} & .067\tiny{$\pm$.013} \\
    {OCS} & .824\tiny{$\pm$}.007 & .035\tiny{$\pm$}.003 & .393\tiny{$\pm$.017} & .053\tiny{$\pm$.012} \\
    {\bf \method{}} & .904\tiny{$\pm$.004} & .008\tiny{$\pm$.001} & .405\tiny{$\pm$.013} & .043\tiny{$\pm$.004} \\
    \cmidrule{1-5}
    {iCaRL -- $\epsilon$-fair} & \underline{.944\tiny{$\pm$.008}} & \underline{.006\tiny{$\pm$.002}} & \underline{.427\tiny{$\pm$.018}} & \underline{.026\tiny{$\pm$.004}} \\
    {WA -- $\epsilon$-fair} & \textbf{.953\tiny{$\pm$.003}} & \underline{.006\tiny{$\pm$.002}} & .404\tiny{$\pm$}.021 & .044\tiny{$\pm$}.020 \\
    {OCS -- $\epsilon$-fair} &  \textbf{.952\tiny{$\pm$}.003} &  .032\tiny{$\pm$}.004 & .384\tiny{$\pm$.009} & .051\tiny{$\pm$.002} \\
    {\bf \method{} -- $\epsilon$-fair} & .906\tiny{$\pm$}.006 & \textbf{.005\tiny{$\pm$}.001} & .405\tiny{$\pm$.013} & \textbf{.021\tiny{$\pm$.004}} \\
    \bottomrule
  \end{tabular}
  \label{tbl:performance_post}
\end{table}

\subsection{Integrating \method{} with Fair Post-processing}
\label{subsec:postprocessing}

To answer {\bf RQ3}, we emphasize the extensibility of \method{} by combining it with a state-of-the-art fair post-processing technique for multi-class tasks, $\epsilon$-fair\,\citep{denis2023fairness}. Since $\epsilon$-fair supports DP, we only report DP results of the Biased MNIST and DRUG datasets in Table~\ref{tbl:performance_post}. 
More results and explanations are in Sec.~\refappendix{appendix:postprocessing}.
Comparing original continual learning methods with their integration of $\epsilon$-fair reduces fairness disparity, while accuracy sometimes degrades. This result shows that combining the fair post-processing technique can further improve fairness after model training with continual learning methods. In addition, \method{} still shows a better accuracy-fairness tradeoff than existing continual learning methods, even when combined with the fair post-processing technique, compared to existing continual learning methods. 

\section{Conclusion}
\label{sec:conclusion}

We proposed \method{}, a fairness-aware sample weighting algorithm for class-incremental learning. Unlike conventional class-incremental learning, we demonstrated how training with all the current task data using equal weights may result in unfair catastrophic forgetting. 
We theoretically showed that the average gradient vector of the current task data should not oppose the average gradient vector of a sensitive group to avoid unfair forgetting. We then proposed \method{} as a solution to adjust the average gradient vector of the current task data, thereby achieving better accuracy-fairness tradeoff results. \method{} supports various group fairness measures by converting the optimization problem into a linear program. 
In our experiments, \method{} consistently outperformed baselines in fairness while maintaining comparable accuracy across diverse datasets from various domains. 
Future work includes generalizing to multiple sensitive attributes, exploring strategies to further reduce computational overhead, and possible strategies of \method{} when data drift occurs.
Further discussion and analysis of these future directions are provided in Sec.~\refappendix{appendix:future_work}.

\paragraph{Limitations} \label{limitation}
Efforts to improve fairness inevitably result in a decrease in accuracy. Despite this inherent trade-off, \method{} effectively minimizes the accuracy drop and achieves the best accuracy-fairness tradeoff among all other evaluated approaches (see Sec.~\ref{subsec:expresults} and ~\refappendix{appendix:tradeoff} for more details). 
Additionally, our theoretical analysis relies on a first-order Taylor approximation, which may introduce approximation errors. While higher-order Taylor approximations could theoretically yield more accurate results, they render the optimization problem non-convex and NP-hard. Empirical analysis further demonstrates that the approximation error diminishes rapidly over training epochs, indicating that our method effectively mitigates approximation inaccuracies introduced by the approximation over time (see Sec.~\ref{appendix:higher_order_approximation} and ~\refappendix{appendix:apx_error} for more details, respectively). 


\begin{acks}
This material is based on work that is partially funded by an unrestricted gift from Google.
This work was supported by the Institute of Information \& Communications Technology Planning \& Evaluation\,(IITP) grant funded by the Korea government\,(MSIT) (No.\ RS-2022-II220157, Robust, Fair, Extensible Data-Centric Continual Learning).
This work was supported by the Institute of Information \& Communications Technology Planning \& Evaluation\,(IITP) grant funded by the Korea government\,(MSIT) (No.\ RS-2024-00444862, Non-invasive near-infrared based AI technology for the diagnosis and treatment of brain diseases).
\end{acks}





\bibliographystyle{ACM-Reference-Format}
\balance
\bibliography{main}

\ifthenelse{\boolean{fullpaper}}{\clearpage
\nobalance 
\newpage
\appendix

\section{Technical Appendices}
\label{appendix:theory}

\subsection{Notations used in the paper} \label{appendix:notations}
Continuing from Sec.~\ref{para:notations}, Table~\ref{tbl:notations} summarizes the notations used in our paper to enhance readability.

\begin{table}[h!]
  \small
  \caption{Summary of notations used throughout the paper.} 
  \centering
  \begin{tabular}{l|C{0.75\linewidth}}
  \toprule
    Symbol & Description \\
    \cmidrule{1-2}
    {X} & Feature of a data point.\\
    {$y$ or y} & True label of a data point.\\
    {$\hat{y}$ or $\hat{\text{y}}$} & Predicted label of a data point.\\
    {$z$ or z} & Sensitive attribute of a data point.\\
    {$d$} & Data point including X and $y$.\\
    {$\mathbb{Y}$} & Set of all classes.\\
    {$\mathbb{Y}_c$} & Set of classes in current task.\\
    {$\mathbb{Z}$} & Set of sensitive attributes.\\
    \midrule
    {$L$} & Total number of tasks.\\
    {$l$} & Index of current task.\\
    {$T_l$} & Dataset for current task $l$ (including $d$).\\
    {$\mathbf{w}_l$} & Training weight of current task $l$.\\
    {$\mathcal{M}_l$} & Buffer data for current task $l$.\\
    {N} & Number of classes in each task.\\
    {M} & Size of buffer.\\
    {m} & Number of samples of buffer data per task.\\
    \midrule
    {$G$} & Sensitive group.\\
    {$G_y$} & Sensitive group with label $y$.\\
    {$G_{y,z}$} & Sensitive group with label $y$ and sensitive attribute $z$.\\
    {$G_{\mathbb{Y}}$} & Average of all $G_y$.\\
    \midrule
    {$\theta$} & Model parameters.\\
    {$f_\theta^{l-1}$} & Model for previous task $l-1$.\\
    {$f_\theta$} & Current Model.\\
    {$\ell$ or $\mathcal{L}$} & Loss function.\\
    {$\tilde{\ell}$} & Approximated loss function.\\
    {$\ell(f_\theta, G)$} & Average loss for the samples in the group $G$.\\
    {$\nabla\ell(f_\theta, G)$} & Average gradient vector for the samples in the group $G$.\\
    {$\nabla\ell(f_\theta, d)$} & Average gradient vector for a sample $d$.\\
    \midrule
    {$\eta$} & Model learning rate.\\
    {$\alpha$} & Hyperparameter that determines the model update rate during optimization.\\
    {$\lambda$} & Hyperparameter that balances accuracy and fairness.\\
    {$\tau$} & Hyperparameter that balances previous and current task samples.\\
    \bottomrule
  \end{tabular}
  \label{tbl:notations}
\end{table}

\subsection{Theoretical Analysis of Unfairness in Class-Incremental Learning} \label{appendix:linear-approximation}

Continuing from Sec.~\ref{sec:unfairforgetting}, we prove the lemma on the updated loss of a group of data after learning the current task data.
\begin{lemma}[Restated from Lemma~\ref{lem:CF}] \label{lem:CF_appendix}
Denote $G$ as a sensitive group containing features $X$ and true labels $y$. Also, denote $f_{\theta}^{l-1}$ as a previous model and $f_{\theta}$ as the updated model after training on the current task $T_l$. Let $\ell$ be any differentiable loss function (e.g., cross-entropy loss), and $\eta$ be a learning rate. Then, the loss of the sensitive group of data after training with a current task sample $d_i \in T_l$ is approximated as follows: 
\begin{equation*}
\tilde{\ell}(f_{\theta}, G) = \ell(f_{\theta}^{l-1}, G) - \eta \nabla_{\theta} \ell(f_{\theta}^{l-1}, G)^\top \nabla_{\theta} \ell(f_{\theta}^{l-1}, d_i),
\end{equation*}
where $\tilde{\ell}(f_{\theta}, G)$ is the approximated average loss between model predictions $f_{\theta}(X)$ and true labels $y$, whereas $\ell(f_{\theta}^{l-1}, G)$ is the exact average loss, $\nabla_{\theta} \ell(f_{\theta}^{l-1}, G)$ is the average gradient vector for the samples in the group $G$, and $\nabla_{\theta} \ell(f_{\theta}^{l-1}, d_i)$ is the gradient vector for a sample $d_i$, each with respect to the previous model $f_{\theta}^{l-1}$.
\end{lemma}
\begin{proof}
Assume that the loss can be well approximated by its first-order Taylor expansion around $\theta^{l-1}$.
We update the model using gradient descent with the current task sample $d_i \in T_l$ and learning rate $\eta$ as follows:
\begin{equation*}
\theta = \theta^{l-1} - \eta \nabla_{\theta} \ell(f_{\theta}^{l-1}, d_i).
\end{equation*}
Using the Taylor series approximation,
\begin{equation*}
\begin{aligned}
\tilde{\ell}(f_{\theta}, G) &= \ell(f_{\theta}^{l-1}, G) + \nabla_{\theta} \ell(f_{\theta}^{l-1}, G)^\top (\theta - \theta^{l-1}) \\
&= \ell(f_{\theta}^{l-1}, G) + \nabla_{\theta} \ell(f_{\theta}^{l-1}, G)^\top (- \eta \nabla_{\theta} \ell(f_{\theta}^{l-1}, d_i)) \\
&= \ell(f_{\theta}^{l-1}, G) - \eta \nabla_{\theta} \ell(f_{\theta}^{l-1}, G)^\top \nabla_{\theta} \ell(f_{\theta}^{l-1}, d_i).
\end{aligned}
\end{equation*}
If we update the model using all the current task data $T_l$, the equation is formulated as $\tilde{\ell}(f_{\theta}, G) = \ell(f_{\theta}^{l-1}, G) - \eta \nabla_{\theta} \ell(f_{\theta}^{l-1}, G)^\top \allowbreak \nabla_{\theta} \ell(f_{\theta}^{l-1}, T_l)$. 
Therefore, if the average gradient vectors of the sensitive group and the current task data have opposite directions, i.e., $\nabla_{\theta} \ell(f_{\theta}^{l-1}, G)^\top \nabla_{\theta} \ell(f_{\theta}^{l-1}, T_l) < 0$, learning the current task data increases the loss of the sensitive group data and finally leads to catastrophic forgetting.
\end{proof}
We next derive a sufficient condition for unfair forgetting.
\begin{theorem}[Restated from Theorem~\ref{thm:CFcondition}]\label{thm:CFcondition_appendix}
    Let $\ell$ be the cross-entropy loss and $G_1$ and $G_2$ the overperforming and underperforming sensitive groups of data, respectively. Also let $d_i$ be a training sample that satisfy the following conditions: $\ell(f_\theta^{l-1}, G_1) < \ell(f_\theta^{l-1}, G_2)$ while $\nabla_\theta \ell(f_\theta^{l-1}, G_1)^\top \nabla_\theta \ell(f_\theta^{l-1}, d_i) > 0$ and $\nabla_\theta \ell(f_\theta^{l-1}, G_2)^\top \nabla_\theta \ell(f_\theta^{l-1}, d_i) \allowbreak < 0$. Then $|\tilde{\ell}(f_\theta, G_1) - \tilde{\ell}(f_\theta, G_2)| > |\ell(f_\theta^{l-1}, G_1) - \ell(f_\theta^{l-1}, G_2)|$.
\end{theorem}
\begin{proof}
Using the derived equation in the lemma~\ref{lem:CF_appendix} $\tilde{\ell}(f_{\theta}, G) = \ell(f_{\theta}^{l-1}, G) - \eta \nabla_{\theta} \ell(f_{\theta}^{l-1}, G)^\top \nabla_{\theta} \ell(f_{\theta}^{l-1}, d_i)$, we compute the disparity of losses between the two groups $G_1$ and $G_2$ after the model update as follows:
\begin{equation*}
\begin{aligned}
\quad&|\tilde{\ell}(f_{\theta}, G_1) - \tilde{\ell}(f_{\theta}, G_2)| \\
=\; &|(\ell(f_{\theta}^{l-1}, G_1) - \eta \nabla_{\theta} \ell(f_{\theta}^{l-1}, G_1)^\top\nabla_{\theta} \ell(f_{\theta}^{l-1}, d_i)) \\
-\; &(\ell(f_{\theta}^{l-1}, G_2) - \eta {\nabla_{\theta} \ell(f_{\theta}^{l-1}, G_2)^\top} \nabla_{\theta} \ell(f_{\theta}^{l-1}, d_i))| \\
=\; & |(\ell(f_{\theta}^{l-1}, G_1) - \ell(f_{\theta}^{l-1}, G_2))\\
-\; & \eta ({\nabla_{\theta} \ell(f_{\theta}^{l-1}, G_1)^\top} \nabla_{\theta} \ell(f_{\theta}^{l-1}, d_i) - {\nabla_{\theta} \ell(f_{\theta}^{l-1}, G_2)^\top} \nabla_{\theta} \ell(f_{\theta}^{l-1}, d_i))|. \\
\end{aligned}
\end{equation*}
Since $\ell(f_\theta^{l-1}, G_1) < \ell(f_\theta^{l-1}, G_2)$, it leads to $\ell(f_\theta^{l-1}, G_1) - \ell(f_\theta^{l-1}, G_2)\allowbreak < 0$. Next, the two assumptions of ${\nabla_\theta \ell(f_\theta^{l-1}, G_1)^\top} \nabla_\theta \ell(f_\theta^{l-1}, d_i) > 0$ and ${\nabla_\theta \ell(f_\theta^{l-1}, G_2)^\top} \nabla_\theta \ell(f_\theta^{l-1}, d_i) < 0$ make $- \eta ({\nabla_{\theta} \ell(f_{\theta}^{l-1}, G_1)^\top}\allowbreak \nabla_{\theta} \ell(f_{\theta}^{l-1}, d_i) - {\nabla_{\theta} \ell(f_{\theta}^{l-1}, G_2)^\top} \nabla_{\theta} \ell(f_{\theta}^{l-1}, d_i)) < 0$. Since the two terms in the absolute value equation are both negative, \\
\begin{equation*}
\begin{aligned}
\quad&|\tilde{\ell}(f_{\theta}, G_1) - \tilde{\ell}(f_{\theta}, G_2)| \\
=\; &|\ell(f_{\theta}^{l-1}, G_1) - \ell(f_{\theta}^{l-1}, G_2)| \\
+\; &|- \eta ({\nabla_{\theta} \ell(f_{\theta}^{l-1}, G_1)^\top} \nabla_{\theta} \ell(f_{\theta}^{l-1}, d_i) - {\nabla_{\theta} \ell(f_{\theta}^{l-1}, G_2)^\top} \nabla_{\theta} \ell(f_{\theta}^{l-1}, d_i))|  \\
>\, &|\ell(f_{\theta}^{l-1}, G_1) - \ell(f_{\theta}^{l-1}, G_2)|.
\end{aligned}
\end{equation*}
We finally have $|\tilde{\ell}(f_\theta, G_1) - \tilde{\ell}(f_\theta, G_2)| > |\ell(f_\theta^{l-1}, G_1) - \ell(f_\theta^{l-1}, G_2)|$, which implies that fairness deteriorates after training on the current task data. 
\end{proof}

\subsection{Higher-order approximation for Taylor series}
\label{appendix:higher_order_approximation}

Continuing from Sec.~\ref{sec:unfairforgetting}, we can extend the approximated loss to a higher-order Taylor approximation to improve theoretical accuracy, but it makes the optimization non-convex and NP-hard. 

\begin{theorem}\label{thm:higher_order_approximation}
Employing a second-order Taylor approximation on Lemma~\ref{lem:CF} makes the optimization problem non-convex and NP-hard.
\end{theorem}

\begin{proof}
Using a second-order Taylor expansion around the current model parameters $\theta$, the approximated loss becomes: 
\[
\tilde{\ell}(G) = \ell(G) - \nabla_\theta \ell(G)^\top \mathbf{x} + \frac{1}{2} \mathbf{x}^\top H_\theta(G) \mathbf{x},
\]
where $H_\theta(G)$ is the Hessian (assumed to be PSD; we also prove a case when it is not, and we define $\mathbf{x}:= \frac{\eta}{|T|} \sum_{d_i \in T} \mathbf{w}^i \nabla_\theta \ell(d_i)$. 
This function is convex in $\mathbf{x}$ and thus convex in $\mathbf{w}$.
To encourage fairness, we penalize the difference between the two group losses:
\[
\mathcal{L}_{\text{fair}}(\mathbf{w}) := \left| \tilde{\ell}(G_1) - \tilde{\ell}(G_2) \right| = \left| c + b^\top \mathbf{x} + \frac{1}{2} \mathbf{x}^\top Q \mathbf{x} \right|,
\]
where $c$ and $b$ are appropriate constants and $Q = H_\theta(G_1) - H_\theta(G_2)$.
Introducing a slack variable $t$, we rewrite the fairness loss constraint as:
\[
- t \leq c + b^\top \mathbf{x} + \frac{1}{2} \mathbf{x}^\top Q \mathbf{x} \leq t.
\]

This yields a Quadratically Constrained Quadratic Program (QCQP) in $\mathbf{w}$.
The constraint involves a quadratic form with matrix $Q = H_\theta(G_1) - H_\theta(G_2)$. In general, $Q$ is indefinite unless $H_\theta(G_1) = H_\theta(G_2)$. Hence, at least one constraint is non-convex.

It is known that a QCQP, even with a single non-PSD constraint matrix, is non-convex and NP-hard.\,\citep{pardalos1991quadratic}
Although $\tilde{\ell}(G)$ is convex in $\mathbf{w}$, the fairness term introduces non-convexity. Thus, the overall optimization problem is non-convex and NP-hard.

Furthermore, if $H_\theta(G)$ is not PSD, then $\tilde{\ell}(G)$ itself becomes non-convex in $\mathbf{w}$, which also leads to an NP-hard optimization problem.
\end{proof}

\subsection{From Cross-Entropy Loss to Group Fairness Metrics}
\label{appendix:CE-as-approximator}
Continuing from Sec.~\ref{sec:unfairforgetting}, we explain how to approximate the group fairness metrics using cross-entropy loss disparity.
Existing works\,\citep{DBLP:conf/naacl/ShenHCBF22, DBLP:conf/iclr/Roh0WS21, roh2023drfairness} empirically demonstrated that using the cross-entropy loss disparity provides reasonable proxies for common group fairness metrics such as equalized odds (EO) and demographic parity (DP) disparity. In addition, we theoretically describe how minimizing the cost function for EO using the cross-entropy loss disparity (i.e., $\mathcal{L}_{EO} = \frac{1}{|\mathbb{Y}| |\mathbb{Z}|} \sum_{y \in \mathbb{Y}, z \in \mathbb{Z}} |\ell(f_{\theta}, G_{y, z}) - \ell(f_{\theta}, G_{y})|$ where $\ell$ is a cross-entropy loss) leads to ensuring EO disparity. \cite{DBLP:conf/naacl/ShenHCBF22} theoretically and empirically showed that using cross-entropy loss instead of the 0-1 loss (i.e., $\mathbbm{1}(\text{y}\neq\hat{\text{y}})$ where $\mathbbm{1}(\cdot)$ is an indicator function, which is equivalent to the probability of correct prediction) can still capture EO disparity in binary classification. We now prove how applying the cross-entropy loss disparity for EO can be extended to multi-class classification as follows:

Let $\mathbb{Y}$ denote the set of all class labels. For each sample $i$, let $\text{y}_i \in \mathbb{Y}$ be its true label, $\hat{\text{y}}_i \in \mathbb{Y}$ its predicted label, and $\text{z}_i$ its sensitive attribute. Define $m_{y,z} := \left| \left\{ i : \text{y}_i = y,\; \text{z}_i = z \right\} \right|$ as the size of a sensitive group.
Also, let
\(\left( \text{y}_i^1, \text{y}_i^2, \dots, \text{y}_i^{|\mathbb{Y}|} \right)^\top\) and 
\(\left( \hat{\text{y}}_i^1, \hat{\text{y}}_i^2, \dots, \hat{\text{y}}_i^{|\mathbb{Y}|} \right)^\top\) 
denote the one-hot encoded vector of the true label $\text{y}_i$ and the predicted probability distribution over all classes for sample $i$, respectively, where $\text{y}_i^j \in \{0, 1\}$ and $\hat{\text{y}}_i^j \in [0,1]$ for $j \in [1,2,\dots,|\mathbb{Y}|]$.
Then, the cross-entropy loss for a sensitive group $G_{y, z}$ can be transformed as follows:
\begin{equation*}
\begin{aligned}
\ell(f_{\theta}, G_{y, z}) &= - \frac{1}{m_{y,z}} \sum_{i=1}^{m_{y,z}}\Big( \sum_{j=1}^{|\mathbb{Y}|} \text{y}_{i}^{j}\cdot\log(\hat{\text{y}}_{i}^{j})\Big) \\ 
&= - \frac{1}{m_{y,z}} \sum_{i=1}^{m_{y,z}} \log(\hat{\text{y}}_{i}^{y}).
\end{aligned}
\end{equation*}
Since $\hat{\text{y}}_{i}^{y}$ is equivalent to $p(\hat{\text{y}}_i=y)$ and we are measuring a loss for the sensitive group (\text{y} = $y$, \text{z} = $z$), $\ell(f_{\theta}, G_{y, z})=-\frac{1}{m_{y,z}} \sum_{i} \log(p(\hat{\text{y}}_i))$ is an unbiased estimator of $- \log p(\hat{\text{y}}|\text{y} = y, \text{z} = z)$. Likewise, $\ell(f_{\theta}, G_{y})$ is an unbiased estimator of $- \log p(\hat{\text{y}}|\text{y} = y)$, and our cost function becomes equivalent to $\Big|\log\frac{p(\hat{\text{y}}|\text{y} = y)}{p(\hat{\text{y}}|\text{y} = y, \text{z} = z)}\Big|$. Since $\frac{p(\hat{\text{y}}|\text{y} = y)}{p(\hat{\text{y}}|\text{y} = y, \text{z} = z)} \allowbreak = 1$ for all $y$, $z$ implies that prediction outcomes are independent of the sensitive attribute given the true label, we conclude that minimizing the cost function for EO can satisfy equalized odds. 

We next perform experiments to evaluate how well the cost function for EO approximates EO disparity (i.e., $\frac{1}{|\mathbb{Y}| |\mathbb{Z}|} \sum_{y \in \mathbb{Y}, z \in \mathbb{Z}} |\Pr(\hat{\text{y}}=y|\text{y}=y, \text{z}=z) - \Pr(\hat{\text{y}}=y|\text{y}=y)|$) on the Biased MNIST dataset as shown in Fig.~\ref{fig:CE_01}. Although the scales of the two metrics are different, the simultaneous movements of these two trends suggest that our cost function is effective in satisfying equalized odds.

\begin{figure}[H]
\centering
  \includegraphics[width=1.0\columnwidth]{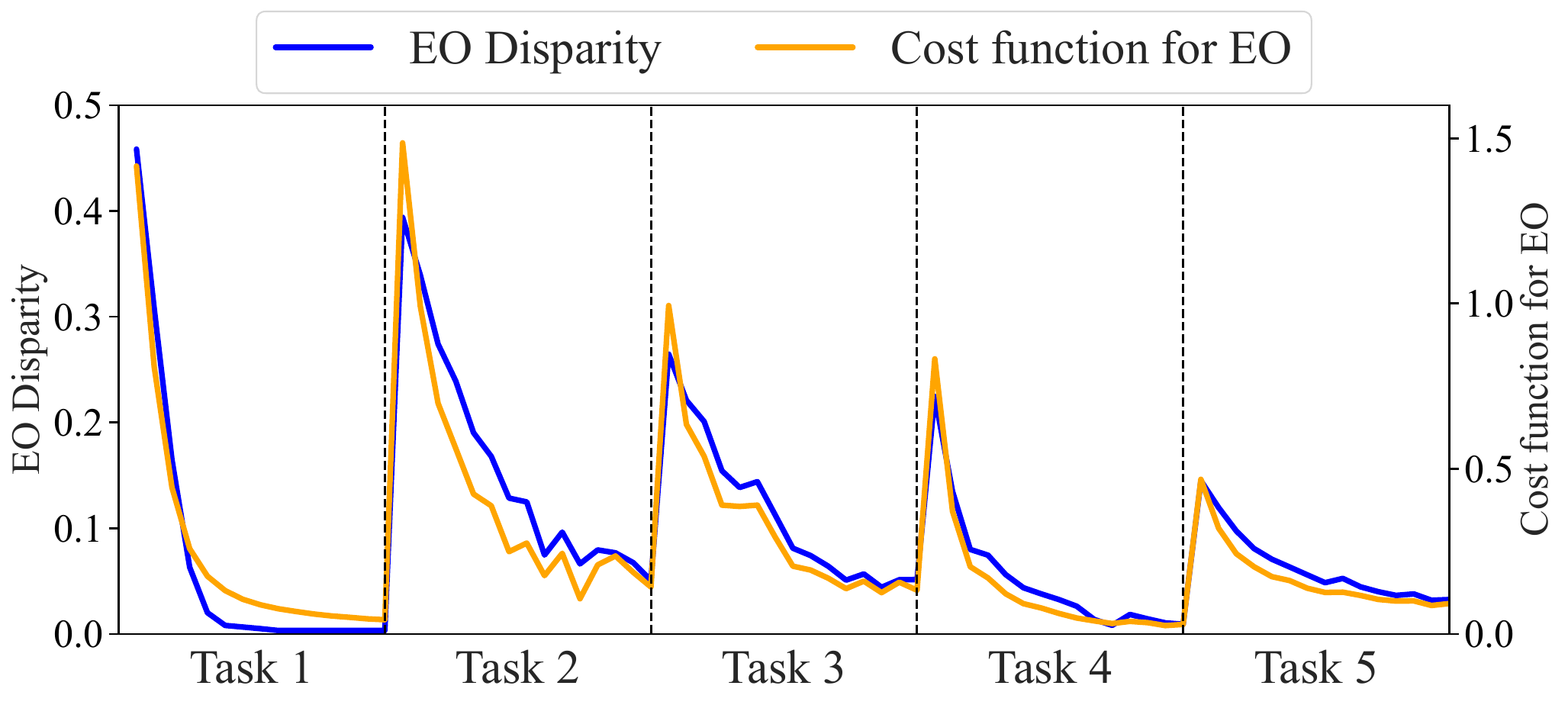}
  \caption{Comparison of EO disparity and cost function for EO during training on the Biased MNIST dataset. We train a model for 15 epochs per task.}
  \label{fig:CE_01}
\end{figure}

We can also extend our theoretical claim to the cost function for EER and DP. In case of EER, $\ell(f_{\theta}, G_y)$ and $\ell(f_{\theta}, G_{\mathbb{Y}})$ are unbiased estimators of $- \log p(\hat{\text{y}}|\text{y} = y)$ and $-\log p(\hat{\text{y}})$, respectively. 
Similarly, for DP, $\ell'(f_{\theta}, G_{y, z})$ and $\ell'(f_{\theta}, G_y)$ are unbiased estimators of $-\log p(\hat{\text{y}}|\text{z}=z)$ and $-\log p(\hat{\text{y}})$, respectively.

\subsection{Derivation of a Sufficient Condition for Demographic Parity in the Multi-Class Setting}
\label{appendix:multi-dp}
Continuing from Sec.~\ref{para:DP}, we derive a sufficient condition for satisfying demographic parity in the multi-class setting.

\begin{proposition}
In the multi-class setting, $\frac{m_{y,z_1}}{m_{*,z_1}} \ell(f_{\theta}, G_{y, z_1}) = \frac{m_{y,z_2}}{m_{*,z_2}} \ell(f_{\theta}, G_{y, z_2})$ where $m_{y,z} := |\{i: \text{y}_i = y, \text{z}_i = z\}|$ and $m_{*,z} := |\{i: \text{z}_i = z\}|$ for $y \in \mathbb{Y}$ and $z_1, z_2 \in \mathbb{Z}$ can serve as a sufficient condition for demographic parity.
\end{proposition}
\begin{proof}
In the multi-class setting, we can extend the definition of demographic parity as $\Pr(\hat{\text{y}} = y|\text{z}=z_1) = \Pr(\hat{\text{y}} = y|\text{z}=z_2)$ for $y \in \mathbb{Y}$ and $z_1, z_2 \in \mathbb{Z}$. The term $\Pr(\hat{\text{y}} = y|\text{z}=z)$ can be decomposed as follows: $\Pr(\hat{\text{y}} = y|\text{z}=z) = \Pr(\hat{\text{y}} = y, \text{y}=y|\text{z}=z) + \sum_{y_n \neq y} \Pr(\hat{\text{y}} = y, \text{y}=y_n|\text{z}=z)$. Without loss of generality, we set $z_1 = 0$ and $z_2 = 1$. Then the definition of demographic parity in the multi-class setting now becomes
\begin{equation*}
\begin{aligned}
&\Pr(\hat{\text{y}} = y, \text{y}=y|\text{z}=0) + \sum_{y_n \neq y} \Pr(\hat{\text{y}} = y, \text{y}=y_n|\text{z}=0) \\ 
=&\Pr(\hat{\text{y}} = y, \text{y}=y|\text{z}=1) + \sum_{y_n \neq y} \Pr(\hat{\text{y}} = y, \text{y}=y_n|\text{z}=1).
\end{aligned}
\end{equation*}
The term $\Pr(\hat{\text{y}} = y, \text{y}=y|\text{z}=0)$ can be represented with the 0-1 loss as follows:
\begin{equation*}
\begin{aligned}
\Pr(\hat{\text{y}} = y, \text{y}=y|\text{z}=0) &= \frac{\Pr(\hat{\text{y}} = y, \text{y}=y, \text{z}=0)}{\Pr(\text{z}=0)} \\
&= \frac{\Pr(\hat{\text{y}} = y | \text{y}=y, \text{z}=0) \Pr(\text{y}=y, \text{z}=0)}{\Pr(\text{z}=0)} \\
&= \frac{1}{m_{*,0}} \sum_{i:y_i=y,z_i=0} (1 - \mathbbm{1}(y_i \neq \hat{y}_i)),
\end{aligned}
\end{equation*}
where $\mathbbm{1}(\cdot)$ is an indicator function. Similarly, $\Pr(\hat{\text{y}} = y, \text{y}=y_n|\text{z}=0)$ for $y_n \neq y$ can be transformed as follows: 
\begin{equation*}
\begin{aligned}
\Pr(\hat{\text{y}} = y, \text{y}=y_n|\text{z}=0) &= \frac{\Pr(\hat{\text{y}} = y, \text{y}=y_n, \text{z}=0)}{\Pr(\text{z}=0)} \\
&= \frac{\Pr(\hat{\text{y}} = y | \text{y}=y_n, \text{z}=0) \Pr(\text{y}=y_n, \text{z}=0)}{\Pr(\text{z}=0)} \\
&= \frac{1}{m_{*,0}} \sum_{j:y_j=y_n,z_j=0} \mathbbm{1}(y_j \neq \hat{y}_j).
\end{aligned}
\end{equation*}
By applying the same technique to $\Pr(\hat{\text{y}} = y, \text{y}=y|\text{z}=1)$ and $\Pr(\hat{\text{y}} = y, \text{y}=y_n|\text{z}=1)$, we have the 0-1 loss-based definition of demographic parity:
\begin{equation*}
\begin{aligned}
\quad&\frac{1}{m_{*,0}} \sum_{i:y_i=y,z_i=0} (1 - \mathbbm{1}(y_i \neq \hat{y}_i)) + \sum_{i:y_i \neq y} \frac{1}{m_{*,0}} \sum_{j:y_j= y_i,z_j=0} \mathbbm{1}(y_j \neq \hat{y}_j) \\
=\;&\frac{1}{m_{*,1}} \sum_{i:y_i=y,z_i=1} (1 - \mathbbm{1}(y_i \neq \hat{y}_i)) + \sum_{i:y_i \neq y} \frac{1}{m_{*,1}} \sum_{j:y_j= y_i,z_j=1} \mathbbm{1}(y_j \neq \hat{y}_j).
\end{aligned}
\end{equation*}
Since the 0-1 loss is not differentiable, it is not suitable to approximate the loss using gradients as in Eq.~\ref{eq:loss_update}. We thus approximate the 0-1 loss to a standard loss function $\ell$ (e.g., cross-entropy loss), 

\begin{equation*}
\begin{aligned}
\quad&\frac{1}{m_{*,0}} \sum_{i:y_i=y,z_i=0} -\ell(f_{\theta}, d_i) + \sum_{i:y_i \neq y} \frac{1}{m_{*,0}} \sum_{j:y_j= y_i,z_j=0} \ell(f_{\theta}, d_j) \\
=\;&\frac{1}{m_{*,1}} \sum_{i:y_i=y,z_i=1} -\ell(f_{\theta}, d_i) + \sum_{i:y_i \neq y} \frac{1}{m_{*,1}} \sum_{j:y_j= y_i,z_j=1} \ell(f_{\theta}, d_j),
\end{aligned}
\end{equation*}
where $\ell(f_{\theta}, d_j)$ is the loss between the model prediction $f_{\theta}(d_j)$ and the true label $y_j$. By replacing $\sum_{i:y_i=y,z_i=z} \ell(f_{\theta}, d_i)$ with $m_{y,z} \ell(f_{\theta}, G_{y,z})$,
\begin{equation*}
\begin{aligned}
\quad&\frac{m_{y,0}}{m_{*,0}} (-\ell(f_{\theta}, G_{y,0})) + \sum_{i:y_i \neq y} \frac{m_{y_i,0}}{m_{*,0}} \ell(f_{\theta}, G_{y_i, 0}) \\
=\;& \frac{m_{y,1}}{m_{*,1}} (-\ell(f_{\theta}, G_{y,1})) + \sum_{i:y_i \neq y} \frac{m_{y_i,1}}{m_{*,1}} \ell(f_{\theta}, G_{y_i, 1}).
\end{aligned}
\end{equation*}
To satisfy the constraint for all $y \in \mathbb{Y}$, the corresponding terms on the left-hand side and the right-hand side of the equation should be equal, i.e., $\frac{m_{y,0}}{m_{*,0}} \ell(f_{\theta}, G_{y,0}) = \frac{m_{y,1}}{m_{*,1}} \ell(f_{\theta}, G_{y,1})$. In general, we derive a sufficient condition for demographic parity as $\frac{m_{y,z_1}}{m_{*,z_1}} \ell(f_{\theta}, G_{y, z_1}) = \frac{m_{y,z_2}}{m_{*,z_2}} \ell(f_{\theta}, G_{y, z_2})$. Note that the number of samples in sensitive groups (e.g., $m_{*, z}$ and $m_{y, z}$) is derived from the definition of demographic parity, which is independent of sample weights.
\end{proof}

\subsection{LP Formulation of Fairness-aware Optimization Problems}\label{appendix:Fairness_LP}
Continuing from Sec.~\ref{para:optim_to_lp}, we prove Theorem~\ref{thm:OptimToLP}, which implies that fairness-aware optimization problems can be transformed into linear programming problems. This transformation is made possible by using Lemma~\ref{lem:min_absolute_with_linear}, which suggests that minimizing the sum of absolute values with linear terms can be transformed into a linear programming form.

\begin{lemma}\label{lem:min_absolute_with_linear} 
The following optimization problem can be reformulated into a linear programming form. Note that in the following equation, y and z refer to arbitrary variables, not to the label or sensitive attribute, respectively.
\begin{gather*}
\min_{\mathbf{x}}\sum_{i=1}^n |y_i| + z_i \\
\begin{aligned}
    \mathit{s.t.}& \quad y_i = a_i - {\mathbf{b}_i^\top} \mathbf{x}, \quad z_i = c_i - {\mathbf{d}_i^\top} \mathbf{x}
    \\& a_i, c_i, y_i, z_i \in \mathbb{R}, \,\, \mathbf{b}_i, \mathbf{d}_i \in {\mathbb{R}^{m\times1}}, \,\, \forall i \in \{1, \ldots, n\}, \,\, \mathbf{x} \in {[0, 1]^{m\times1}}.
\end{aligned}
\end{gather*}
\end{lemma}
\begin{proof}
The transformation for minimizing the sum of absolute values was introduced in\,\citep{robert1958linear, mccarl1997applied, math10020283}. Note that considering the additional affine term does not affect the flow of the proof. We first substitute $y_i$ for $y_i^+ - y_i^-$ where both $y_i^+$ and $y_i^-$ are nonnegative. Then, the optimization problem becomes
\begin{gather*}
\min_{\mathbf{x}}\sum_{i=1}^n|y_i^+ - y_i^-| + z_i \\
\begin{aligned}
    \mathit{s.t.}& \quad y_i^+ - y_i^- = a_i - {\mathbf{b}_i^\top} \mathbf{x}, \quad z_i = c_i - {\mathbf{d}_i^\top} \mathbf{x}, 
    \\ & y_i^+ - y_i^- = y_i, \quad y_i^+, y_i^- \in \mathbb{R}^+, 
    \\ & a_i, c_i, y_i, z_i \in \mathbb{R}, \,\, \mathbf{b}_i, \mathbf{d}_i \in {\mathbb{R}^{m\times1}}, \,\, \forall i \in \{1, \ldots, n\}, \,\, \mathbf{x} \in {[0, 1]^{m\times1}}.
\end{aligned}
\end{gather*}
This problem is still nonlinear. However, the absolute value terms can be simplified when either $y_i^+$ or $y_i^-$ equals to zero (i.e., $y_i^+y_i^-=0$), as the consequent absolute value reduces to zero plus the other term. Then, the absolute value term can be written as the sum of two variables,
\begin{gather*}
    |y_i^+ - y_i^-| = |y_i^+| + |y_i^-| = y_i^+ + y_i^- \quad \text{if} \quad y_i^+y_i^-=0.
\end{gather*}
By using the assumption, the formulation becomes
\begin{gather*}
\min_{\mathbf{x}}\sum_{i=1}^n y_i^+ + y_i^- + z_i \\
\begin{aligned}
    \mathit{s.t.}& \quad y_i^+ - y_i^- = a_i - {\mathbf{b}_i^\top} \mathbf{x}, \quad z_i = c_i - {\mathbf{d}_i^\top} \mathbf{x}, 
    \\ & y_i^+ - y_i^- = y_i, \quad \underline{y_i^+ y_i^- = 0}, \quad y_i^+, y_i^- \in \mathbb{R}^+, 
    \\ & a_i, c_i, y_i, z_i \in \mathbb{R}, \,\, \mathbf{b}_i, \mathbf{d}_i \in {\mathbb{R}^{m\times1}}, \,\, \forall i \in \{1, \ldots, n\}, \,\, \mathbf{x} \in {[0, 1]^{m\times1}},
\end{aligned}
\end{gather*}
with the underlined condition added. However, this condition can be dropped. Assume there exist $y_i^+$ and $y_i^-$, which do not satisfy $y_i^+ y_i^- = 0$. When $y_i^+ \geq y_i^- > 0$, there exists a better solution ($y_i^+ - y_i^-$, $0$) instead of ($y_i^+$, $y_i^-$), which satisfies all the conditions, but has a smaller objective function value $y_i^+ - y_i^- + 0 + z_i < y_i^+ + y_i^- + z_i$. For the case of $y_i^- > y_i^+ > 0$, a solution ($0$, $y_i^- - y_i^+$) works better for similar reasons. Thus, the minimization automatically leads to $y_i^+ y_i^- = 0$, and the underlined nonlinear constraint becomes unnecessary. Consequently, the final formulation becomes LP:
\begin{gather*}
\min_{\mathbf{x}}\sum_{i=1}^n y_i^+ + y_i^- + z_i \\
\begin{aligned}
    \mathit{s.t.}& \quad y_i^+ - y_i^- = a_i - {\mathbf{b}_i^\top} \mathbf{x}, \quad z_i = c_i - {\mathbf{d}_i^\top} \mathbf{x}, 
    \\ & y_i^+ - y_i^- = y_i  y_i^+, \quad y_i^- \in \mathbb{R}^+,
    \\ & a_i, c_i, y_i, z_i \in \mathbb{R}, \,\, \mathbf{b}_i, \mathbf{d}_i \in {\mathbb{R}^{m\times1}}, \,\, \forall i \in \{1, \ldots, n\}, \,\, \mathbf{x} \in {[0, 1]^{m\times1}}.
\end{aligned}
\end{gather*}
\end{proof}
Applying Lemma~\ref{lem:min_absolute_with_linear}, we next prove Theorem~\ref{thm:OptimToLP}.
By using the result of Theorem~\ref{thm:OptimToLP}, we show that the fairness-aware optimization problems, where the objective function includes both fairness ($\mathcal{L}_{fair}$) and accuracy ($\mathcal{L}_{acc}$) losses, can be transformed into linear programming (LP) problems.

\begin{theorem}[Restated from Theorem~\ref{thm:OptimToLP}]\label{thm:OptimToLP_appendix}
The fairness-aware optimization problems (Eqs.~\ref{eq:eer},~\ref{eq:eo}, and~\ref{eq:dp}) can be formulated as linear programming (LP) problems.
\end{theorem}

\begin{proof}
For every update of the model, the corresponding loss of each group can be approximated linearly in the same way as in Sec.~\ref{appendix:linear-approximation}: $\tilde{\ell}(f_{\theta}, G) = \ell(f_{\theta}^{l-1}, G) - \eta \nabla_{\theta} \ell(f_{\theta}^{l-1}, G)^\top \nabla_{\theta} \ell(f_{\theta}^{l-1}, T_l)$. 
With a technique of sample weighting for the current task data, $\nabla_{\theta} \ell(f_{\theta}^{l-1}, T_l)$ can be changed as $\frac{1}{|T_l|} \sum_{d_i \in T_l} \mathbf{w}_l^i \nabla_{\theta} \ell(f_{\theta}^{l-1}, d_i)$ where $\mathbf{w}_l^i$ represents a training weight for the current task sample $d_i$. 

We believe that this transformation is natural and valid, as models are generally updated using the average gradient of training data, formulated as $\frac{1}{|T_l|} \sum_{d_i \in T_l} \nabla_{\theta} \ell(f_{\theta}^{l-1}, d_i)$, and a training weight is additionally assigned to each sample for weighting. Here, $|T_l|$ is the number of samples in the current task data, which is independent of the fairness notions considered. Note that if the normalization coefficient $\frac{1}{|T_l|}$ is replaced with $\frac{1}{\sum \mathbf{w}_l^i}$, the revised equation cannot handle the case where all weights are zero. Also, our revised optimization problems of Eqs.~\ref{eq:eer},~\ref{eq:eo}, and~\ref{eq:dp} would no longer be LP. \\ 
Thus, $\tilde{\ell}(f_{\theta}, G)$ can be rewritten as follows:

\begin{equation*}
    \begin{aligned}
        \tilde{\ell}(f_{\theta}, G) 
        &= \ell(f_{\theta}^{l-1}, G) - \eta {\nabla_{\theta} \ell(f_{\theta}^{l-1}, G)^\top \biggl(\frac{1}{|T_l|} \sum_{d_i \in T_l} \mathbf{w}_l^i \nabla_{\theta} \ell(f_{\theta}^{l-1}, d_i)\biggr)}\\
        &= \ell(f_{\theta}^{l-1}, G) - {\frac{\eta}{|T_l|}\nabla_{\theta} \ell(f_{\theta}^{l-1}, G)^\top}\bigl\langle[\,\ldots,\;,\nabla_{\theta} \ell(f_{\theta}^{l-1}, d_i),\;\ldots\,]
        \\&\qquad\qquad\qquad\qquad\qquad\qquad\qquad,[\,\ldots,\;w_l^i,\;\ldots\,]^\top\bigr\rangle\\
        &= a_G - {\mathbf{b}_G^\top} \mathbf{w},
    \end{aligned}
\end{equation*}
where $a_G:=\ell(f_{\theta}^{l-1}, G)$ and $\mathbf{b}_G:=\frac{\eta}{|T_l|}[\,\ldots,\;,\nabla_{\theta} \ell(f_{\theta}^{l-1}, d_i),\;\ldots\,]^\top \allowbreak \cdot \nabla_{\theta} \ell(f_{\theta}^{l-1}, G)$ are a constant and a vector with constants, respectively, and $\mathbf{w}:= [\,\ldots,\;w_l^i,\;\ldots\,]^\top$ is a variable where $w_l^i\in[0, 1]$.

\textbf{Case 1}. If target fairness measure is EER ($\mathcal{L}_{fair} = \mathcal{L}_{EER}$), \\
\begin{equation*}
    \begin{aligned}
        \quad&\mathcal{L}_{EER} +  \lambda \mathcal{L}_{acc} \\
        =\;& \frac{1}{|\mathbb{Y}|} \sum_{y \in \mathbb{Y}} | \tilde{\ell}(f_{\theta}, G_y) - \tilde{\ell}(f_{\theta}, G_{\mathbb{Y}}) | + \lambda  \frac{1}{|\mathbb{Y}_c|} \sum_{y \in \mathbb{Y}_c} \tilde{\ell}(f_{\theta}, G_y) \\
        =\;& \frac{1}{|\mathbb{Y}|} \sum_{y \in \mathbb{Y}} | (a_{G_y} - a_{G_\mathbb{Y}}) - {(\mathbf{b}_{G_y} - \mathbf{b}_{G_\mathbb{Y}})^\top} \mathbf{w} | \\
        +\;& \lambda  \frac{1}{|\mathbb{Y}_c|} \sum_{y \in \mathbb{Y}_c} (a_{G_y} - {\mathbf{b}_{G_y}^\top}\mathbf{w}).
    \end{aligned}
\end{equation*}

\textbf{Case 2}. If target fairness measure is EO ($\mathcal{L}_{fair} = \mathcal{L}_{EO}$), \\
\begin{equation*}
    \begin{aligned}
        \quad&\mathcal{L}_{EO} + \lambda \mathcal{L}_{acc} \\
        =\;& \frac{1}{|\mathbb{Y}| |\mathbb{Z}|} \sum_{y \in \mathbb{Y}, z \in \mathbb{Z}} |\tilde{\ell}(f_{\theta}, G_{y, z}) - \tilde{\ell}(f_{\theta}, G_{y})| \\
        +\;& \lambda  \frac{1}{|\mathbb{Y}_c||\mathbb{Z}|} \sum_{y \in \mathbb{Y}_c, z \in \mathbb{Z}} \tilde{\ell}(f_{\theta}, G_{y, z}) \\
        =\;& \frac{1}{|\mathbb{Y}| |\mathbb{Z}|} \sum_{y \in \mathbb{Y}, z \in \mathbb{Z}} | (a_{G_{y, z}} - a_{G_{y}}) - {(\mathbf{b}_{G_{y, z}} - \mathbf{b}_{G_{y}})^\top} \mathbf{w} | \\
        +\;& \lambda \frac{1}{|\mathbb{Y}_c||\mathbb{Z}|} \sum_{y \in \mathbb{Y}_c, z \in \mathbb{Z}} (a_{G_{y, z}} - {\mathbf{b}_{G_{y, z}}^\top}\mathbf{w}).
    \end{aligned}
\end{equation*}

\textbf{Case 3}. If target fairness measure is DP ($\mathcal{L}_{fair} = \mathcal{L}_{DP}$), \\
\begin{equation*}
    \begin{aligned}
        \quad&\mathcal{L}_{DP} + \lambda \mathcal{L}_{acc} \\
        =\;& \frac{1}{|\mathbb{Y}| |\mathbb{Z}|} \sum_{y \in \mathbb{Y}, z \in \mathbb{Z}} |\tilde{\ell}'(f_{\theta}, G_{y, z}) - \tilde{\ell}'(f_{\theta}, G_{y})| \\
        +\;& \lambda  \frac{1}{|\mathbb{Y}_c||\mathbb{Z}|} \sum_{y \in \mathbb{Y}_c, z \in \mathbb{Z}} \tilde{\ell}(f_{\theta}, G_{y, z}) \\
        =\;& \frac{1}{|\mathbb{Y}| |\mathbb{Z}|} \sum_{y \in \mathbb{Y}, z \in \mathbb{Z}} | (a'_{G_{y, z}} - a'_{G_{y}}) - {(\mathbf{b}'_{G_{y, z}} - \mathbf{b}'_{G_{y}})^\top} \mathbf{w} | \\
        +\;&\lambda \frac{1}{|\mathbb{Y}_c||\mathbb{Z}|} \sum_{y \in \mathbb{Y}_c, z \in \mathbb{Z}} (a_{G_{y, z}} - {\mathbf{b}_{G_{y, z}}^\top}\mathbf{w}),
    \end{aligned}
\end{equation*}
where $a'_{G_{y, z}}:=\frac{m_{y,z}}{m_{*,z}} a_{G_{y, z}}$, $a'_{G_{y}}:=\sum\limits_{z\in\mathbb{Z}}\frac{m_{y,z}}{m_{*,z}}a_{G_{y, z}}$, $\mathbf{b}'_{G_{y, z}}:=\frac{m_{y,z}}{m_{*,z}} \mathbf{b}_{G_{y, z}}$, $\mathbf{b}'_{G_{y}}:=\sum\limits_{z\in\mathbb{Z}}\frac{m_{y,z}}{m_{*,z}}\mathbf{b}_{G_{y, z}}$.

Since $a_G$ and $\mathbf{b}_G$ are composed of constant values, each equation above can be reformulated to a linear programming form by applying Lemma~\ref{lem:min_absolute_with_linear}.
\end{proof}

\paragraph{Remark (Non-triviality of the solution)}
\label{appendix:non_trivial_solution}
Although the weights $\mathbf{w}_l$ are constrained to $[0, 1]^{|T_l|}$, the optimization does not admit the trivial solution $\mathbf{w}=\mathbf{0}$ unless $\lambda=0$.
This is because the objective jointly optimizes both accuracy and fairness objectives: assigning zero weights eliminates the fairness term, while the prediction loss reduces to a constant term (the $a_G$ terms).

\paragraph{Remark (Generality of the fairness formulation)} 
\label{appendix:more_metrics}
Although we focus on EER, EO, and DP in this work, the formulation in Sec.~\ref{appendix:Fairness_LP} is not restricted to these specific metrics. More generally, any group fairness notion that can be expressed as a function of group-wise performance disparities admits a similar LP formulation, and can therefore be incorporated into our framework.

\clearpage
\newpage

\section{More Experimental Results}
\label{appendix:experiments}
\subsection{T-SNE Results for Real Datasets}
\label{subsec:tsne}
Continuing from Sec.~\ref{sec:intro}, we provide t-SNE results for real datasets to show that data overlapping between different classes also occurs in real scenarios, similar to the synthetic dataset results depicted in Fig.~\ref{fig:synthetic_dataset}. Using t-SNE, we project the high-dimensional data of the MNIST, FMNIST, Biased MNIST, and DRUG datasets into a lower-dimensional 2D space with $x_1$ and $x_2$, as shown in Fig.~\ref{fig:tsne_exp}. Since BiasBios is a text dataset that requires pre-trained embeddings to represent the data, we do not include the t-SNE results for it. In the MNIST dataset, the images with labels of 3 (red), 5 (brown), and 8 (yellow) exhibit similar characteristics and overlap, but belong to different classes. As another example, in the FMNIST dataset, the images of the classes `Sandal' (brown), `Sneaker' (gray), and `Ankel boot' (sky-blue) also have similar characteristics and overlap.

\begin{figure}[!t]
    \raggedright
    \begin{subfigure}[t]{0.4\textwidth}
        \includegraphics[width=\linewidth]{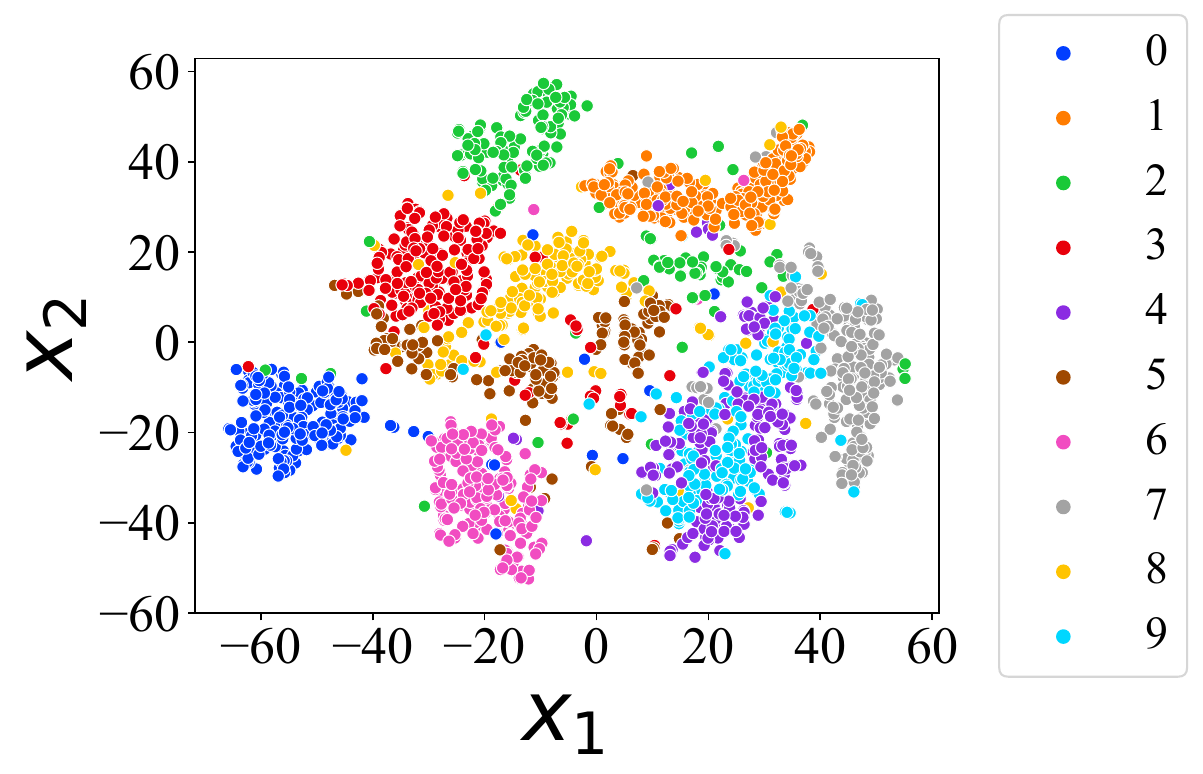}
        \caption{MNIST.}
        \label{fig:mnist_tsne}
    \end{subfigure}
    \hspace{0.03\textwidth}
    \begin{subfigure}[t]{0.47\textwidth}
        \includegraphics[width=\linewidth]{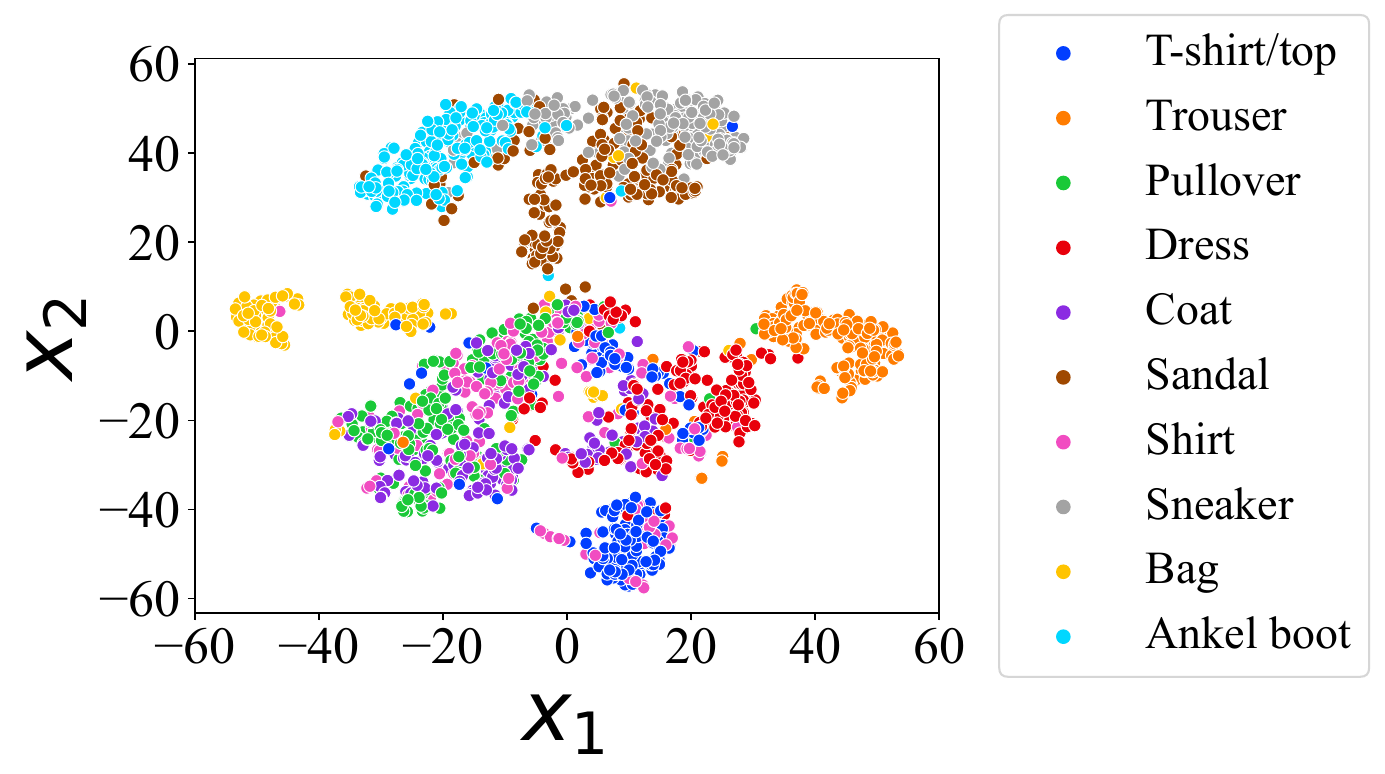}
        \caption{FMNIST.}
        \label{fig:fmnist_tsne}
    \end{subfigure}
    \begin{subfigure}[t]{0.4\textwidth}
        \includegraphics[width=\linewidth]{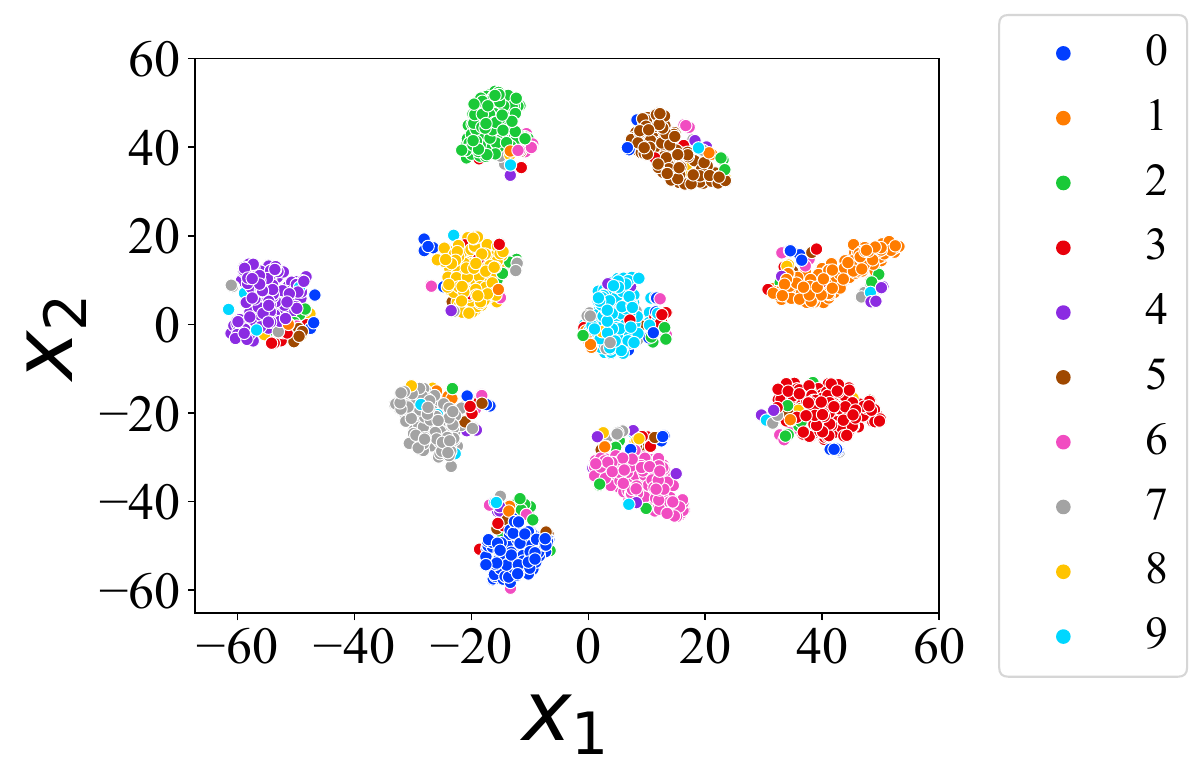}
        \caption{Biased MNIST.}
        \label{fig:biased_mnist_tsne}
    \end{subfigure}
    \hspace{0.03\textwidth}
    \begin{subfigure}[t]{0.5\textwidth}
        \includegraphics[width=\linewidth]{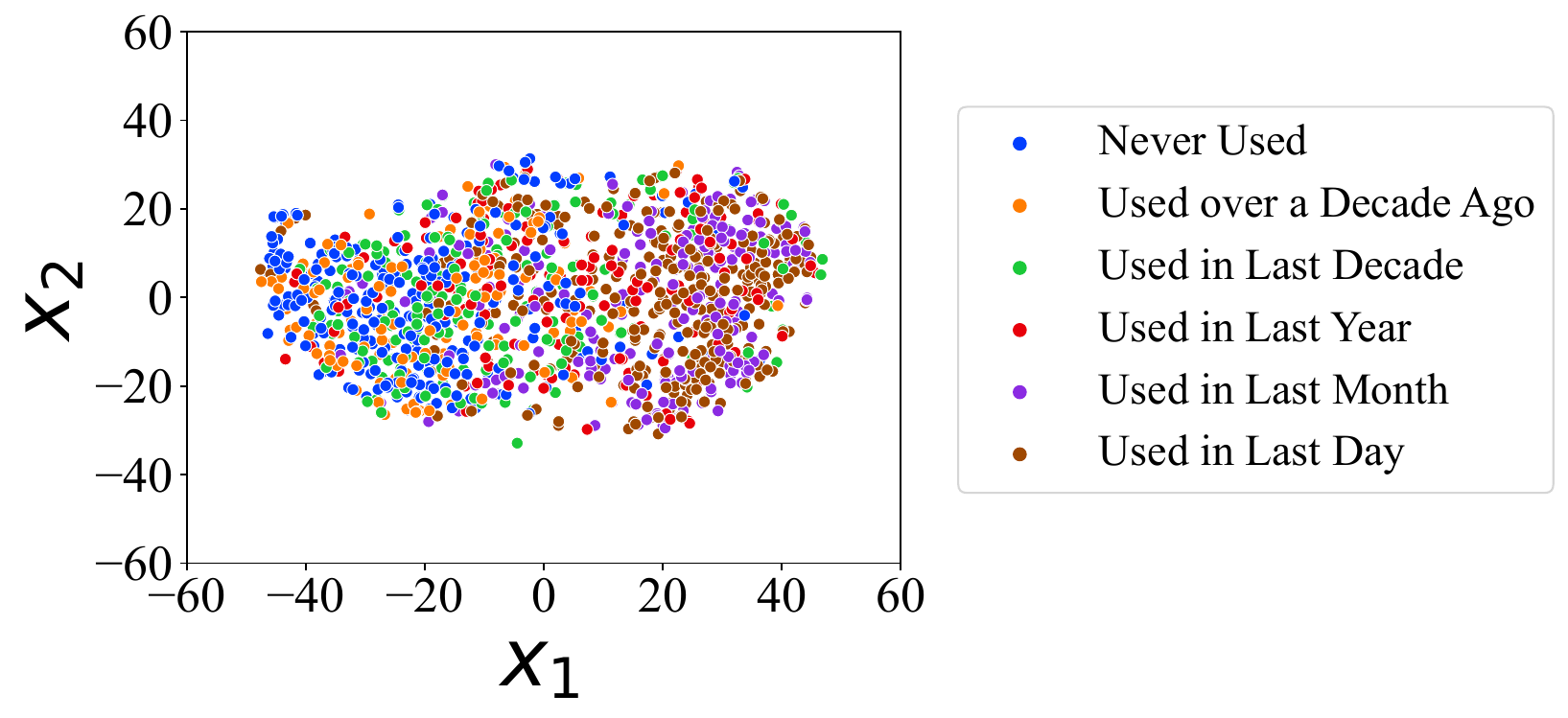}
        \caption{DRUG.}
        \label{fig:drug_tsne}
    \end{subfigure}
    \caption{t-SNE results for the MNIST, FMNIST, Biased MNIST, and DRUG datasets.}\label{fig:tsne_exp}
\end{figure}

\subsection{Approximation Error of Taylor Series}
\label{appendix:apx_error}
Continuing from Sec.~\ref{sec:unfairforgetting}, we provide empirical approximation errors between true losses and approximated losses derived from first-order Taylor series on the MNIST and Biased MNIST datasets as shown in Fig.~\ref{fig:loss_approx_error}. For each task, we train the model for 5 epochs and 15 epochs on the MNIST and Biased MNIST datasets, respectively. The approximation error is large when a new task begins because new samples with unseen classes are introduced. However, the error gradually decreases as the number of epochs increases while training a model for the task.

\begin{figure}[h]
    \centering
    \begin{subfigure}[t]{0.43\textwidth}
        \includegraphics[width=\linewidth]{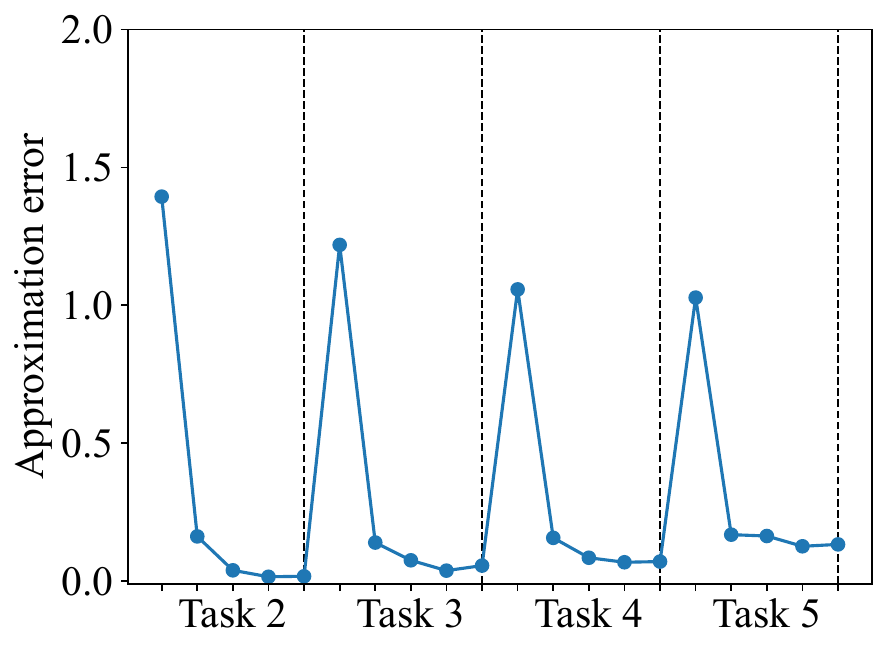}
        \caption{MNIST.}
        \label{fig:MNIST_approx_error}
    \end{subfigure}
    \begin{subfigure}[t]{0.43\textwidth}
        \includegraphics[width=\linewidth]{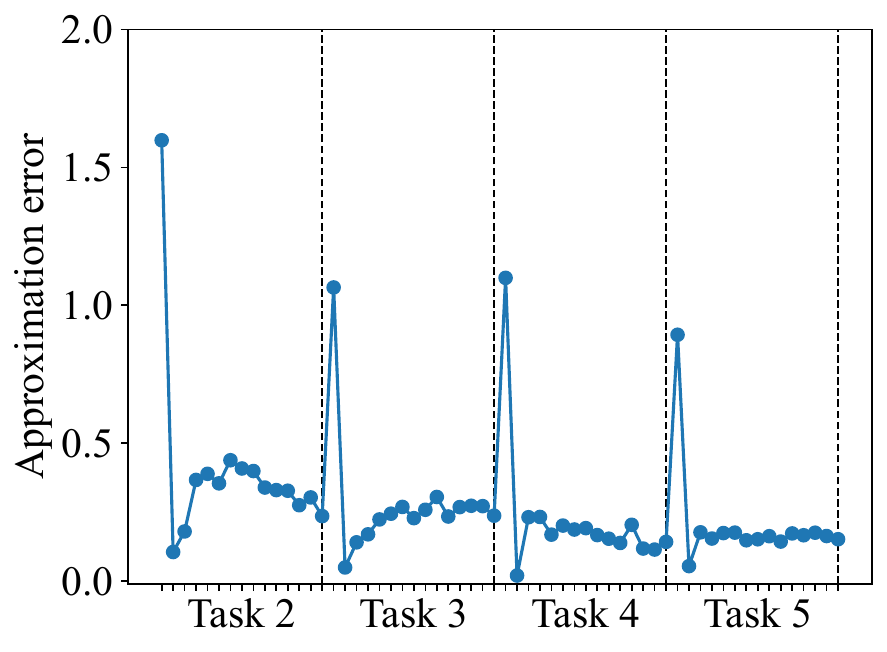}
        \caption{Biased MNIST.}
        \label{fig:Biased_MNIST_approx_error}
    \end{subfigure}
    \caption{Absolute errors between true losses and approximated losses derived from first-order Taylor series while training a model.}
    \label{fig:loss_approx_error}
\end{figure}



\subsection{Computational Complexity and Runtime Results of \method{}}
\label{subsec:runtime}

Continuing from Sec.~\ref{subsec:alg}, we provide computational complexity as shown in Fig.~\ref{fig:complexity}. Our empirical results show that for about twelve thousand current-task samples, the time to solve an LP problem is a few seconds for the MNIST dataset. By applying the log-log regression model to the results in Fig.~\ref{fig:complexity}, the computational complexity of solving LP at each epoch is $\mathcal{O}(|T_l|^{1.642})$ where $|T_l|$ denotes the number of current task samples. We note that this complexity can be quadratic in the worst case. If the task size becomes too large, we believe that clustering similar samples and assigning weights to the clusters, rather than samples, could be a solution to reduce the computational overhead. 
A more detailed discussion is provided in Sec.~\ref{appendix:reduce_overhead}.

We also show overall runtime results of \method{} using the MNIST and Biased MNIST datasets and compare the total runtime of all baselines in Fig.~\ref{fig:runtime} and Table~\ref{tbl:runtime_comparison}, respectively. 
In Fig.~\ref{fig:runtime}, we present the overall runtime of \method{} divided into three steps: Gradient Computation, CPLEX Computation, and Model Training. The total runtime of \method{} with other baselines is in Table~\ref{tbl:runtime_comparison}, and we checked that the overhead of \method{} is not excessive compared to the baselines.

\begin{figure}[h]
    \centering
    \includegraphics[width=0.4\textwidth]{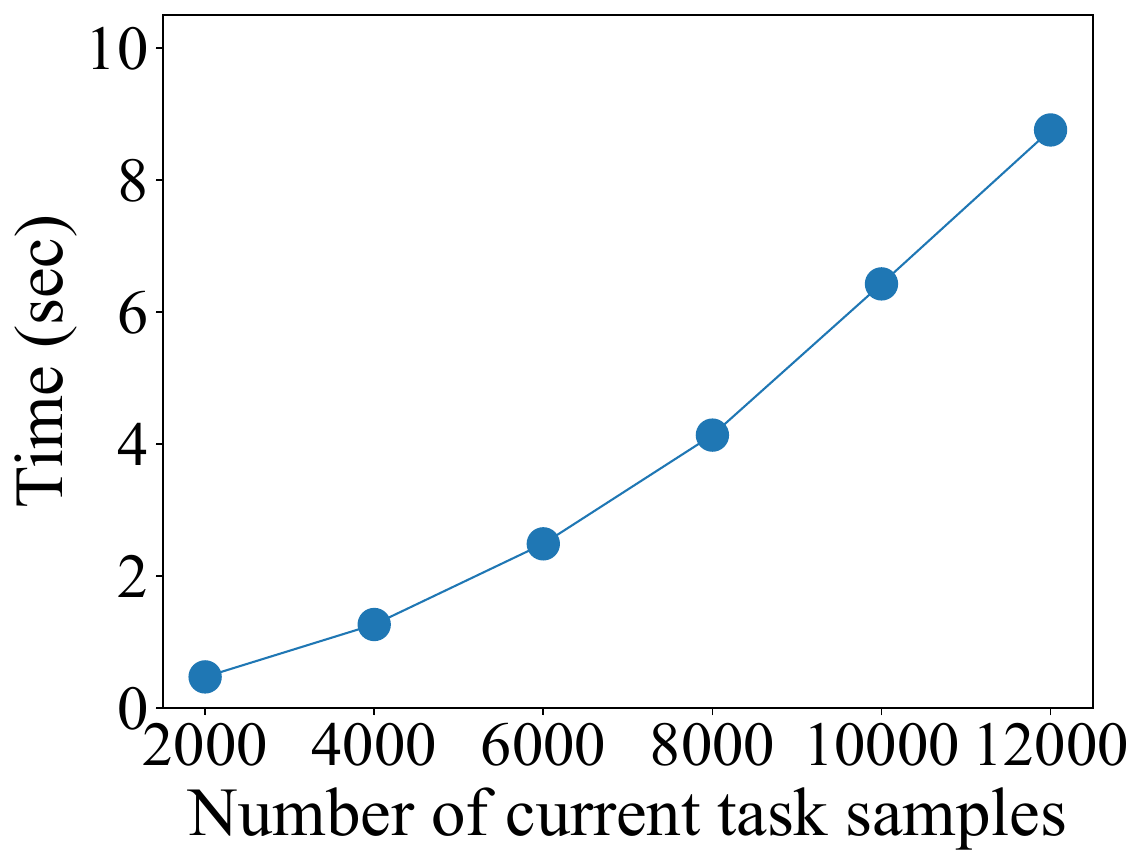}
    \caption{Runtime results of solving a single LP problem in \method{} using CPLEX for the MNIST dataset.}
    \label{fig:complexity}
\end{figure}

\begin{figure}[h]
    \centering
    \includegraphics[width=0.48\textwidth]{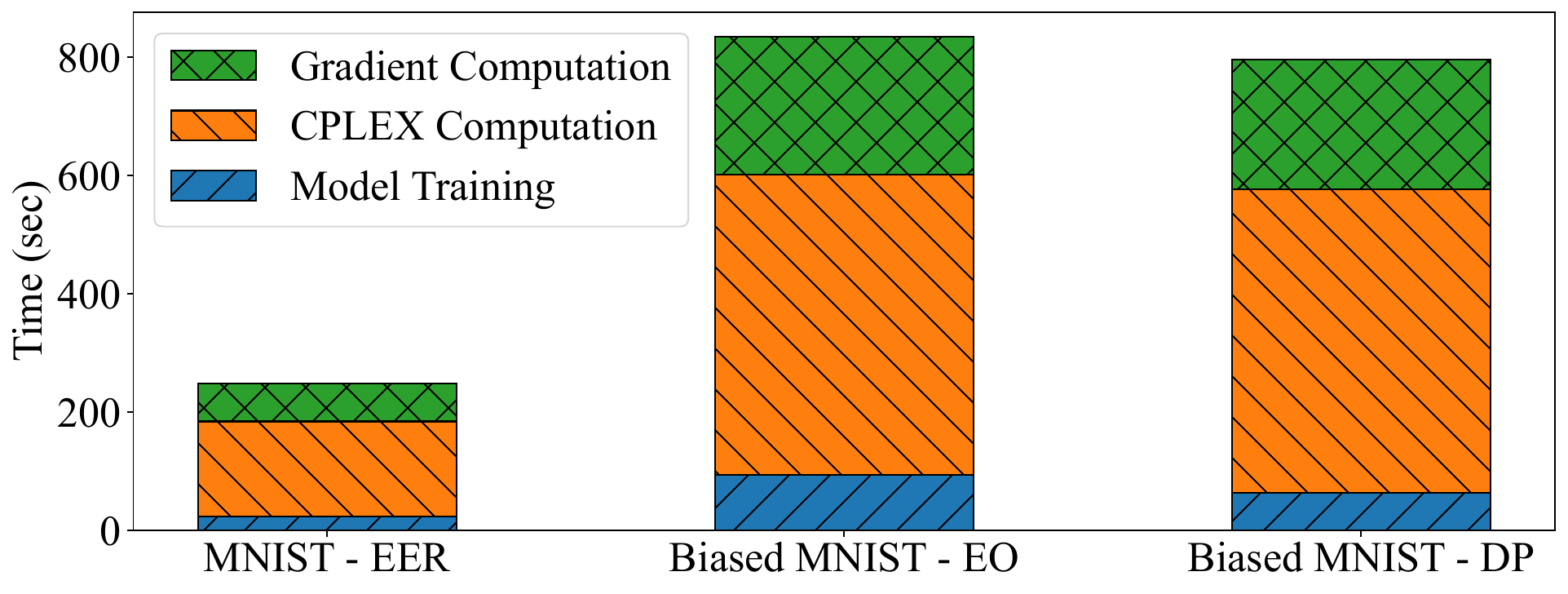}
    \caption{Overall runtime results of our framework on all tasks for three datasets: MNIST--EER, Biased MNIST--EO, and Biased MNIST--DP.}
    \label{fig:runtime}
\end{figure}

\begin{table}[h!]
    \setlength{\tabcolsep}{1.82pt}
    \caption{Overall runtime results of \method{} (w.r.t. EO Disp.) with other baselines on the Biased MNIST datasets.}
    \centering
    \begin{tabular}{l|ccccccc}
        \toprule
        {Methods} & {\method{}} & {iCaRL} & {WA} & {CLAD} & {GSS} & {OCS} & {FaIRL} \\
        \midrule
        {Time (sec)} & 835\tiny{$\pm$018} & 947\tiny{$\pm$015} & 177\tiny{$\pm$023} & 96\tiny{$\pm$018} & 5563\tiny{$\pm$069} & 2681\tiny{$\pm$074} & 5827\tiny{$\pm$228} \\
        \bottomrule
    \end{tabular}
    \label{tbl:runtime_comparison}
\end{table}

\subsection{More Details on Datasets}
\label{appendix:dataset}

Continuing from Sec.~\ref{subsubsec:dataset}, we provide more details of the two datasets using the class as the sensitive attribute and the three datasets with separate sensitive attributes. For datasets with a total of $\mathcal{C}$ classes, we divide the datasets into $L$ sequences of tasks where each task consists of $\mathcal{C}/L$ classes, and assume that task boundaries are available\,\citep{DBLP:journals/corr/abs-1904-07734}. We also consider using standard benchmark datasets in the fairness field, but they are unsuitable for class-incremental learning experiments either because there are only two classes (e.g., COMPAS\,\citep{angwin2016machine}, AdultCensus\,\citep{DBLP:conf/kdd/Kohavi96}, and Jigsaw\,\citep{jigsaw-unintended-bias-in-toxicity-classification}), or because it is difficult to apply group fairness metrics. For instance, in the case of CelebA\,\citep{DBLP:conf/iccv/LiuLWT15}, each person is considered a class, making the sensitive attribute dependent on the true label.

\paragraph{MNIST\,\citep{DBLP:journals/pieee/LeCunBBH98}} The MNIST dataset is a standard benchmark for evaluating the performance of machine learning models, especially in image classification tasks. The dataset is a collection of grayscale images of handwritten digits ranging from 0 to 9, each measuring 28 pixels in width and 28 pixels in height. The dataset consists of 60,000 training images and 10,000 test images. We configure a class-incremental learning setup, where a total of 10 classes are evenly distributed across 5 tasks, with 2 classes per task. We assume the class itself is the sensitive attribute.

\paragraph{Fashion-MNIST (FMNIST)\,\citep{DBLP:journals/corr/abs-1708-07747}} The Fashion-MNIST dataset is a specialized variant of the original MNIST dataset, designed for the classification of various clothing items into 10 distinct classes. The classes include `T-shirt/top', `Trouser', `Pullover', `Dress', `Coat', `Sandal', `Shirt', `Sneaker', `Bag', and `Ankle boot'. The dataset consists of grayscale images with dimensions of 28 pixels by 28 pixels including 60,000 training images and 10,000 test images. We configure a class-incremental learning setup, where a total of 10 classes are evenly distributed across 5 tasks, with 2 classes per task. We assume the class itself is the sensitive attribute.

\paragraph{Biased MNIST\,\citep{DBLP:conf/icml/BahngCYCO20}} The Biased MNIST dataset is a modified version of the MNIST dataset that introduces bias by incorporating background colors highly correlated with the digits. We select 10 distinct background colors and assign one to each digit from 0 to 9. For the training images, each digit is assigned the selected background color with a probability of 0.95, or one of the other colors at random with a probability of 0.05. For the test images, the background color of each digit is assigned from the selected color or other random colors with equal probability of 0.5. The dataset consists of 60,000 training images and 10,000 test images. We configure a class-incremental learning setup, where a total of 10 classes are evenly distributed across 5 tasks, with 2 classes per task. We set the background color as the sensitive attribute and consider two sensitive groups: the origin color and other random colors for each digit.

\paragraph{Drug Consumption (DRUG)\,\citep{fehrman2017factor}} The Drug Consumption dataset contains information about the usage of various drugs by individuals and correlates it with different demographic and personality traits. The dataset includes records for 1,885 respondents, each with 12 attributes including NEO-FFI-R, BIS-11, ImpSS, level of education, age, gender, country of residence, and ethnicity. We split the dataset into the ratio of 70/30 for training and testing. All input attributes are originally categorical, but we quantify them as real values for training. Participants were questioned about their use of 18 drugs, and our task is to predict cannabis usage. The label variable contains six classes: `Never Used', `Used over a Decade Ago', `Used in Last Decade', `Used in Last Year', `Used in Last Month', and `Used in Last Day'. We configure a class-incremental learning setup, where a total of 6 classes are distributed across 3 tasks, with 2 classes per task. We set gender as the sensitive attribute and consider two sensitive groups: male and female.

\paragraph{BiasBios\,\citep{DBLP:conf/fat/De-ArteagaRWCBC19}} The BiasBios dataset is a benchmark designed to explore and evaluate bias in natural language processing models, particularly in the context of profession classification from bios. The dataset consists of short textual biographies collected from online sources, labeled with one of the 28 profession classes, such as `professor', `nurse', or `software engineer'. The dataset includes gender annotations, which makes it suitable for studying biases related to gender. The dataset contains approximately 350k biographies where 253k are for training and 97k for testing. We configure a class-incremental learning setup using the 25 most-frequent professions, where a total of 25 classes are distributed across 5 tasks, with 5 classes per task. As the number of samples for each class varies significantly, we arrange the classes in descending order based on their size\,\citep{DBLP:conf/aaai/ChowdhuryC23}. We set gender as the sensitive attribute and consider two sensitive groups: male and female.

\subsection{More Details on Models and Hyperparemeters}
\label{appendix:models_hyparparmas}

Continuing from Sec.~\ref{subsubsec:models_hyparparmas}, we provide more details on experimental settings. 
For training, we use an SGD optimizer with momentum 0.9 and a batch size of 64 for all experiments. We also set the initial learning rate and the number of epochs for each dataset as follows: For the MNIST, FMNIST, Biased MNIST, and DRUG datasets, we train both our model and baselines with initial learning rates of [0.001, 0.01, 0.1], for 5, 5, 15, and 25 epochs, respectively. For the BiasBios dataset, we use learning rates of [0.00002, 0.0001, 0.001] for 10 epochs and set the maximum token length to 128. For hyperparameters, we perform cross-validation with a grid search for $\alpha \in \{0.0005, 0.001, 0.002, 0.01\}$, $\lambda \in \{0.1, 0.5, 1\}$, and $\tau \in \{1, 2, 5, 10\}$. All evaluations are performed on separate test sets and repeated with five random seeds. We write the average and standard deviation of performance results and run experiments on Intel Xeon Silver 4114 CPUs and NVIDIA RTX A6000 GPUs.

\subsection{More Results on Accuracy and Fairness}
\label{appendix:results-main}

Continuing from Sec.~\ref{subsec:expresults}, we compare \method{} with other baselines with respect to EER, EO, and DP disparity as shown in Tables~\ref{tbl:performance_eer_disp},~\ref{tbl:performance_eo_disp}, and~\ref{tbl:performance_dp_disp}, respectively. 

On BiasBios, {\it iCaRL}’s accuracy surpasses even {\it Joint training}, commonly considered the upper bound, while it has low accuracy on the Biased MNIST dataset.
{\it iCaRL}'s high accuracy on BiasBios likely stems from its prototypical classifier, which is structurally distinct from the fully-connected layer used by other methods. On the other hand, \method{} achieves fairness and accuracy comparable to {\it Joint training} on BiasBios, effectively mitigating unfair forgetting without compromising accuracy.
As illustrated by the t-SNE visualization on the Biased MNIST dataset (Sec.~\ref{subsec:tsne}), some samples are misaligned due to background color biases, causing difficulty for {\it iCaRL} to accurately classify these points. In comparison, \method{} consistently exhibits lower EO and DP disparities compared to {\it iCaRL} for all the datasets. This supports that \method{} is more robust, particularly when the backbone network struggles to generate discriminative representations.


We note that the performance of {\it FaIRL} in our experiments may not fully represent its potential. However, the complicated model structure of {\it FaIRL} not only causes the instability of training, but also introduces additional multiple loss terms, resulting in numerous hyperparameters that are difficult to optimize. Using the identical hyperparameter set from {\it FaIRL} shows significantly lower performance, mainly because we standardize all models to use the same backbone for a fair comparison, differing from {\it FaIRL}’s original setup. Moreover, among all the baseline methods, {\it FaIRL} requires the longest training time, as discussed in Sec~\ref{subsec:runtime}, limiting our ability to thoroughly tune its hyperparameters. Consequently, the inherent challenges of the method and experimental constraints lead to suboptimal results.

\begin{table*}[h!]
  \setlength{\tabcolsep}{27pt}
  \caption{Accuracy and fairness results on the MNIST and FMNIST datasets with respect to EER disparity, where the class is the sensitive attribute. We compare \method{} with four types of baselines: na\"ive ({\it Joint Training} and {\it Fine Tuning}), state-of-the-art ({\it iCaRL}, {\it WA}, and {\it CLAD}), sample selection ({\it GSS} and {\it OCS}), and fairness-aware ({\it FaIRL}) methods. We mark the best and second-best results with \textbf{bold} and \underline{underline}, respectively.}
  \centering
  \begin{tabular}{l|cccc}
  \toprule
    {Methods} & \multicolumn{2}{c}{\sf MNIST} & \multicolumn{2}{c}{\sf FMNIST} \\
    \cmidrule{1-5}
    {} & {Acc.} & {EER Disp.} & {Acc.} & {EER Disp.} \\
    \midrule
    {Joint Training} & .989\tiny{$\pm$.000} & .003\tiny{$\pm$.000} & .921\tiny{$\pm$.002} & .024\tiny{$\pm$.002} \\
    {Fine Tuning}& .455\tiny{$\pm$.000} & .326\tiny{$\pm$.000} & .451\tiny{$\pm$.000} & .325\tiny{$\pm$.000} \\
    \cmidrule{1-5}
    {iCaRL} & .918\tiny{$\pm$.005} & .048\tiny{$\pm$.003} & \textbf{.852\tiny{$\pm$.002}} & \underline{.047\tiny{$\pm$.001}} \\
    {WA} & .911\tiny{$\pm$.007} & .052\tiny{$\pm$.006} & .809\tiny{$\pm$.005} & .088\tiny{$\pm$.003} \\
    {CLAD} & .835\tiny{$\pm$.016} & .099\tiny{$\pm$.016} & .782\tiny{$\pm$.018} & .118\tiny{$\pm$.022} \\
    \cmidrule{1-5}
    {GSS} & .889\tiny{$\pm$.010} & .080\tiny{$\pm$.009} & .732\tiny{$\pm$.021} & .149\tiny{$\pm$.019} \\
    {OCS} & \textbf{.929\tiny{$\pm$.002}} & \underline{.040\tiny{$\pm$.003}} & .799\tiny{$\pm$.008} & .109\tiny{$\pm$.007} \\
    \cmidrule{1-5}
    {FaIRL} & .558\tiny{$\pm$.060} & .273\tiny{$\pm$.018} & .531\tiny{$\pm$.032} & .289\tiny{$\pm$.019} \\
    \cmidrule{1-5}
    {\bf \method{}} & \underline{.925\tiny{$\pm$.004}} & \textbf{.032\tiny{$\pm$.005}} & \underline{.824\tiny{$\pm$.006}} & \textbf{.039\tiny{$\pm$.006}} \\
    \bottomrule
  \end{tabular}
  \label{tbl:performance_eer_disp}
\end{table*}


\begin{table*}[h!]
  \setlength{\tabcolsep}{15.6pt}
  \caption{Accuracy and fairness results on the Biased MNIST, DRUG, and BiasBios datasets with respect to EO disparity, where background color is the sensitive attribute for Biased MNIST, and gender for DRUG and BiasBios, respectively. Due to the excessive time ($>$5 days) required to run {\it OCS} on BiasBios, we are not able to measure the results and mark them as `--'. The other settings are the same as in Table~\ref{tbl:performance_eer_disp}.}
  \centering
  \begin{tabular}{l|cccccc}
  \toprule
    {Methods} & \multicolumn{2}{c}{\sf Biased MNIST} & \multicolumn{2}{c}{\sf DRUG} & \multicolumn{2}{c}{\sf BiasBios} \\
    \cmidrule{1-7}
    {} & {Acc.} & {EO Disp.} & {Acc.} & {EO Disp.} & {Acc.} & {EO Disp.} \\
    \midrule
    {Joint Training} & .944\tiny{$\pm$.002} & .108\tiny{$\pm$.003} & .442\tiny{$\pm$.015} & .179\tiny{$\pm$.052} & .823\tiny{$\pm$.002} & .076\tiny{$\pm$.001} \\
    {Fine Tuning} & .449\tiny{$\pm$.001} & .016\tiny{$\pm$.002} & .357\tiny{$\pm$.009} & .125\tiny{$\pm$.034} & .420\tiny{$\pm$.001} & .028\tiny{$\pm$.002} \\
    \cmidrule{1-7}
    {iCaRL} & .802\tiny{$\pm$.008} & .365\tiny{$\pm$.021} & \textbf{.444\tiny{$\pm$.025}} & .190\tiny{$\pm$.017} & \textbf{.829\tiny{$\pm$.002}} & .084\tiny{$\pm$.003} \\
    {WA} & \textbf{.916\tiny{$\pm$.002}} & .140\tiny{$\pm$.004} & .408\tiny{$\pm$.022} & .134\tiny{$\pm$.029} & .796\tiny{$\pm$.003} & .076\tiny{$\pm$.001} \\
    {CLAD} & .871\tiny{$\pm$.012} & .198\tiny{$\pm$.022} & .410\tiny{$\pm$.026} & .114\tiny{$\pm$.043} & .799\tiny{$\pm$.003} & .074\tiny{$\pm$.002} \\
    \cmidrule{1-7}
    {GSS} & .809\tiny{$\pm$.005} & .325\tiny{$\pm$.017} & \underline{.426\tiny{$\pm$.010}} & .167\tiny{$\pm$.038} & \underline{.808\tiny{$\pm$.003}} & .081\tiny{$\pm$.002} \\
    {OCS} & .824\tiny{$\pm$.007} & .331\tiny{$\pm$.013} & .406\tiny{$\pm$.024} & .142\tiny{$\pm$.030} & -- & -- \\
    \cmidrule{1-7}
    {FaIRL} & .411\tiny{$\pm$.012} & \textbf{.118\tiny{$\pm$.011}} & .354\tiny{$\pm$.011} & \textbf{.060\tiny{$\pm$.021}} & .400\tiny{$\pm$.060} & \textbf{.055\tiny{$\pm$.020}} \\
    \cmidrule{1-7}
    {\bf \method{}} & \underline{.909\tiny{$\pm$.004}} & \underline{.119\tiny{$\pm$.007}} & .406\tiny{$\pm$.014} & \underline{.077\tiny{$\pm$.010}} & \underline{.808\tiny{$\pm$.002}} & \underline{.072\tiny{$\pm$.001}} \\
    \bottomrule
  \end{tabular}
  \label{tbl:performance_eo_disp}
\end{table*}


\begin{table*}[h!]
  \setlength{\tabcolsep}{15.6pt}
  \caption{Accuracy and fairness results on the Biased MNIST, DRUG, and BiasBios datasets with respect to DP disparity. The other settings are the same as in Table~\ref{tbl:performance_eo_disp}.}
  \centering
  \begin{tabular}{l|cccccc}
  \toprule
    {Methods} & \multicolumn{2}{c}{\sf Biased MNIST} & \multicolumn{2}{c}{\sf DRUG} & \multicolumn{2}{c}{\sf BiasBios} \\
    \cmidrule{1-7}
    {} & {Acc.} & {DP Disp.} & {Acc.} & {DP Disp.} & {Acc.} & {DP Disp.} \\
    \midrule
    {Joint Training}& .944\tiny{$\pm$.002} & .006\tiny{$\pm$.001} & .442\tiny{$\pm$.015} & .090\tiny{$\pm$.020}& .823\tiny{$\pm$.002} & .021\tiny{$\pm$.000} \\
    {Fine Tuning}& .449\tiny{$\pm$.001} & .017\tiny{$\pm$.008} & .357\tiny{$\pm$.009} & .102\tiny{$\pm$.013} & .420\tiny{$\pm$.001} & .028\tiny{$\pm$.002} \\
    \cmidrule{1-7}
    {iCaRL}& .802\tiny{$\pm$.008} & .015\tiny{$\pm$.001} & \textbf{.444\tiny{$\pm$.025}} & .093\tiny{$\pm$.009} & \textbf{.829\tiny{$\pm$.002}} & \underline{.022\tiny{$\pm$.000}} \\
    {WA} & \textbf{.916\tiny{$\pm$.002}} & \underline{.009\tiny{$\pm$.001}} & .408\tiny{$\pm$.022} & .067\tiny{$\pm$.013} & .796\tiny{$\pm$.003} & \underline{.022\tiny{$\pm$.000}} \\
    {CLAD} & .871\tiny{$\pm$.012} & .013\tiny{$\pm$.001} & \underline{.410\tiny{$\pm$.026}} & .069\tiny{$\pm$.019} & .799\tiny{$\pm$.003} & \underline{.022\tiny{$\pm$.000}} \\
    \cmidrule{1-7}
    {GSS} & .809\tiny{$\pm$.005} & .039\tiny{$\pm$.003} & .392\tiny{$\pm$.022} & .065\tiny{$\pm$.015} & .808\tiny{$\pm$.003} & .023\tiny{$\pm$.000} \\
    {OCS} & .824\tiny{$\pm$.007} & .035\tiny{$\pm$.003} & .393\tiny{$\pm$.017} & .053\tiny{$\pm$.012} & -- & -- \\
    \cmidrule{1-7}
    {FaIRL}& .411\tiny{$\pm$.012} & .026\tiny{$\pm$.008} & .354\tiny{$\pm$.011} & \textbf{.040\tiny{$\pm$.008}} & .400\tiny{$\pm$.060} & \textbf{.015\tiny{$\pm$.002}} \\
    \cmidrule{1-7}
    {\bf \method{}} & \underline{.904\tiny{$\pm$.004}} & \textbf{.008\tiny{$\pm$.001}} & .405\tiny{$\pm$.013} & \underline{.043\tiny{$\pm$.004}} & \underline{.809\tiny{$\pm$.003}} & \underline{.022\tiny{$\pm$.000}} \\
    \bottomrule
  \end{tabular}
  \label{tbl:performance_dp_disp}
\end{table*}

\subsection{More Results on Tradeoff between Accuracy and Fairness}
\label{appendix:tradeoff}
Continuing from Sec.~\ref{subsec:expresults}, we evaluate the tradeoff between accuracy and fairness of \method{} with other baselines as shown in Fig.~\ref{fig:MNIST_FMNIST_EER_tradeoff}--Fig.~\ref{fig:Bios_tradeoff}. 
The figures show \method{} positioned in the lower right corner of the graph, indicating better accuracy-fairness tradeoff results compared to other baselines. 
\method{} in the figures represents the result for different values of $\lambda$, a hyperparameter that balances fairness and accuracy. Since other baselines do not have a balancing parameter, we select Pareto-optimal points from all search spaces, where a Pareto-optimal point is defined as a point for which there does not exist another point with both higher accuracy and lower fairness disparity. 

\subsection{More Results on Sequential Accuracy and Fairness}
\label{appendix:results-seq}

Continuing from Sec.~\ref{subsec:expresults}, we present the sequential performance results for each task as shown in Fig.~\ref{fig:mnist_seq_acc_eer}--Fig.~\ref{fig:biasbios_seq_acc_dp}. 

\newpage

\begin{figure*}[h!]
    \centering
    \begin{subfigure}[t]{0.8\textwidth}
    \centering
    \includegraphics[width=\linewidth]{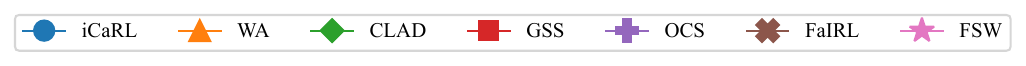}
    \end{subfigure}
    \begin{subfigure}[t]{0.2655\textwidth}
        \includegraphics[width=\linewidth]{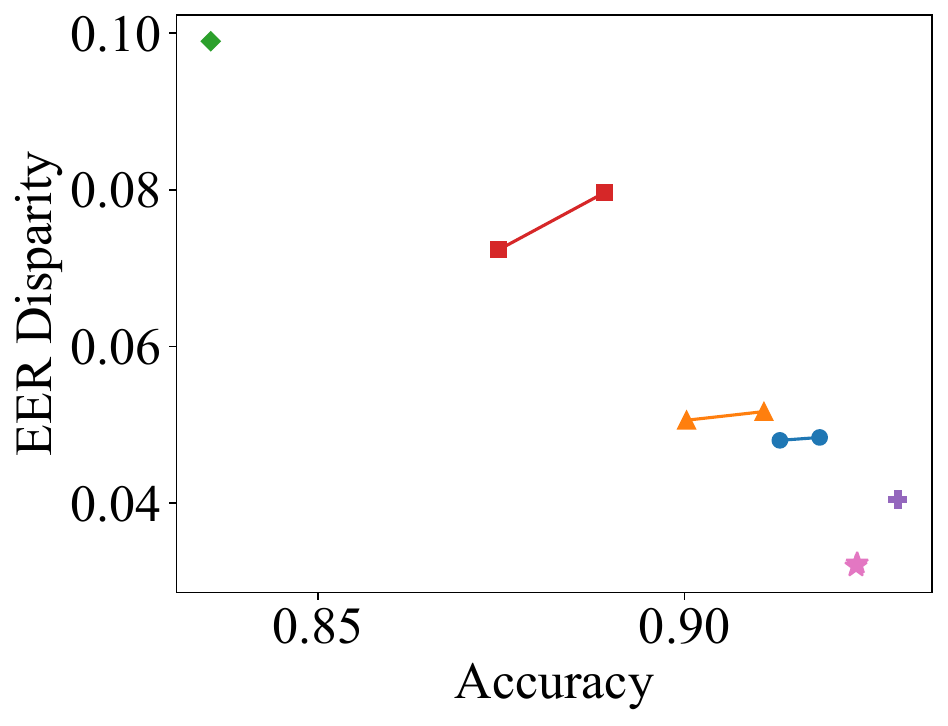}
        \caption{MNIST (EER).}
    \end{subfigure}
    \hspace{1cm}
    \begin{subfigure}[t]{0.273\textwidth}
        \includegraphics[width=\linewidth]{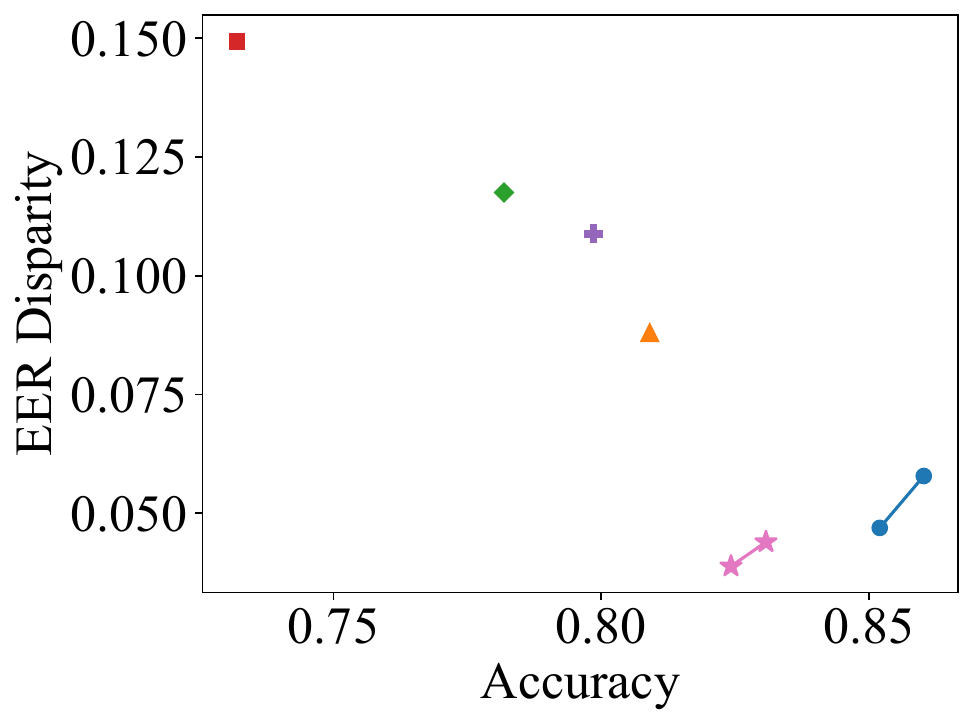}
        \caption{FMNIST (EER).}
    \end{subfigure}
    \caption{Tradeoff results between accuracy and fairness (EER) on the MNIST and FMNIST datasets.}
    \label{fig:MNIST_FMNIST_EER_tradeoff}
\end{figure*}
\begin{figure*}[h!]
    \centering
    \begin{subfigure}[t]{0.2655\textwidth}
        \includegraphics[width=\linewidth]{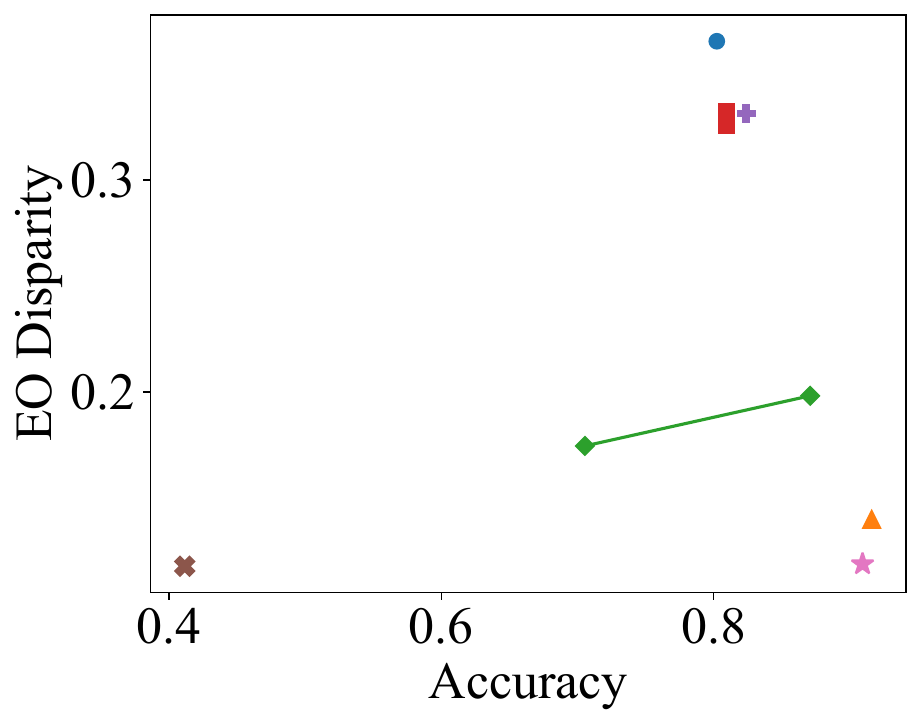}
        \caption{Biased MNIST (EO).}
    \end{subfigure}
    \hspace{1cm}
    \begin{subfigure}[t]{0.273\textwidth}
        \includegraphics[width=\linewidth]{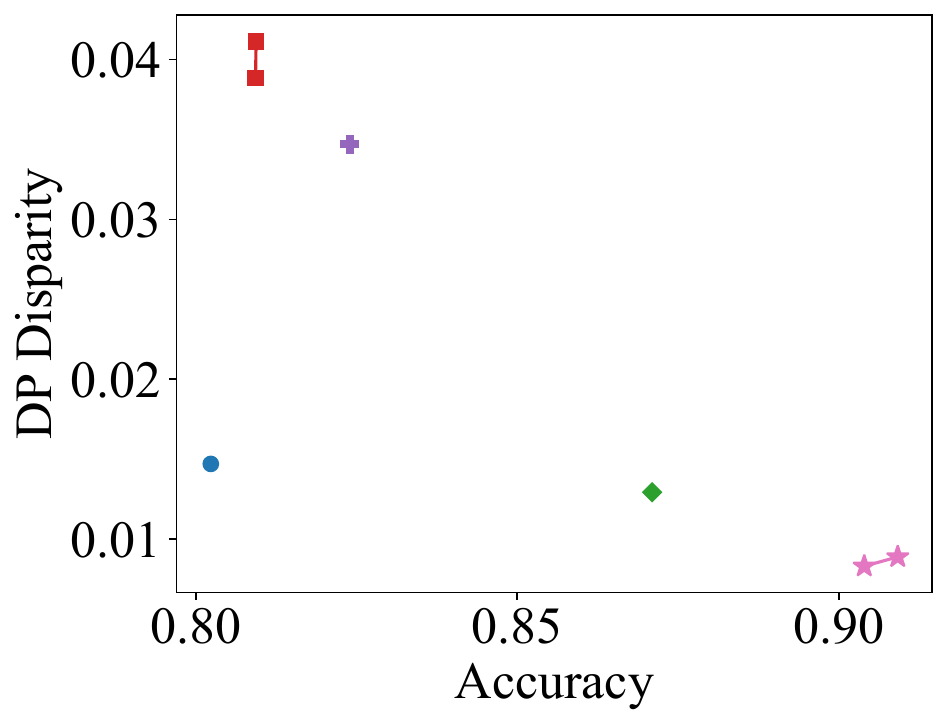}
        \caption{Biased MNIST (DP).}
    \end{subfigure}
    \caption{Tradeoff results between accuracy and fairness (EO and DP) on the Biased MNIST dataset.}
    \label{fig:BiasedMNIST_tradeoff}
\end{figure*}

\begin{figure*}[h!]
    \centering
    \begin{subfigure}[t]{0.273\textwidth}
        \includegraphics[width=\linewidth]{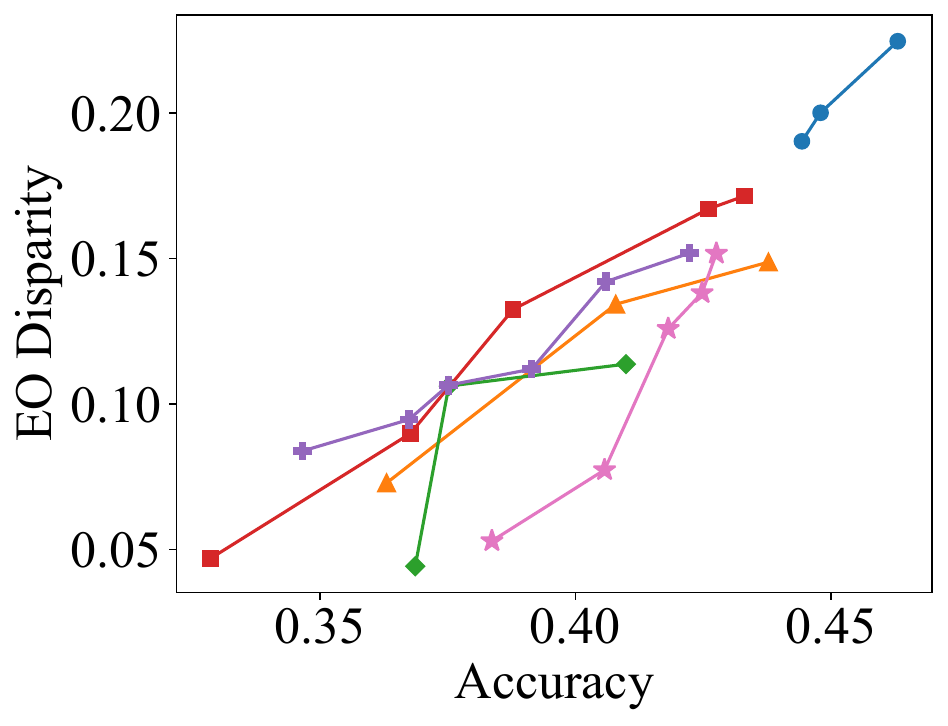}
        \caption{DRUG (EO).}
    \end{subfigure}
    \hspace{1cm}
    \begin{subfigure}[t]{0.273\textwidth}
        \includegraphics[width=\linewidth]{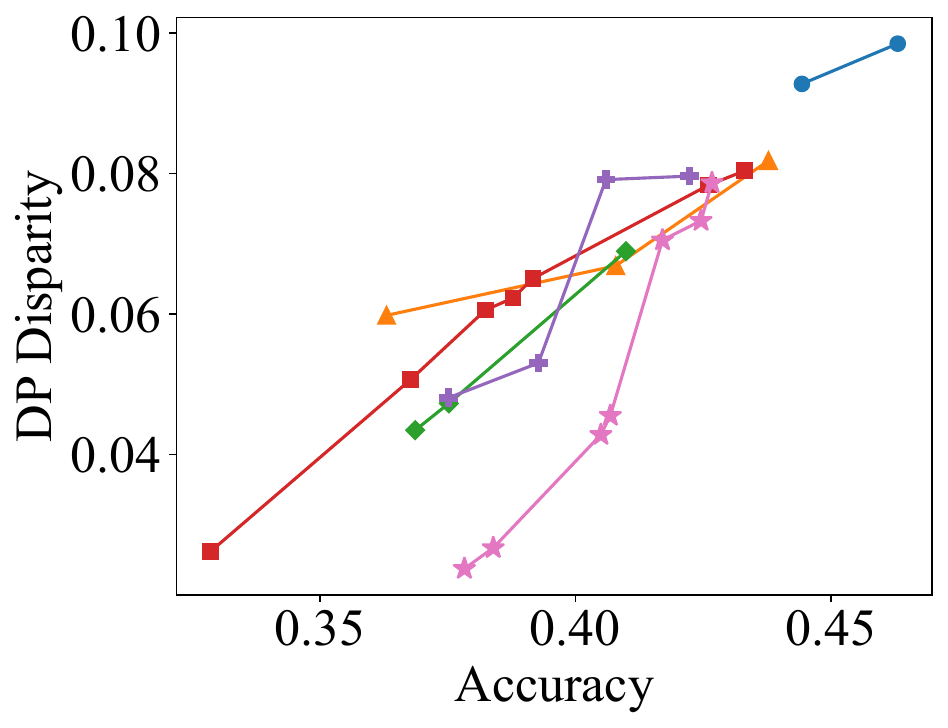}
        \caption{DRUG (DP).}
    \end{subfigure}
    \caption{Tradeoff results between accuracy and fairness (EO and DP) on the DRUG dataset.}
    \label{fig:Drug_tradeoff}
\end{figure*}

\begin{figure*}[h!]
    \centering
    \begin{subfigure}[t]{0.2655\textwidth}
        \includegraphics[width=\linewidth]{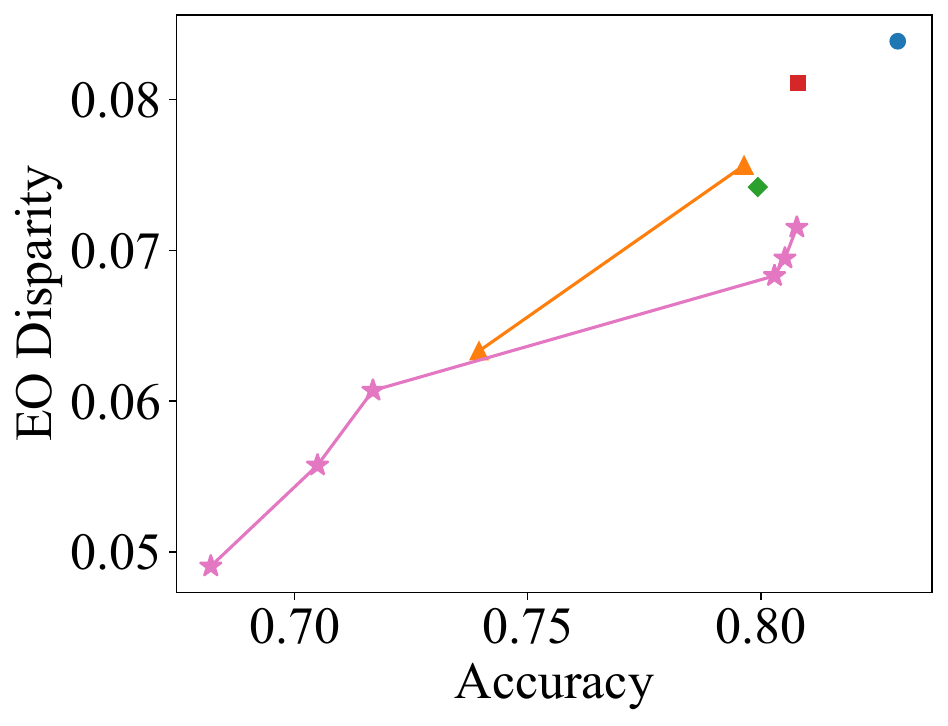}
        \caption{BiasBios (EO).}
    \end{subfigure}
    \hspace{1cm}
    \begin{subfigure}[t]{0.273\textwidth}
        \includegraphics[width=\linewidth]{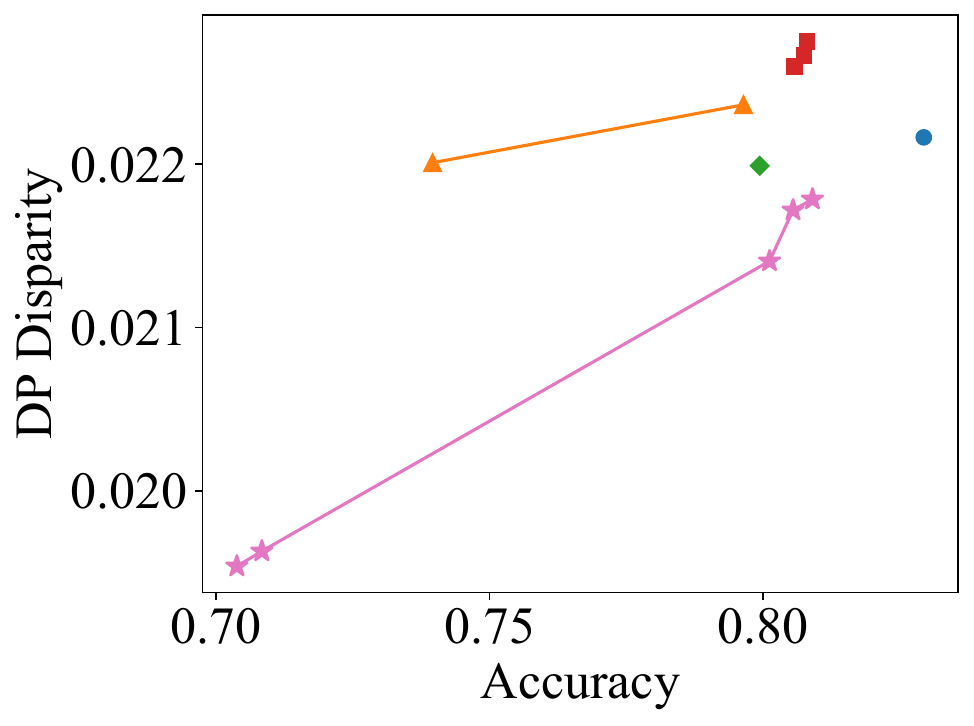}
        \caption{BiasBios (DP).}
    \end{subfigure}
    \caption{Tradeoff results between accuracy and fairness (EO and DP) on the BiasBios dataset.}
    \label{fig:Bios_tradeoff}
\end{figure*}

\clearpage

\begin{figure*}[h!]
    \centering
    \begin{subfigure}[t]{0.8\textwidth}
    \centering
    \includegraphics[width=\linewidth]{figures/legend.pdf}
    \end{subfigure}
    \begin{subfigure}[t]{0.27\textwidth}
        \includegraphics[width=\linewidth]{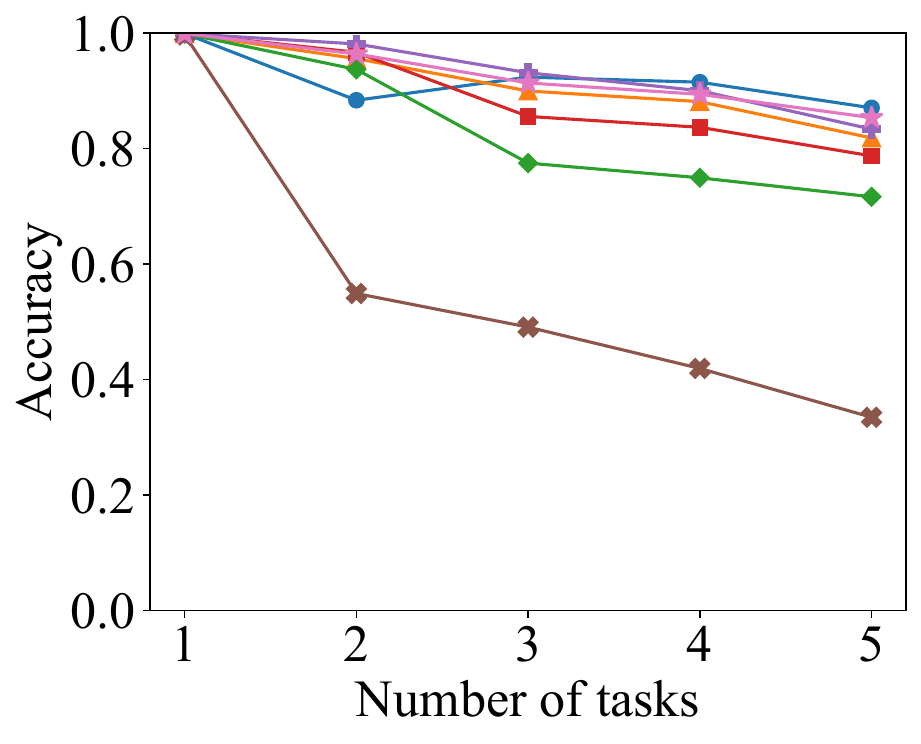}
        \caption{Accuracy.}
    \end{subfigure}
    \hspace{1cm}
    \begin{subfigure}[t]{0.27\textwidth}
        \includegraphics[width=\linewidth]{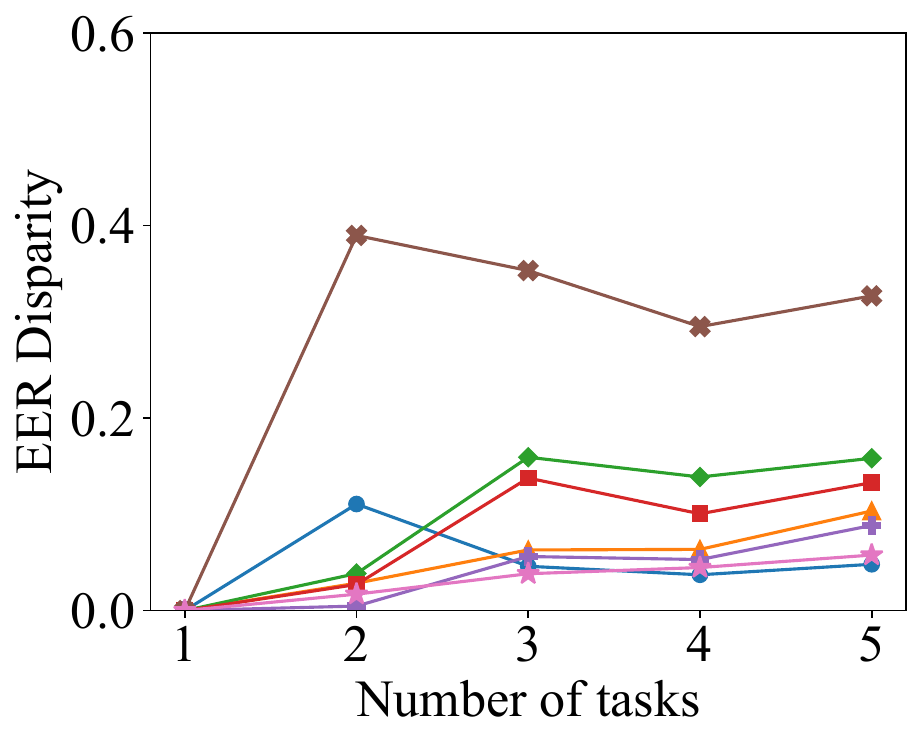}
        \caption{EER Disparity.}
    \end{subfigure}
    \vspace{-0.2cm}
    \caption{Sequential accuracy and fairness (EER) results on the MNIST dataset.}\label{fig:mnist_seq_acc_eer}
\end{figure*}

\begin{figure*}[h!]
    \centering
    \begin{subfigure}[t]{0.27\textwidth}
        \includegraphics[width=\linewidth]{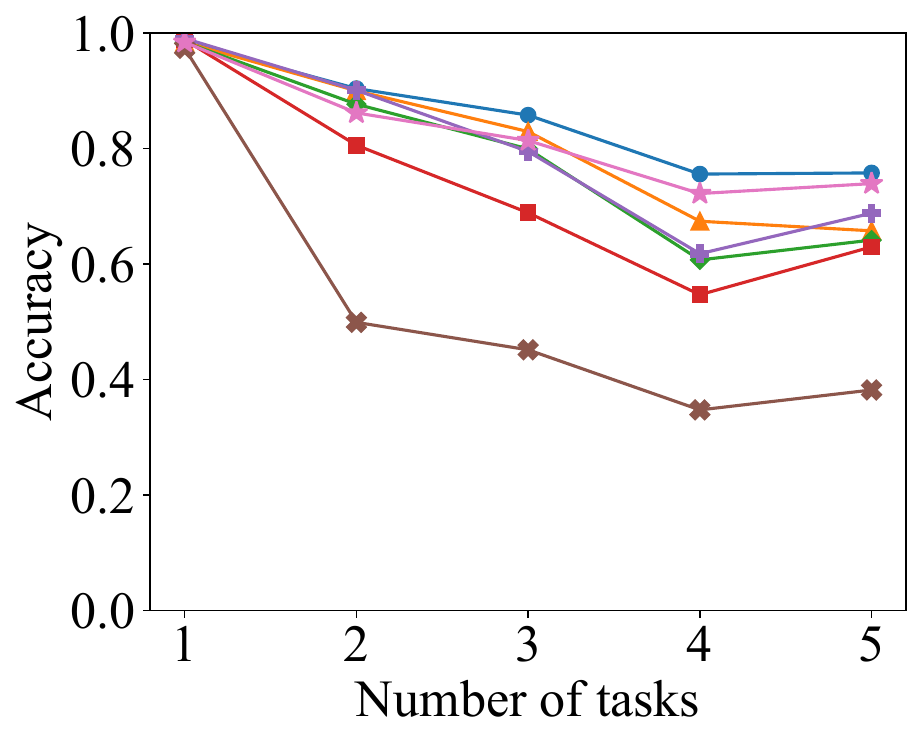}
        \caption{Accuracy.}
    \end{subfigure}
    \hspace{1cm}
    \begin{subfigure}[t]{0.27\textwidth}
        \includegraphics[width=\linewidth]{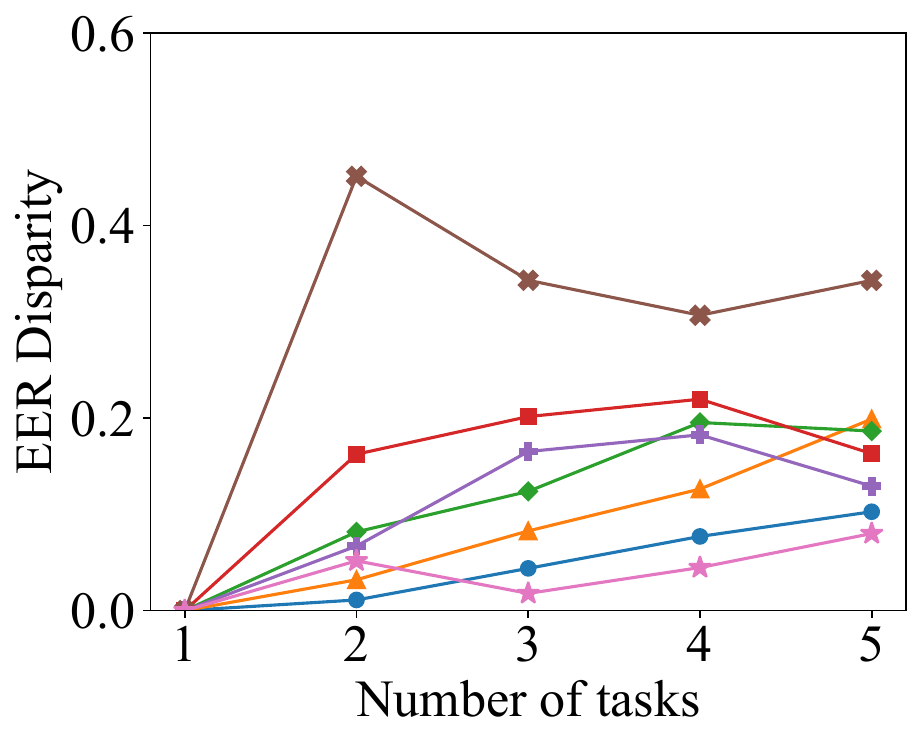}
        \caption{EER Disparity.}
    \end{subfigure}
    \vspace{-0.2cm}
    \caption{Sequential accuracy and fairness (EER) results on the FMNIST dataset.}\label{fig:fmnist_seq_acc_eer}
\end{figure*}

\begin{figure*}[h!]
    \centering
    \begin{subfigure}[t]{0.27\textwidth}
        \includegraphics[width=\linewidth]{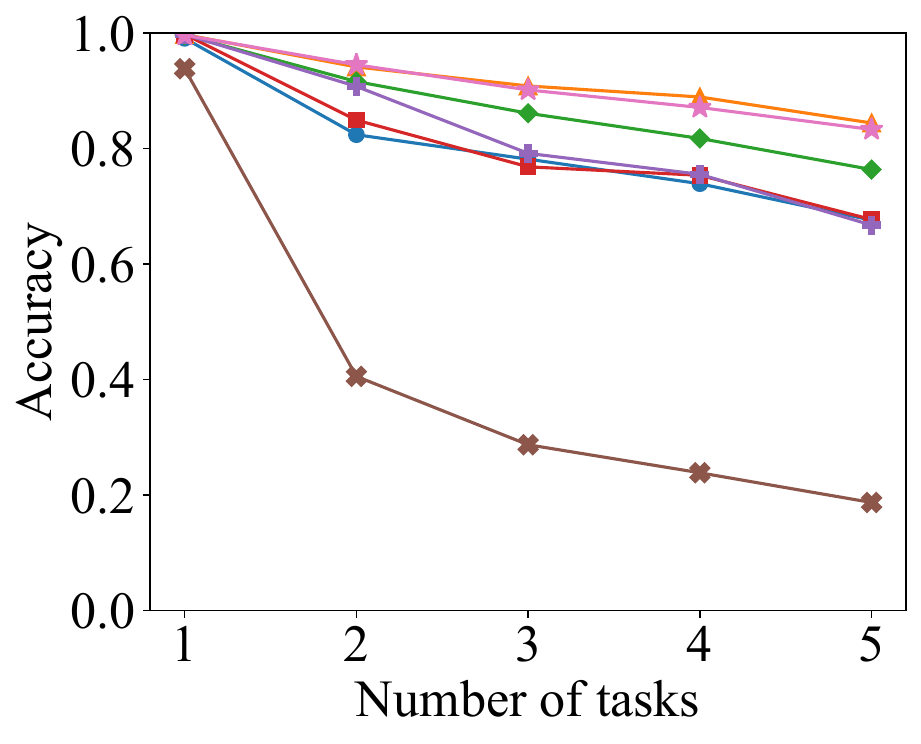}
        \caption{Accuracy.}
    \end{subfigure}
    \hspace{1cm}
    \begin{subfigure}[t]{0.27\textwidth}
        \includegraphics[width=\linewidth]{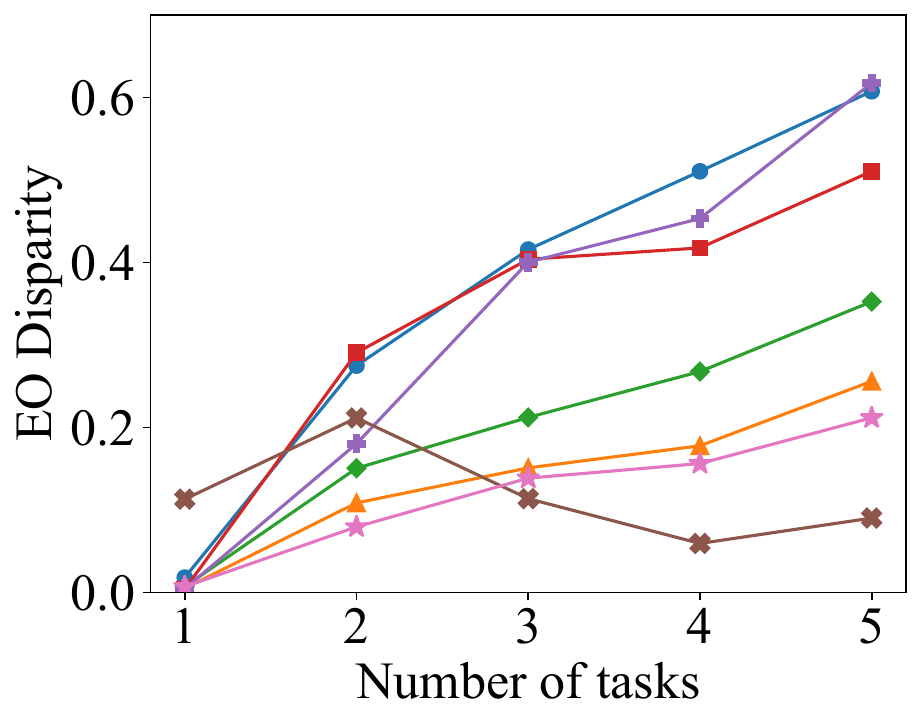}
        \caption{EO Disparity.}
    \end{subfigure}
    \vspace{-0.2cm}
    \caption{Sequential accuracy and fairness (EO) results on the Biased MNIST dataset.}\label{fig:biased_mnist_seq_acc_eo}
\end{figure*}

\begin{figure*}[h!]
    \centering
    \begin{subfigure}[t]{0.2625\textwidth}
        \includegraphics[width=\linewidth]{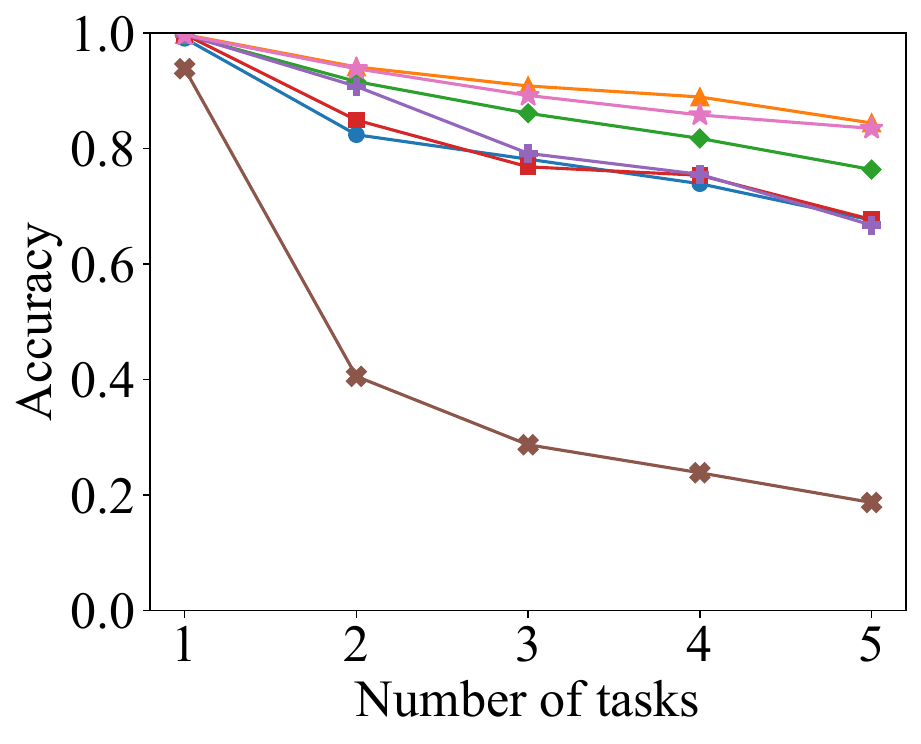}
        \caption{Accuracy.}
    \end{subfigure}
    \hspace{1cm}
    \begin{subfigure}[t]{0.27\textwidth}
        \includegraphics[width=\linewidth]{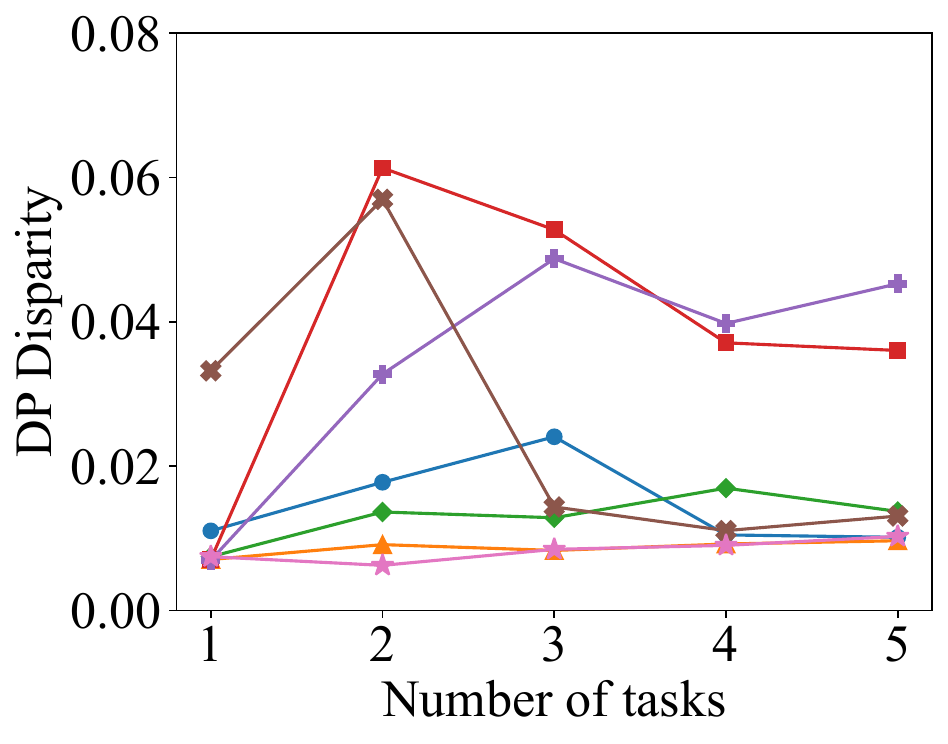}
        \caption{DP Disparity.}
    \end{subfigure}
    \vspace{-0.2cm}
    \caption{Sequential accuracy and fairness (DP) results on the Biased MNIST dataset.}\label{fig:biased_mnist_seq_acc_dp}
\end{figure*}

\begin{figure*}[h!]
    \centering
    \begin{subfigure}[t]{0.8\textwidth}
    \centering
    \includegraphics[width=\linewidth]{figures/legend.pdf}
    \end{subfigure}
    \begin{subfigure}[t]{0.27\textwidth}
        \includegraphics[width=\linewidth]{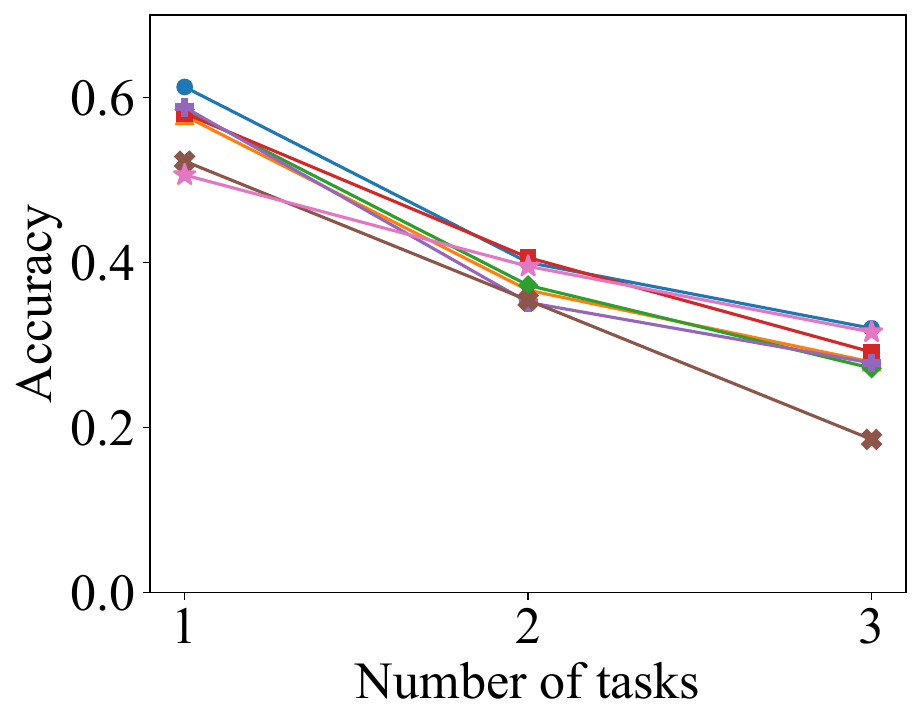}
        \caption{Accuracy.}
    \end{subfigure}
    \hspace{1cm}
    \begin{subfigure}[t]{0.27\textwidth}
        \includegraphics[width=\linewidth]{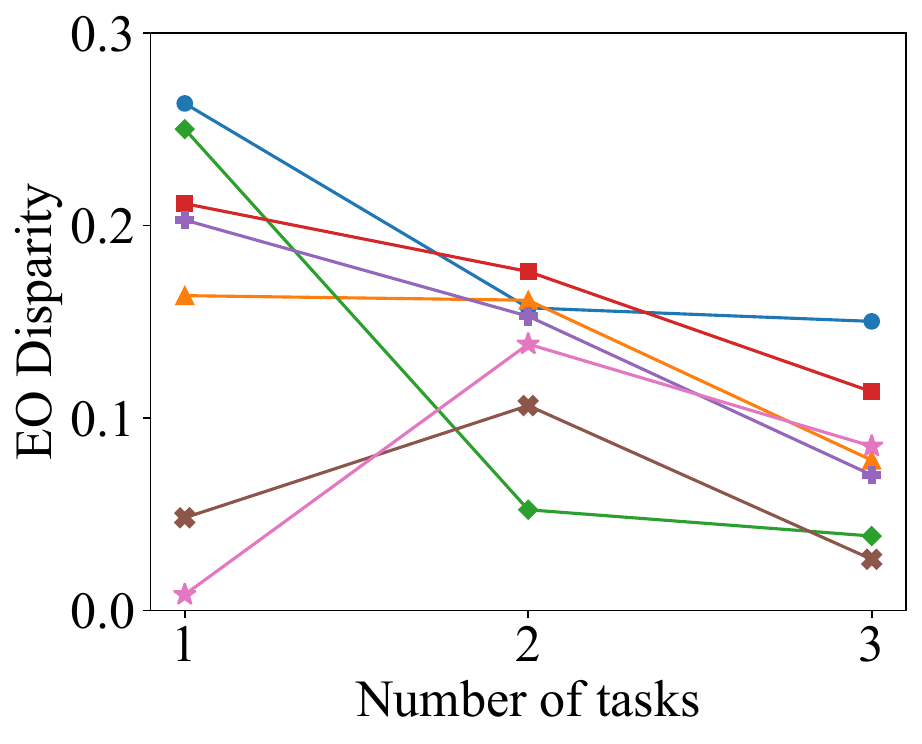}
        \caption{EO Disparity.}
    \end{subfigure}
    \vspace{-0.2cm}
    \caption{Sequential accuracy and fairness (EO) results on the DRUG dataset.}\label{fig:drug_seq_acc_eo}
\end{figure*}

\begin{figure*}[h!]
    \centering
    \begin{subfigure}[t]{0.27\textwidth}
        \includegraphics[width=\linewidth]{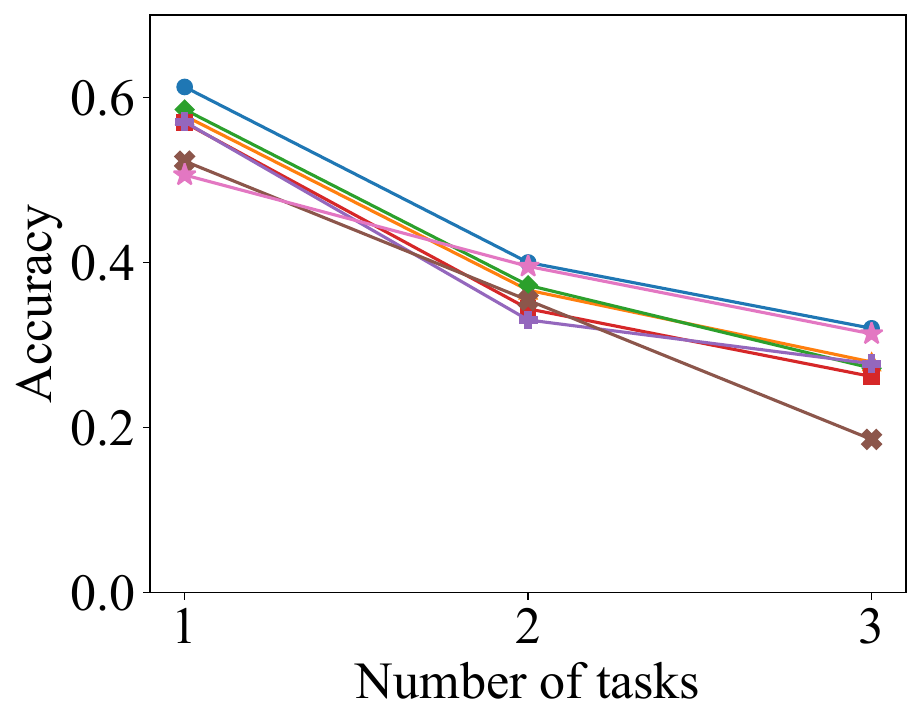}
        \caption{Accuracy.}
    \end{subfigure}
    \hspace{1cm}
    \begin{subfigure}[t]{0.278\textwidth}
        \includegraphics[width=\linewidth]{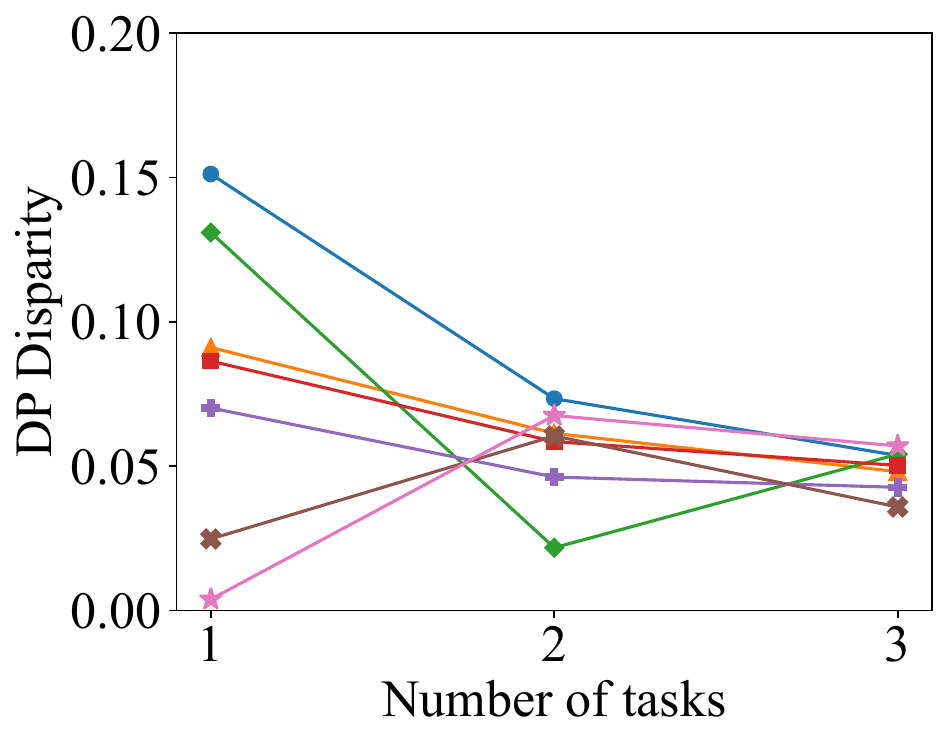}
        \caption{DP Disparity.}
    \end{subfigure}
    \vspace{-0.2cm}
    \caption{Sequential accuracy and fairness (DP) results on the DRUG dataset.}\label{fig:drug_seq_acc_dp}
\end{figure*}

\begin{figure*}[h!]
    \centering
    \begin{subfigure}[t]{0.27\textwidth}
        \includegraphics[width=\linewidth]{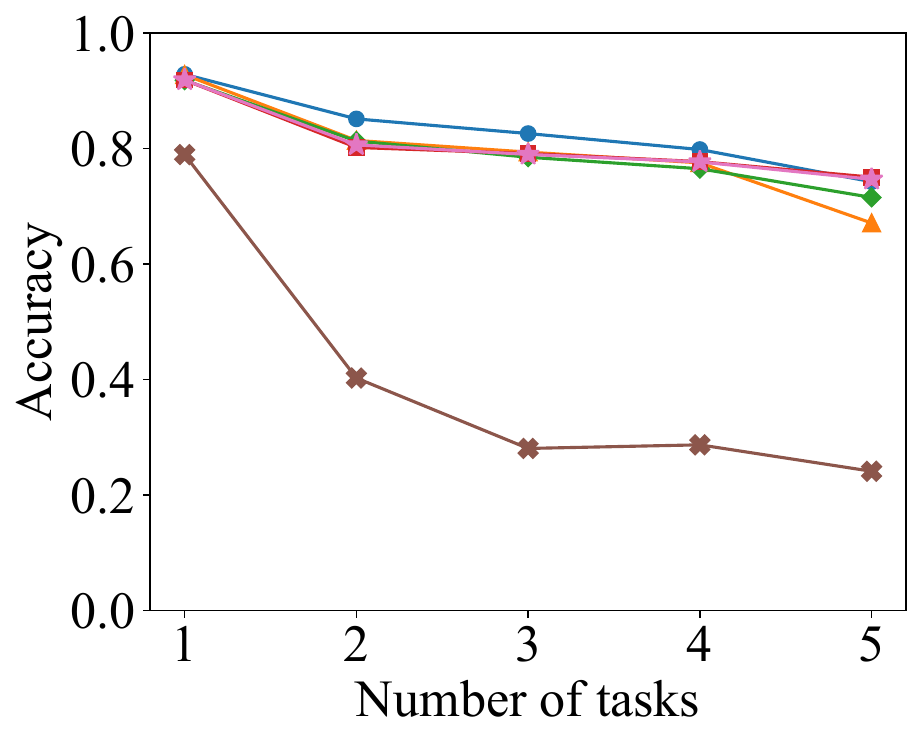}
        \caption{Accuracy.}
    \end{subfigure}
    \hspace{1cm}
    \begin{subfigure}[t]{0.278\textwidth}
        \includegraphics[width=\linewidth]{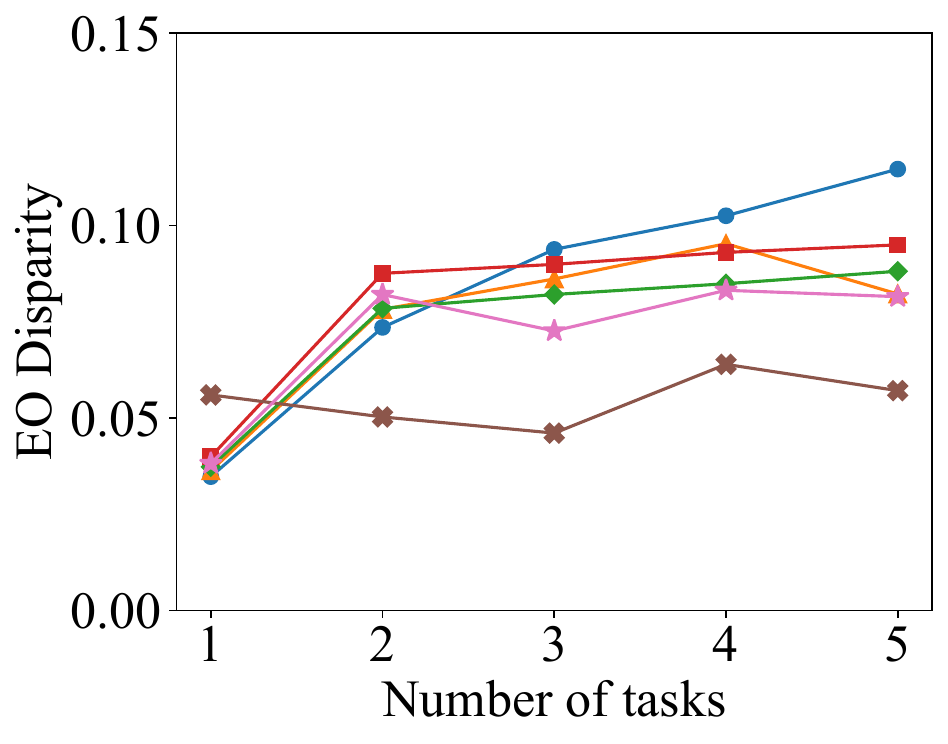}
        \caption{EO Disparity.}
    \end{subfigure}
    \vspace{-0.2cm}
    \caption{Sequential accuracy and fairness (EO) results on the BiasBios dataset.}\label{fig:biasbios_seq_acc_eo}
\end{figure*}

\begin{figure*}[h!]
    \centering
    \begin{subfigure}[t]{0.27\textwidth}
        \includegraphics[width=\linewidth]{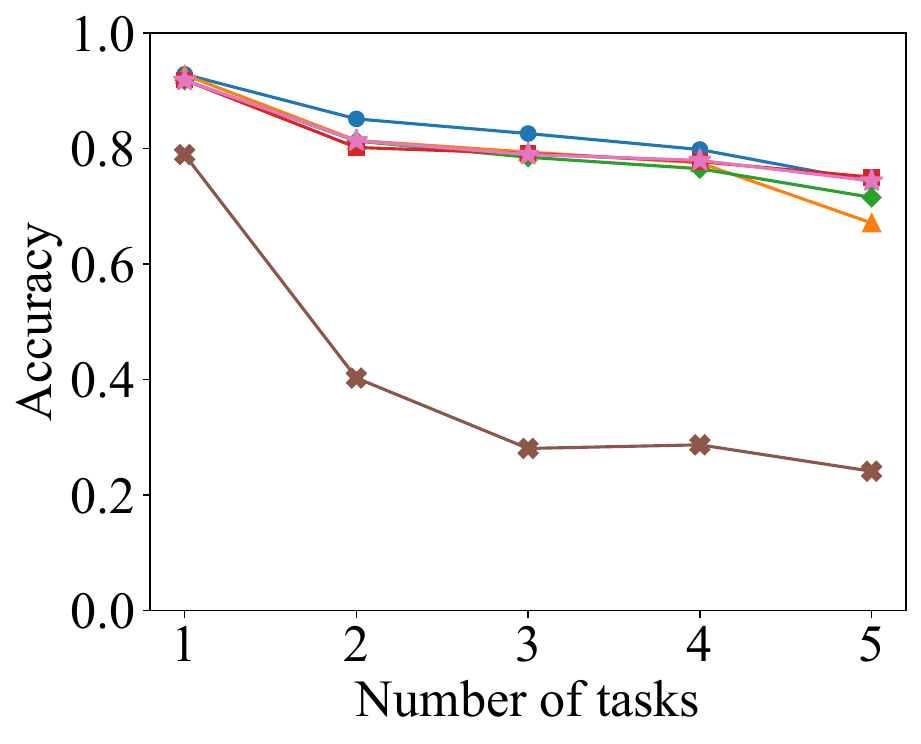}
        \caption{Accuracy.}
    \end{subfigure}
    \hspace{1cm}
    \begin{subfigure}[t]{0.278\textwidth}
        \includegraphics[width=\linewidth]{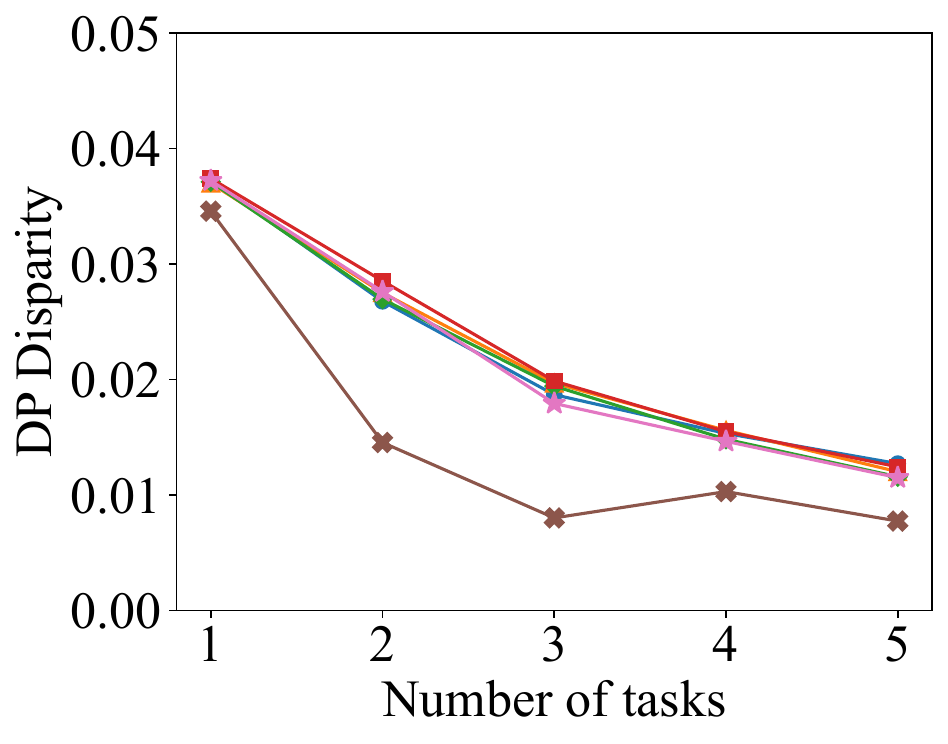}
        \caption{DP Disparity.}
    \end{subfigure}
    \vspace{-0.2cm}
    \caption{Sequential accuracy and fairness (DP) results on the BiasBios dataset.}\label{fig:biasbios_seq_acc_dp}
\end{figure*}

\clearpage

\subsection{More Results of \method{} when Varying the Buffer Size}
\label{appendix:buffer_size}

Continuing from Sec.~\ref{subsec:expresults}, we have additional experimental results of \method{} on the MNIST and Biased MNIST datasets when varying the buffer size to 16, 32, 64, and 128 per sensitive group as shown in Fig.~\ref{fig:exp_buffer_size}. As the buffer size increases, both accuracy and fairness performances improve. In addition, we compute the number of current task data assigned with non-zero weights as shown in Fig.~\ref{fig:exp_non_zero}, and there is no clear relationship between the buffer size and weights.

\begin{figure}[h!]
    \centering
    \begin{subfigure}[t]{0.492\textwidth}
        \includegraphics[width=\linewidth]{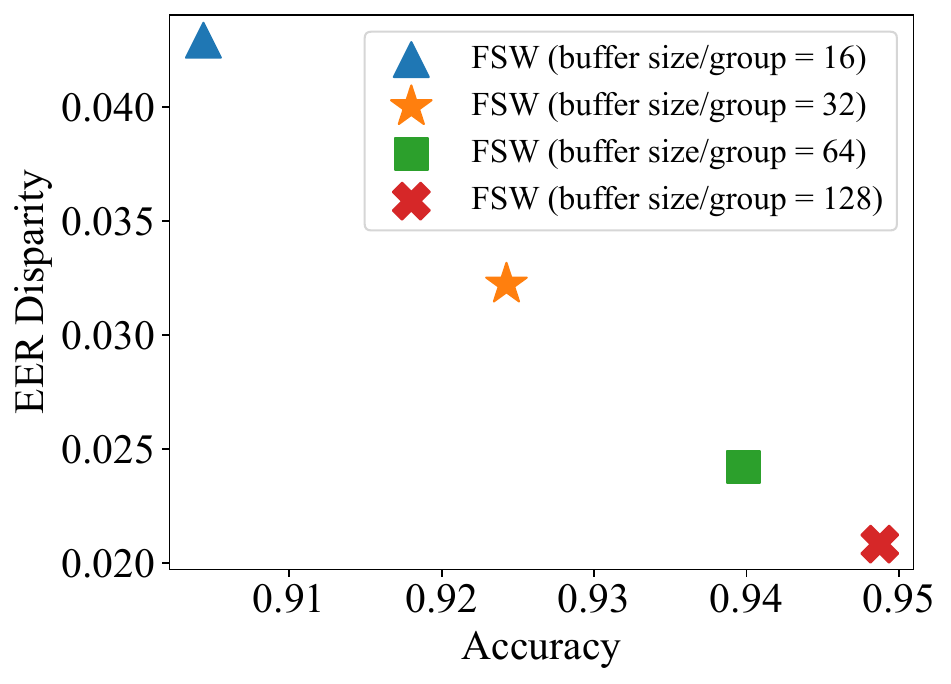}
        \caption{MNIST.}
    \end{subfigure}
    \begin{subfigure}[t]{0.48\textwidth}
        \includegraphics[width=\linewidth]{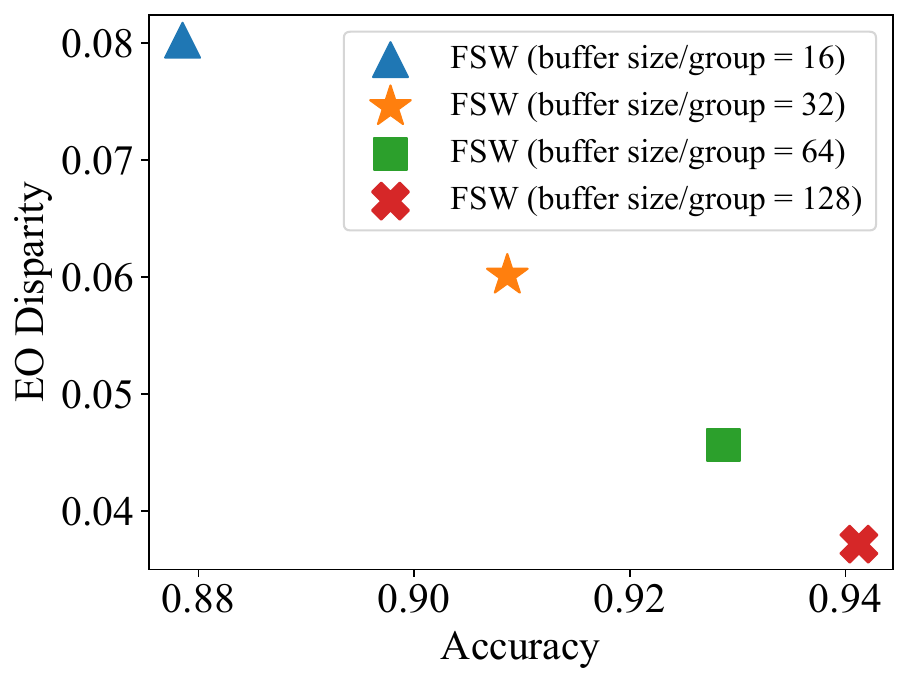}
        \caption{Biased MNIST.}
    \end{subfigure}
    \caption{Accuracy and fairness results of \method{} when varying the buffer size on the MNIST and Biased MNIST datasets.}\label{fig:exp_buffer_size}
\end{figure}

\begin{figure}[h!]
    \centering
    \begin{subfigure}[t]{0.485\textwidth}
        \includegraphics[width=\linewidth]{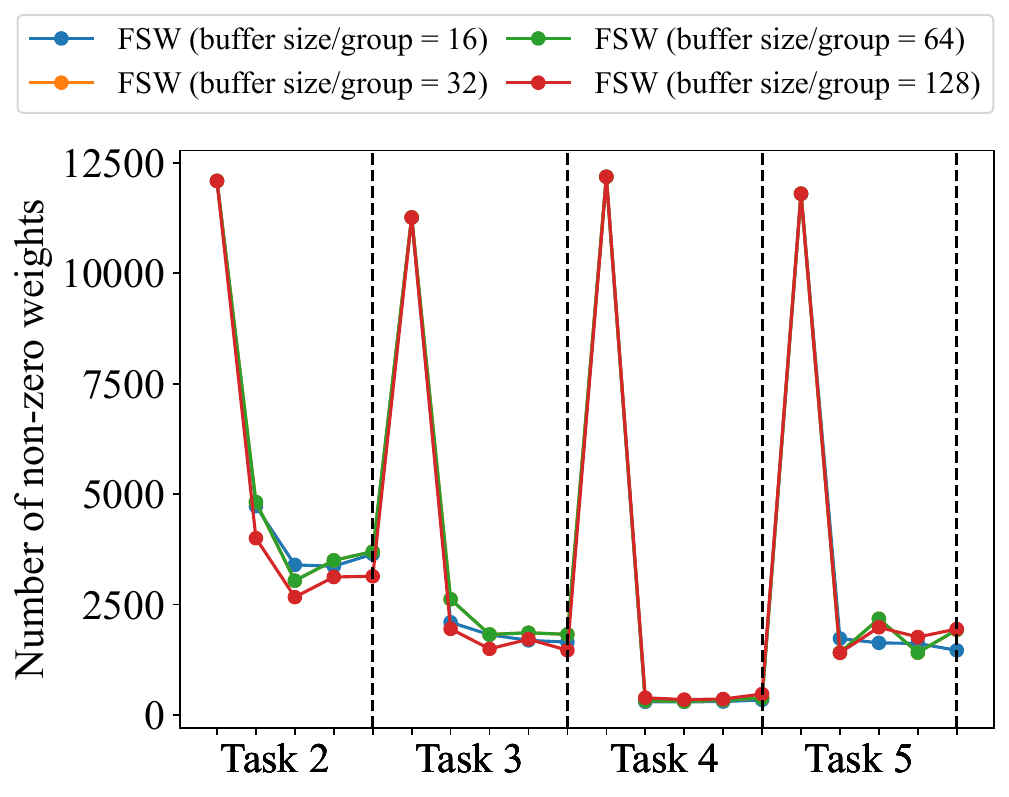}
        \caption{MNIST.}
    \end{subfigure}
    \begin{subfigure}[t]{0.485\textwidth}
        \includegraphics[width=\linewidth]{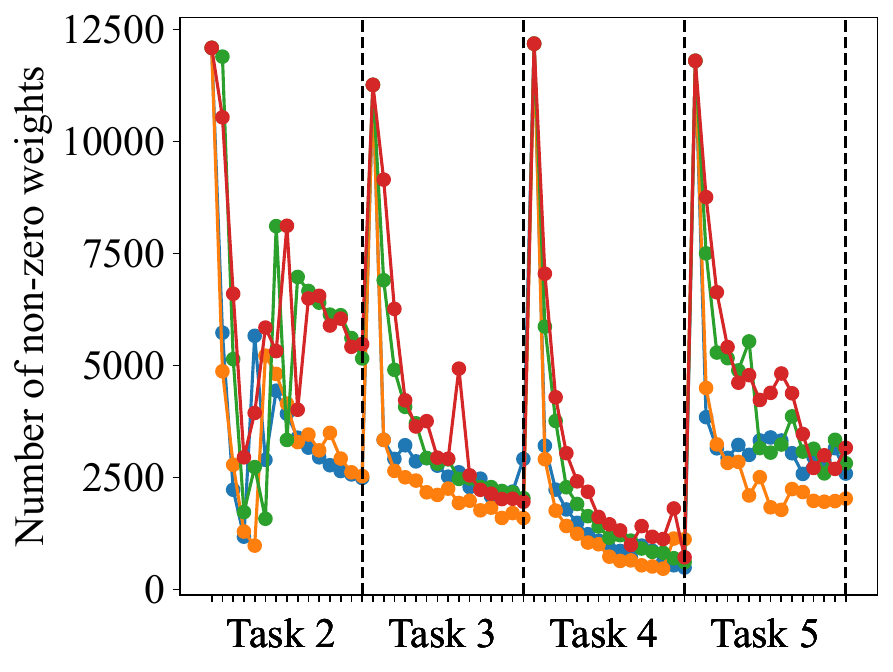}
        \caption{Biased MNIST.}
    \end{subfigure}
    \caption{Number of current task data assigned with non-zero weights when varying the buffer size on the MNIST and Biased MNIST datasets.}\label{fig:exp_non_zero}
\end{figure}






\newpage
\subsection{More Results on Sample Weighting Analysis}
\label{subsec:sampleweighting}

Continuing from Sec.~\ref{subsec:analysis}, we show more results from the sample weighting analysis for all sequential tasks of each dataset, as shown in the figures below (Fig.~\ref{fig:weights_mnist_eer}--Fig.~\ref{fig:weights_biasbios_dp}). 
\clearpage
\newpage

\begin{figure*}[h]
    \centering
    \begin{subfigure}[t]{0.248\textwidth}
        \includegraphics[width=\linewidth]{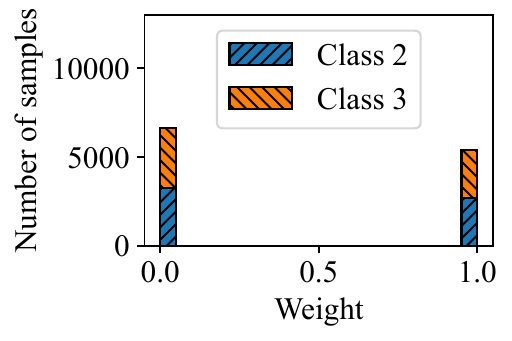}
        \caption{Task 2.}
    \end{subfigure}
    \begin{subfigure}[t]{0.2315\textwidth}
        \includegraphics[width=\linewidth]{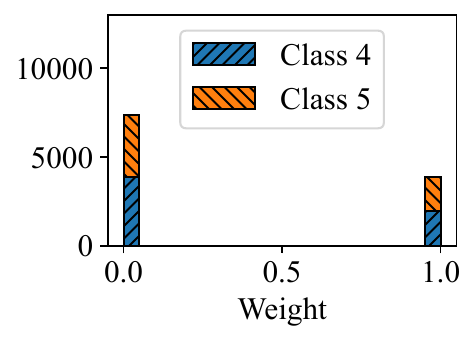}
        \caption{Task 3.}
    \end{subfigure}
    \begin{subfigure}[t]{0.2315\textwidth}
        \includegraphics[width=\linewidth]{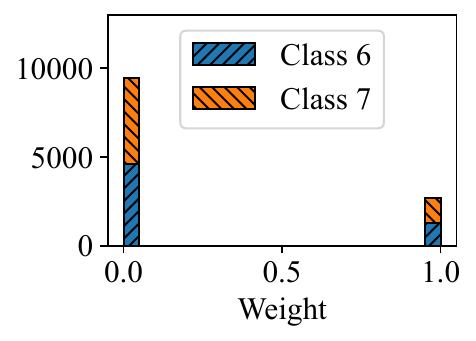}
        \caption{Task 4.}
    \end{subfigure}
    \begin{subfigure}[t]{0.2315\textwidth}
        \includegraphics[width=\linewidth]{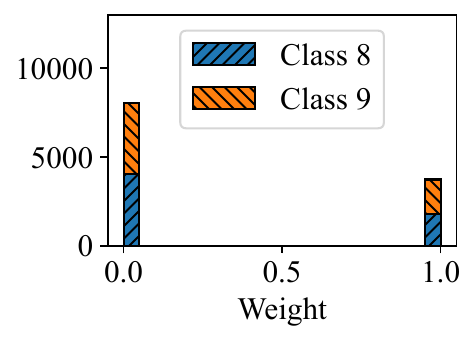}
        \caption{Task 5.}
    \end{subfigure}
    \caption{Distribution of sample weights for EER in sequential tasks of the MNIST dataset.}\label{fig:weights_mnist_eer}
\end{figure*}

\begin{figure*}[h]
    \centering
    \begin{subfigure}[t]{0.248\textwidth}
        \includegraphics[width=\linewidth]{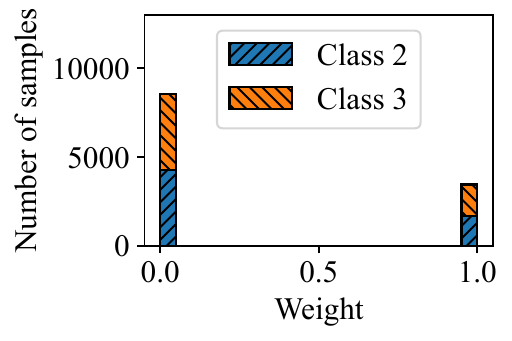}
        \caption{Task 2.}
    \end{subfigure}
    \begin{subfigure}[t]{0.2315\textwidth}
        \includegraphics[width=\linewidth]{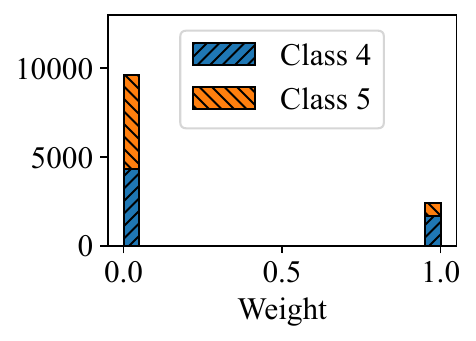}
        \caption{Task 3.}
    \end{subfigure}
    \begin{subfigure}[t]{0.2315\textwidth}
        \includegraphics[width=\linewidth]{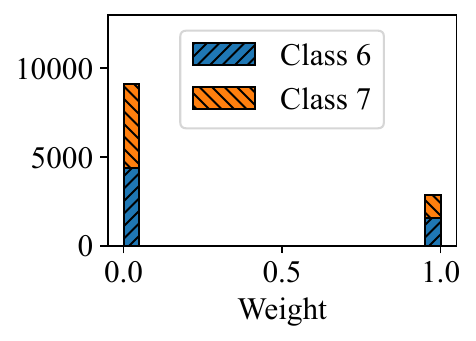}
        \caption{Task 4.}
    \end{subfigure}
    \begin{subfigure}[t]{0.2315\textwidth}
        \includegraphics[width=\linewidth]{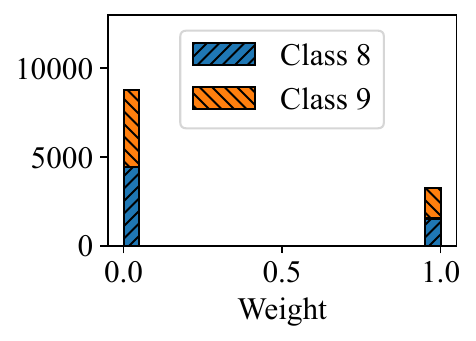}
        \caption{Task 5.}
    \end{subfigure}
    \caption{Distribution of sample weights for EER in sequential tasks of the FMNIST dataset.}\label{fig:weights_fmnist_eer}
\end{figure*}

\begin{figure*}[h]
    \centering
    \begin{subfigure}[t]{0.248\textwidth}
        \includegraphics[width=\linewidth]{figures/weight/BMNIST_EO_32_task_2.pdf}
        \caption{Task 2.}
    \end{subfigure}
    \begin{subfigure}[t]{0.2315\textwidth}
        \includegraphics[width=\linewidth]{figures/weight/BMNIST_EO_32_task_3.pdf}
        \caption{Task 3.}
    \end{subfigure}
    \begin{subfigure}[t]{0.2315\textwidth}
        \includegraphics[width=\linewidth]{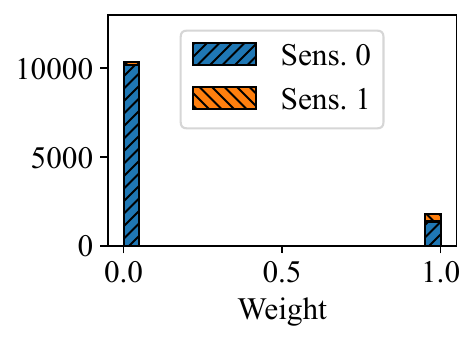}
        \caption{Task 4.}
    \end{subfigure}
    \begin{subfigure}[t]{0.2315\textwidth}
        \includegraphics[width=\linewidth]{figures/weight/BMNIST_EO_32_task_5.pdf}
        \caption{Task 5.}
    \end{subfigure}
    \caption{Distribution of sample weights for EO in sequential tasks of the Biased MNIST dataset.}\label{fig:weights_biased_mnist_eo}
\end{figure*}

\begin{figure*}[h]
    \centering
    \begin{subfigure}[t]{0.248\textwidth}
        \includegraphics[width=\linewidth]{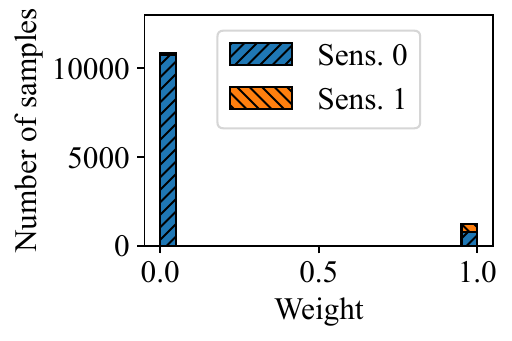}
        \caption{Task 2.}
    \end{subfigure}
    \begin{subfigure}[t]{0.2315\textwidth}
        \includegraphics[width=\linewidth]{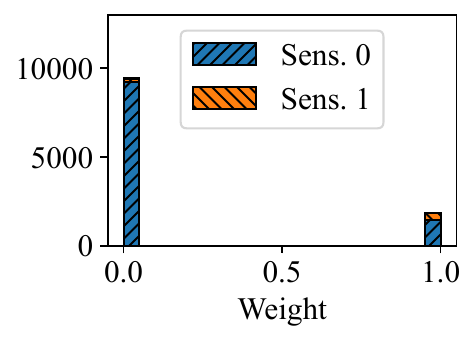}
        \caption{Task 3.}
    \end{subfigure}
    \begin{subfigure}[t]{0.2315\textwidth}
        \includegraphics[width=\linewidth]{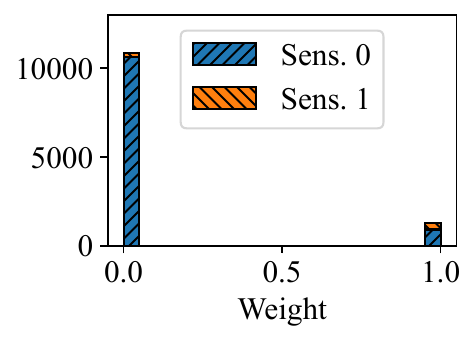}
        \caption{Task 4.}
    \end{subfigure}
    \begin{subfigure}[t]{0.2315\textwidth}
        \includegraphics[width=\linewidth]{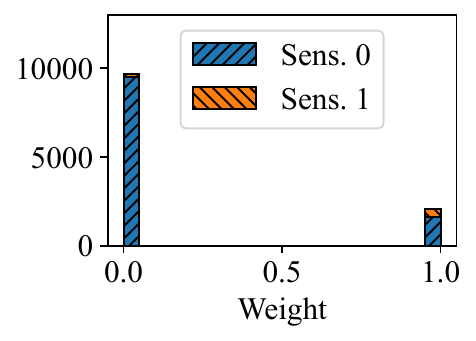}
        \caption{Task 5.}
    \end{subfigure}
    \caption{Distribution of sample weights for DP in sequential tasks of the Biased MNIST dataset.}\label{fig:weights_biased_mnist_dp}
\end{figure*}

\begin{figure*}[h]
    \centering
    \begin{subfigure}[t]{0.251\textwidth}
        \includegraphics[width=\linewidth]{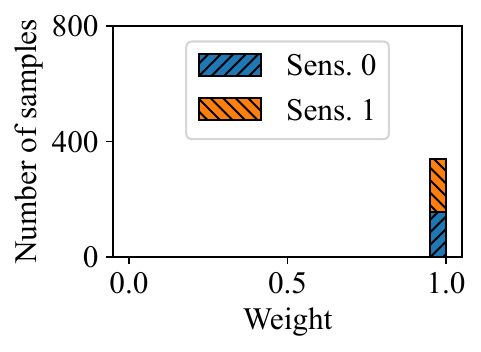}
        \caption{Task 2.}
    \end{subfigure}
    \begin{subfigure}[t]{0.234\textwidth}
        \includegraphics[width=\linewidth]{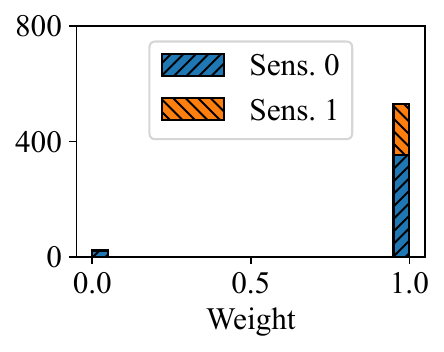}
        \caption{Task 3.}
    \end{subfigure}
    \caption{Distribution of sample weights for EO in sequential tasks of the DRUG dataset.}\label{fig:weights_drug_eo}
\end{figure*}

\begin{figure*}[h]
    \centering
    \begin{subfigure}[t]{0.251\textwidth}
        \includegraphics[width=\linewidth]{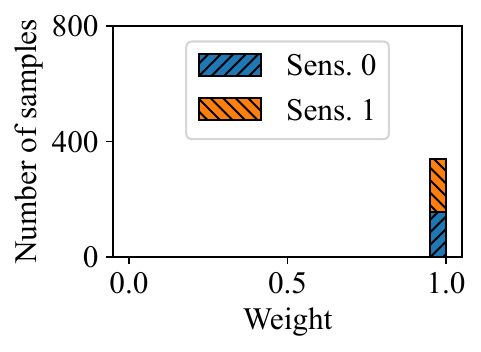}
        \caption{Task 2.}
    \end{subfigure}
    \begin{subfigure}[t]{0.234\textwidth}
        \includegraphics[width=\linewidth]{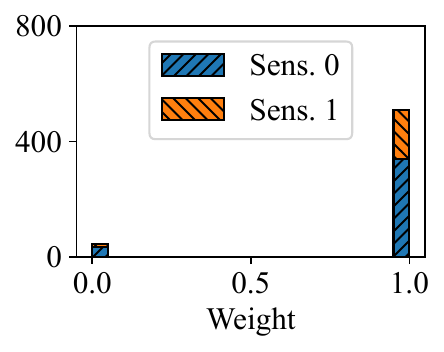}
        \caption{Task 3.}
    \end{subfigure}
    \caption{Distribution of sample weights for DP in sequential tasks of the DRUG dataset.}\label{fig:weights_drug_dp}
\end{figure*}

\begin{figure*}[h]
    \centering
    \begin{subfigure}[t]{0.248\textwidth}
        \includegraphics[width=\linewidth]{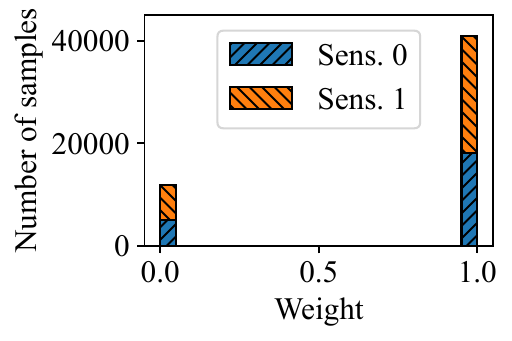}
        \caption{Task 2.}
    \end{subfigure}
    \begin{subfigure}[t]{0.2315\textwidth}
        \includegraphics[width=\linewidth]{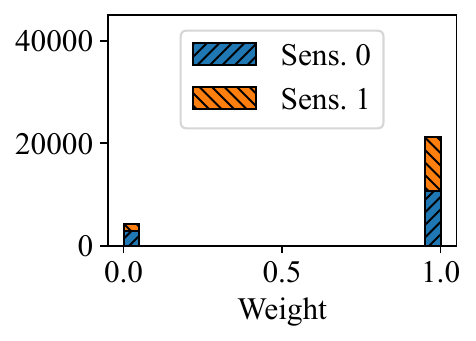}
        \caption{Task 3.}
    \end{subfigure}
    \begin{subfigure}[t]{0.2315\textwidth}
        \includegraphics[width=\linewidth]{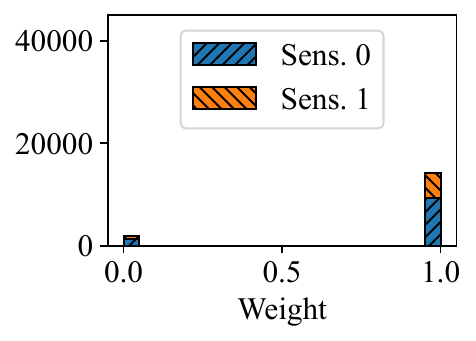}
        \caption{Task 4.}
    \end{subfigure}
    \begin{subfigure}[t]{0.2315\textwidth}
        \includegraphics[width=\linewidth]{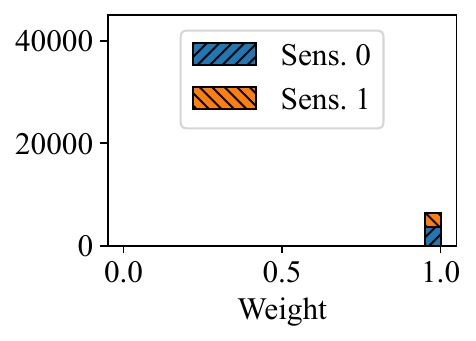}
        \caption{Task 5.}
    \end{subfigure}
    \caption{Distribution of sample weights for EO in sequential tasks of the BiasBios dataset.}\label{fig:weights_biasbios_eo}
\end{figure*}

\begin{figure*}[h]
    \centering
    \begin{subfigure}[t]{0.248\textwidth}
        \includegraphics[width=\linewidth]{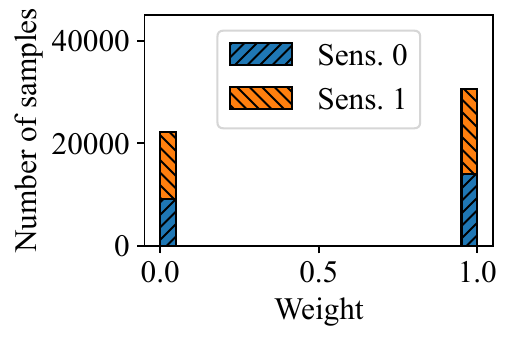}
        \caption{Task 2.}
    \end{subfigure}
    \begin{subfigure}[t]{0.2315\textwidth}
        \includegraphics[width=\linewidth]{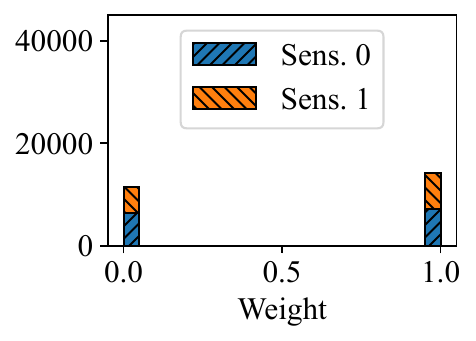}
        \caption{Task 3.}
    \end{subfigure}
    \begin{subfigure}[t]{0.2315\textwidth}
        \includegraphics[width=\linewidth]{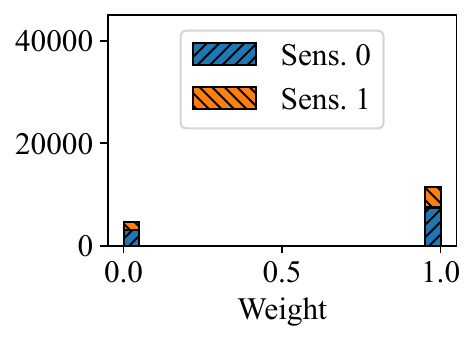}
        \caption{Task 4.}
    \end{subfigure}
    \begin{subfigure}[t]{0.2315\textwidth}
        \includegraphics[width=\linewidth]{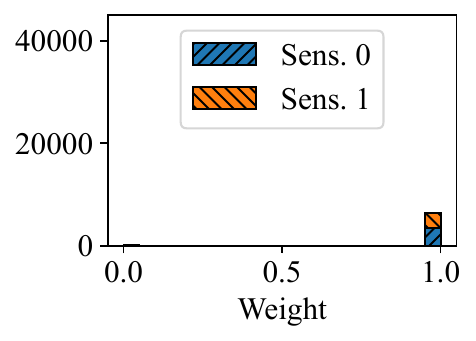}
        \caption{Task 5.}
    \end{subfigure}
    \caption{Distribution of sample weights for DP in sequential tasks of the BiasBios dataset.}\label{fig:weights_biasbios_dp}
\end{figure*}

\clearpage
\subsection{More Results on Ablation Study}
\label{appendix:ablation}

Continuing from Sec.~\ref{subsec:ablationstudy}, we present additional results of the ablation study to demonstrate the contribution of \method{} to the overall accuracy and fairness performance. The results are shown in Tables~\ref{tbl:performance_ablation_eer},~\ref{tbl:performance_ablation_eo}, and~\ref{tbl:performance_ablation_dp}.

\begin{table}[h]
  \setlength{\tabcolsep}{7.5pt}
  \caption{Accuracy and fairness results on the MNIST and FMNIST datasets, with respect to EER disparity with or without \method{}.} 
  \centering
  \begin{tabular}{l|cccc}
  \toprule
    {Methods} & \multicolumn{2}{c}{\sf  MNIST} & \multicolumn{2}{c}{\sf FMNIST} \\
    \cmidrule{1-5}
    {} & {Acc.} & {EER Disp.} & {Acc.} & {EER Disp.} \\
    \midrule
    {W/o \method{}} & .912\tiny{$\pm$.004} & .051\tiny{$\pm$.005} & .810\tiny{$\pm$.004} & .092\tiny{$\pm$.003} \\
    {\bf \method{}} & \textbf{.925\tiny{$\pm$.004}} & \textbf{.032\tiny{$\pm$.005}} & \textbf{.824\tiny{$\pm$.006}} & \textbf{.039\tiny{$\pm$.006}} \\
    \bottomrule
  \end{tabular}
  \label{tbl:performance_ablation_eer}
\end{table}

\begin{table}[h]
  \setlength{\tabcolsep}{1.6pt}
  \caption{Accuracy and fairness results on the Biased MNIST, DRUG, and BiasBios datasets with respect to EO disparity with or without \method{}.} 
  \centering
  \begin{tabular}{l|cccccc}
  \toprule
    {Methods} & \multicolumn{2}{c}{\sf Biased MNIST} & \multicolumn{2}{c}{\sf DRUG} & \multicolumn{2}{c}{\sf BiasBios} \\
    \cmidrule{1-7}
    {} & {Acc.} & {EO Disp.} & {Acc.} & {EO Disp.} & {Acc.} & {EO Disp.} \\
    \midrule
    {W/o \method{}} & \textbf{.910\tiny{$\pm$.003}} & .126\tiny{$\pm$.005} & .402\tiny{$\pm$.010} & .080\tiny{$\pm$.005} & .806\tiny{$\pm$.003} & .073\tiny{$\pm$.002} \\
    {\bf \method{}} & .909\tiny{$\pm$.004} & \textbf{.119\tiny{$\pm$.007}} & \textbf{.406\tiny{$\pm$.014}} & \textbf{.077\tiny{$\pm$.010}} & \textbf{.808\tiny{$\pm$.002}} & \textbf{.072\tiny{$\pm$.001}} \\
    \bottomrule
  \end{tabular}
  \label{tbl:performance_ablation_eo}
\end{table}

\begin{table}[h]
  \setlength{\tabcolsep}{1.6pt}
  \caption{Accuracy and fairness results on the Biased MNIST, DRUG, and BiasBios datasets with respect to DP disparity with or without \method{}.} 
  \centering
  \begin{tabular}{l|cccccc}
  \toprule
    {Methods} & \multicolumn{2}{c}{\sf Biased MNIST} & \multicolumn{2}{c}{\sf DRUG} & \multicolumn{2}{c}{\sf BiasBios} \\
    \cmidrule{1-7}
    {} & {Acc.} & {DP Disp.} & {Acc.} & {DP Disp.} & {Acc.} & {DP Disp.} \\
    \midrule
    {W/o \method{}} & \textbf{.910\tiny{$\pm$.003}} & .009\tiny{$\pm$.001} & .402\tiny{$\pm$.010} & .044\tiny{$\pm$.004} & .806\tiny{$\pm$.003} & \textbf{.022\tiny{$\pm$.000}} \\
    {\bf \method{}} & .904\tiny{$\pm$.004} & \textbf{.008\tiny{$\pm$.001}} & \textbf{.405\tiny{$\pm$.013}} & \textbf{.043\tiny{$\pm$.004}} & \textbf{.809\tiny{$\pm$.003}} & \textbf{.022\tiny{$\pm$.000}} \\
    \bottomrule
  \end{tabular}
  \label{tbl:performance_ablation_dp}
\end{table}

\subsection{Impact on Non-targeted Fairness Metrics}
\label{appendix:ablation_extension}

Continuing from Sec.~\ref{subsec:ablationstudy}, we further investigate the effect of optimizing \method{} for a specific fairness metric on other non-targeted fairness metrics. The results are shown in Tables~\ref{tbl:performance_ablation_extension_BMNIST}, ~\ref{tbl:performance_ablation_extension_DRUG}, and~\ref{tbl:performance_ablation_extension_BIOS}. 

\begin{table}[h]
  \setlength{\tabcolsep}{3.4pt}
  \caption{Accuracy and more fairness metrics (including EO, DP, and EER Disparity) results on the Biased MNIST dataset without \method{}, \method{} with respect to EO and DP disparity.} 
  \centering
  \begin{tabular}{l|cccc}
  \toprule
    {Methods} & {Acc.} & {EO Disp.} & {DP Disp.} & {EER Disp.} \\
    \midrule
    {W/o \method{}} & \textbf{.910\tiny{$\pm$.003}} & .126\tiny{$\pm$.005}& .009\tiny{$\pm$.001}& .032\tiny{$\pm$.002} \\
    {\method{} (w.r.t. EO Disp.)} & .909\tiny{$\pm$.004} & \textbf{.119\tiny{$\pm$.007}}& \textbf{.008\tiny{$\pm$.001}}& .032\tiny{$\pm$.003}\\
    {\method{} (w.r.t. DP Disp.)} & .904\tiny{$\pm$.004} & .124\tiny{$\pm$.008}& \textbf{.008\tiny{$\pm$.001}}& \textbf{.029\tiny{$\pm$.004}}\\
    \bottomrule
  \end{tabular}
  \label{tbl:performance_ablation_extension_BMNIST}
\end{table}

\begin{table}[h]
  \setlength{\tabcolsep}{3.4pt}
  \caption{Accuracy and more fairness metrics (including EO, DP, and EER Disparity) results on the DRUG dataset without \method{}, \method{} with respect to EO and DP disparity.} 
  \centering
  \begin{tabular}{l|cccc}
  \toprule
    {Methods} & {Acc.} & {EO Disp.} & {DP Disp.} & {EER Disp.} \\
    \midrule
    {W/o \method{}} & .402\tiny{$\pm$.010} & .080\tiny{$\pm$.005}& .044\tiny{$\pm$.004}& .078\tiny{$\pm$.008}\\
    {\method{} (w.r.t. EO Disp.)} & \textbf{.406\tiny{$\pm$.014}} & \textbf{.077\tiny{$\pm$.010}} & \textbf{.043\tiny{$\pm$.004}} & \textbf{.077\tiny{$\pm$.007}} \\
    {\method{} (w.r.t. DP Disp.)} & .405\tiny{$\pm$.013} & .079\tiny{$\pm$.006}& \textbf{.043\tiny{$\pm$.004}} & .078\tiny{$\pm$.006} \\
    \bottomrule
  \end{tabular}
  \label{tbl:performance_ablation_extension_DRUG}
\end{table}

\begin{table}[h]
  \setlength{\tabcolsep}{3.4pt}
  \caption{Accuracy and more fairness metrics (including EO, DP, and EER Disparity) results on the BiasBios dataset without \method{}, \method{} with respect to EO and DP disparity.} 
  \centering
  \begin{tabular}{l|cccc}
  \toprule
    {Methods} & {Acc.} & {EO Disp.} & {DP Disp.} & {EER Disp.} \\
    \midrule
    {W/o \method{}}  & .806\tiny{$\pm$.003} & .073\tiny{$\pm$.002} & \textbf{.022\tiny{$\pm$.000}} & .056\tiny{$\pm$.002}\\
    {\method{} (w.r.t. EO Disp.)} & .808\tiny{$\pm$.002} & .072\tiny{$\pm$.001} & \textbf{.022\tiny{$\pm$.000}} & .049\tiny{$\pm$.005}\\
    {\method{} (w.r.t. DP Disp.)} & \textbf{.809\tiny{$\pm$.003}} & \textbf{.069\tiny{$\pm$.002}} & \textbf{.022\tiny{$\pm$.000}}& \textbf{.036\tiny{$\pm$.009}}\\
    \bottomrule
  \end{tabular}
  \label{tbl:performance_ablation_extension_BIOS}
\end{table}
\begin{table*}[t]
  \setlength{\tabcolsep}{14.9pt}
  \caption{Accuracy and fairness (DP disparity) results when combining fair post-processing techniques ($\epsilon$-fair) with continual learning methods ({\it iCaRL}, {\it WA}, {\it CLAD}, {\it GSS}, {\it OCS}, and \method{}). Due to the excessive time ($>$5 days) required to run {\it OCS} on BiasBios, we are not able to measure the results and mark them as `--'.}
  \centering
  \begin{tabular}{l|cccccc}
  \toprule
    {Methods} & \multicolumn{2}{c}{\sf Biased MNIST} & \multicolumn{2}{c}{\sf DRUG} & \multicolumn{2}{c}{\sf BiasBios} \\
    \cmidrule{1-7}
    {} & {Acc.} & {DP Disp.} & {Acc.} & {DP Disp.} & {Acc.} & {DP Disp.} \\
    \midrule
    {iCaRL} & .802\tiny{$\pm$.008} & .015\tiny{$\pm$.001} & \textbf{.444\tiny{$\pm$.025}} & .093\tiny{$\pm$.009} & \textbf{.829\tiny{$\pm$.002}} & .022\tiny{$\pm$.000} \\
    {WA} & .916\tiny{$\pm$.002} & .009\tiny{$\pm$.001} & .408\tiny{$\pm$.022} & .067\tiny{$\pm$.013} & .796\tiny{$\pm$.003} & .022\tiny{$\pm$.000} \\
    {CLAD} & .871\tiny{$\pm$.012} & .013\tiny{$\pm$.001} & .410\tiny{$\pm$.026} & .069\tiny{$\pm$.019} & .799\tiny{$\pm$.003} & .022\tiny{$\pm$.000} \\
    {GSS}& .809\tiny{$\pm$.005} & .039\tiny{$\pm$.003} & .392\tiny{$\pm$.022} & .065\tiny{$\pm$.015} & .808\tiny{$\pm$.003} & .023\tiny{$\pm$.000} \\
    {OCS}& .824\tiny{$\pm$.007} & .035\tiny{$\pm$.003} & .393\tiny{$\pm$.017} & .053\tiny{$\pm$.012} & -- & -- \\
    {\bf \method{}}& .904\tiny{$\pm$.004} & .008\tiny{$\pm$.001} & .405\tiny{$\pm$.013} & .043\tiny{$\pm$.004} & \underline{.809\tiny{$\pm$.003}} & .022\tiny{$\pm$.000} \\
    \cmidrule{1-7}
    {iCaRL -- $\epsilon$-fair} & .944\tiny{$\pm$.008} & \underline{.006\tiny{$\pm$.002}} & \underline{.427\tiny{$\pm$.018}} & \underline{.026\tiny{$\pm$.004}} & .753\tiny{$\pm$.002} & \underline{.017\tiny{$\pm$.000}} \\
    {WA -- $\epsilon$-fair} & \textbf{.953\tiny{$\pm$.003}} & \underline{.006\tiny{$\pm$.002}} & .404\tiny{$\pm$.021} & .044\tiny{$\pm$.020} & .708\tiny{$\pm$.003} & \textbf{.016\tiny{$\pm$.000}} \\
    {CLAD -- $\epsilon$-fair} & .924\tiny{$\pm$.012} & \underline{.006\tiny{$\pm$.002}} & .406\tiny{$\pm$.027} & .030\tiny{$\pm$.010} & .716\tiny{$\pm$.004} & \textbf{.016\tiny{$\pm$.001}} \\
    {GSS -- $\epsilon$-fair} & .938\tiny{$\pm$.006} & \underline{.006\tiny{$\pm$.002}} & .382\tiny{$\pm$.014} & .035\tiny{$\pm$.017} & .717\tiny{$\pm$.005} & \textbf{.016\tiny{$\pm$.000}} \\
    {OCS -- $\epsilon$-fair} & \underline{.952\tiny{$\pm$.003}} & .032\tiny{$\pm$.004} & .384\tiny{$\pm$.009} & .051\tiny{$\pm$.002} & -- & -- \\
    {\bf \method{} -- $\epsilon$-fair} & .906\tiny{$\pm$.006} & \textbf{.005\tiny{$\pm$.001}} & .405\tiny{$\pm$.013} & \textbf{.021\tiny{$\pm$.004}} & .723\tiny{$\pm$.004} & \textbf{.016\tiny{$\pm$}.000} \\
    \bottomrule
  \end{tabular}
  \label{tbl:performance_appendix_post}
\end{table*}

\subsection{More Results on Integrating \method{} with a Fair Post-processing Method}
\label{appendix:postprocessing}

Continuing from Sec.~\ref{subsec:postprocessing}, we provide additional results on integrating with a fair post-processing method ($\epsilon$-fair) as shown in Table~\ref{tbl:performance_appendix_post}. 



\subsection{Alternative Loss Function for Group Fairness Metrics}
\label{appendix:alternative_loss_design}

Continuing from Sec.~\ref{sec:unfairforgetting}, we use cross-entropy loss disparity to approximate group fairness metrics such as EER, EO, and DP disparity. Both theoretical and empirical results show that the cross-entropy loss disparity can effectively approximate these group fairness metrics, as discussed in Sec.~\ref{appendix:CE-as-approximator}. However, the cross-entropy loss disparity is not the only possible type of loss for approximating the group fairness metrics; the disparity of other loss functions may yield better performance.
Our method can be applied regardless of the loss definition if (1) the loss update process can be linearly approximated (as in Sec.~\ref{appendix:linear-approximation}) and (2) the loss disparity promotes fairness (as in Sec.~\ref{appendix:CE-as-approximator}).

To verify if \method{} can also be effective with different loss function designs, we conduct simple experiments using hinge loss (i.e., $\sum_{j\neq y_i}\max(0, s_j-s_{y_i} + 1$) where $y_i$ is the true label, and $s_j$ is the softmax output for label $j$) to approximate group fairness metrics in \method{}. The results are shown in Tables~\ref{tbl:hingeloss_eer},~\ref{tbl:hingeloss_eo}, and~\ref{tbl:hingeloss_dp}. Overall, both methods show comparable accuracy-fairness results, suggesting that \method{} performs well regardless of the type of loss function used to approximate group fairness metrics. 
Here, we would like to note that the cross-entropy loss disparity is widely used and empirically verified as a reasonable proxy for capturing group fairness metrics~\citep{DBLP:conf/naacl/ShenHCBF22, DBLP:conf/iclr/Roh0WS21, roh2023drfairness, gupta2024fairly}, which is why we use it, although we could also use other losses. 

\begin{table}[H]
  \setlength{\tabcolsep}{6.6pt}
  \caption{Accuracy and fairness results on the MNIST and FMNIST datasets with respect to EER disparity. ``FSW (hinge)'' uses hinge loss, while ``FSW'' uses cross-entropy loss to approximate the group fairness metric.} 
  \centering
  \begin{tabular}{l|cccc}
  \toprule
    {Methods} & \multicolumn{2}{c}{\sf MNIST} & \multicolumn{2}{c}{\sf FMNIST} \\
    \cmidrule{1-5}
    {} & {Acc.} & {EER Disp.} & {Acc.} & {EER Disp.} \\
    \midrule
    {\method{}} & \textbf{.925\tiny{$\pm$.004}} & .032\tiny{$\pm$.005} & .824\tiny{$\pm$.006} & \textbf{.039\tiny{$\pm$.006}} \\
    {\method{} (hinge)} & \textbf{.925\tiny{$\pm$.003}} & \textbf{.030\tiny{$\pm$.006}} & \textbf{.825\tiny{$\pm$.006}} & \textbf{.039\tiny{$\pm$.005}} \\
    \bottomrule
  \end{tabular}
  \label{tbl:hingeloss_eer}
\end{table}

\begin{table}[H]
  \setlength{\tabcolsep}{0.95pt}
  \caption{Accuracy and fairness results on the Biased MNIST, DRUG, and BiasBios datasets with respect to EO disparity. The other settings are the same as in Table~\ref{tbl:hingeloss_eer}.} 
  \centering
  \begin{tabular}{l|cccccc}
  \toprule
    {Methods} & \multicolumn{2}{c}{\sf Biased MNIST} & \multicolumn{2}{c}{\sf DRUG} & \multicolumn{2}{c}{\sf BiasBios} \\
    \cmidrule{1-7}
    {} & {Acc.} & {EO Disp.} & {Acc.} & {EO Disp.} & {Acc.} & {EO Disp.} \\
    \midrule
    {\method{}} & \textbf{.909\tiny{$\pm$.004}} & \textbf{.119\tiny{$\pm$.007}} & \textbf{.406\tiny{$\pm$.014}} & \textbf{.077\tiny{$\pm$.010}} & \textbf{.808\tiny{$\pm$.002}} & .072\tiny{$\pm$.001} \\
    {\method{} (hinge)} & \textbf{.909\tiny{$\pm$.004}} & \textbf{.119\tiny{$\pm$.006}} & \textbf{.406\tiny{$\pm$.014}} & \textbf{.077\tiny{$\pm$.010}} & .807\tiny{$\pm$.002} & \textbf{.071\tiny{$\pm$.002}} \\
    \bottomrule
  \end{tabular}
  \label{tbl:hingeloss_eo}
\end{table}

\begin{table}[H]
  \setlength{\tabcolsep}{1pt}
  \caption{Accuracy and fairness results on the Biased MNIST, DRUG, and BiasBios datasets with respect to DP disparity. The other settings are the same as in Table~\ref{tbl:hingeloss_eer}.} 
  \centering
  \begin{tabular}{l|cccccc}
  \toprule
    {Methods} & \multicolumn{2}{c}{\sf Biased MNIST} & \multicolumn{2}{c}{\sf DRUG} & \multicolumn{2}{c}{\sf BiasBios} \\
    \cmidrule{1-7}
    {} & {Acc.} & {DP Disp.} & {Acc.} & {DP Disp.} & {Acc.} & {DP Disp.} \\
    \midrule
    {\method{}} & \textbf{.904\tiny{$\pm$.004}} & \textbf{.008\tiny{$\pm$.001}} & \textbf{.405\tiny{$\pm$.013}} & \textbf{.043\tiny{$\pm$.004}} & \textbf{.809\tiny{$\pm$.003}} & \textbf{.022\tiny{$\pm$.000}} \\
    {\method{} (hinge)} & \textbf{.904\tiny{$\pm$.004}} & \textbf{.008\tiny{$\pm$.001}} & \textbf{.405\tiny{$\pm$.013}} & \textbf{.043\tiny{$\pm$.004}} & .807\tiny{$\pm$.006} & \textbf{.022\tiny{$\pm$.000}} \\
    \bottomrule
  \end{tabular}
  \label{tbl:hingeloss_dp}
\end{table}

\subsection{Fairness Considerations in Buffering Strategy}
\label{appendix:buffer_strategy}

Continuing from Sec.~\ref{subsec:alg}, we further elaborate on the distinct roles of buffering and sample weighting in our framework. Buffering primarily aims to preserve representative and diverse samples from previously learned groups, whereas sample weighting explicitly addresses the mitigation of unfair catastrophic forgetting during model training.
Selecting buffer samples based on high fairness weights—those considered important for current-task fairness—might initially appear beneficial; however, this approach could inadvertently distort the original distribution of groups within the buffer. Modified distributions might negatively impact classification performance and complicate the accurate determination of fairness weights in future tasks. This complication arises because the buffer acts as a reference point for evaluating both fairness and accuracy metrics. Consequently, to maintain representative distributions and ensure reliable fairness measurement across tasks, we utilize a simple random sampling technique for buffer sample selection.
For a fair comparison, we apply the same buffer usage across other baselines, without altering each method’s original buffer selection rule (e.g., herding or gradient-based selection).

To empirically validate this reasoning, we compare several buffering strategies, including random sampling (the original \method{}) and sampling with a large-weight-first strategy. The results are shown in Tables~\ref{tbl:buffer_strategy_eer},~\ref{tbl:buffer_strategy_eo}, and~\ref{tbl:buffer_strategy_dp}. Although the large-weight-first sampling strategy achieves comparable fairness results, the original random sampling method generally achieves superior accuracy, aligning with our theoretical expectations, while providing comparable or improved fairness performance.

\begin{table}[H]
  \setlength{\tabcolsep}{6.0pt}
  \small
  \caption{Accuracy and fairness results on the MNIST and FMNIST datasets with respect to EER disparity. ``FSW (large-weight-first)'' denotes a buffering strategy prioritizing samples with large weights, while ``FSW'' uses simple random sampling.} 
  \centering
  \begin{tabular}{L{2.2cm}|cccc}
  \toprule
    {Methods} & \multicolumn{2}{c}{\sf MNIST} & \multicolumn{2}{c}{\sf FMNIST} \\
    \cmidrule{1-5}
    {} & {Acc.} & {EER Disp.} & {Acc.} & {EER Disp.} \\
    \midrule
    {\method{}} & .925\tiny{$\pm$.004} & .032\tiny{$\pm$.005} & .824\tiny{$\pm$.006} & \textbf{.039\tiny{$\pm$.006}} \\
    {\method{} (large-weight-first)} & \raisebox{-1.5ex}{\textbf{.927\tiny{$\pm$.005}}} & \raisebox{-1.5ex}{\textbf{.030\tiny{$\pm$.003}}} & \raisebox{-1.5ex}{\textbf{.827\tiny{$\pm$.005}}} & \raisebox{-1.5ex}{.041\tiny{$\pm$.008}} \\
    \bottomrule
  \end{tabular}
  \label{tbl:buffer_strategy_eer}
\end{table}

\begin{table}[H]
  \setlength{\tabcolsep}{0.8pt}
  \small
  \caption{Accuracy and fairness results on the Biased MNIST, DRUG, and BiasBios datasets with respect to EO disparity. The other settings are the same as in Table~\ref{tbl:buffer_strategy_eer}.} 
  \centering
  \begin{tabular}{L{2.25cm}|cccccc}
  \toprule
    {\;\;Methods} & \multicolumn{2}{c}{\sf Biased MNIST} & \multicolumn{2}{c}{\sf DRUG} & \multicolumn{2}{c}{\sf BiasBios} \\
    \cmidrule{1-7}
    {} & {Acc.} & {EO Disp.} & {Acc.} & {EO Disp.} & {Acc.} & {EO Disp.} \\
    \midrule
    {\;\;\method{}} & \textbf{.909\tiny{$\pm$.004}} & \textbf{.119\tiny{$\pm$.007}} & \textbf{.406\tiny{$\pm$.014}} & \textbf{.077\tiny{$\pm$.010}} & \textbf{.808\tiny{$\pm$.002}} & .072\tiny{$\pm$.001} \\
    {\;\;\method{} \,(large-weight-first)} & \raisebox{-1.5ex}{.893\tiny{$\pm$.006}} & \raisebox{-1.5ex}{.124\tiny{$\pm$.005}} & \raisebox{-1.5ex}{.405\tiny{$\pm$.012}} & \raisebox{-1.5ex}{\textbf{.077\tiny{$\pm$.007}}} & \raisebox{-1.5ex}{.800\tiny{$\pm$.005}} & \raisebox{-1.5ex}{\textbf{.071\tiny{$\pm$.001}}} \\
    \bottomrule
  \end{tabular}
  \label{tbl:buffer_strategy_eo}
\end{table}

\begin{table}[H]
  \setlength{\tabcolsep}{.8pt}
  \small
  \caption{Accuracy and fairness results on the Biased MNIST, DRUG, and BiasBios datasets with respect to DP disparity. The other settings are the same as in Table~\ref{tbl:buffer_strategy_eer}.} 
  \centering
  \begin{tabular}{L{2.25cm}|cccccc}
  \toprule
    {\;\;Methods} & \multicolumn{2}{c}{\sf Biased MNIST} & \multicolumn{2}{c}{\sf DRUG} & \multicolumn{2}{c}{\sf BiasBios} \\
    \cmidrule{1-7}
    {} & {Acc.} & {DP Disp.} & {Acc.} & {DP Disp.} & {Acc.} & {DP Disp.} \\
    \midrule
    {\;\;\method{}} & \textbf{.904\tiny{$\pm$.004}} & \textbf{.008\tiny{$\pm$.001}} & \textbf{.405\tiny{$\pm$.013}} & .043\tiny{$\pm$.004} & \textbf{.809\tiny{$\pm$.003}} & \textbf{.022\tiny{$\pm$.000}} \\
    {\;\;\method{} \,(large-weight-first)} & \raisebox{-1.5ex}{.876\tiny{$\pm$.009}} & \raisebox{-1.5ex}{\textbf{.008\tiny{$\pm$.001}}} & \raisebox{-1.5ex}{.404\tiny{$\pm$.013}} & \raisebox{-1.5ex}{\textbf{.041\tiny{$\pm$.005}}} & \raisebox{-1.5ex}{.787\tiny{$\pm$.014}} & \raisebox{-1.5ex}{\textbf{.022\tiny{$\pm$.001}}} \\
    \bottomrule
  \end{tabular}
  \label{tbl:buffer_strategy_dp}
\end{table}

\subsection{Utilizing the Gradients of the Model's Last Layer}
\label{appendix:gradient_last_layer}

Continuing from Sec.~\ref{subsec:alg}, we utilize the gradients of the model’s last layer during computation of the gradient vectors. Although using only the gradients from the model's last layer may result in the loss of information from earlier layers, it significantly improves computational efficiency by reducing the gradient calculation time. This approach is theoretically supported by several studies\,\citep{DBLP:conf/icml/KatharopoulosF18, DBLP:conf/iclr/AshZK0A20, DBLP:conf/icml/MirzasoleimanBL20}.
Neglecting earlier layers is equivalent to substituting the gradient of model parameters in $\nabla_{\theta} \ell(f_{\theta}^{l-1}, d_i)$ of Eq.~\ref{eq:loss_update} to zero when the model parameter $\theta$ does not belong to the last layer. This simplification aligns with the statements from the references, ``the variation of the gradient norm is mostly captured by the gradient of the loss function with respect to the pre-activation outputs of the last layer,'' which implies that the last layer parameters dominate gradient contributions, making the parameters from earlier layers ignorable.
Several studies\,\citep{DBLP:conf/aaai/KillamsettySRI21, DBLP:conf/icml/KillamsettySRDI21, DBLP:conf/aaai/KimHW24} also only utilize the last layer's gradients for efficiency.  

To observe the efficiency of utilizing the last layer's gradient, we conduct experiments comparing the performance and runtime between utilizing the full gradient and the last layer's gradient, as shown in Table~\ref{tbl:gradient_last_layer}. Using only the last layer's gradients provides comparable accuracy and fairness to using gradients from all layers, while significantly reducing the gradient computation time. The CPLEX computation time is comparable since the number of variables and constraints remains unchanged. In addition, the model training time increases when using gradients from all layers, primarily due to an increase in the average number of non-zero sample weights (from 2863.4 to 4280.5), resulting in a higher number of batches during training.

\begin{table}[H]
  \setlength{\tabcolsep}{1.1pt}
  \caption{Accuracy, fairness, and runtime results on the Biased MNIST datasets with respect to EO disparity.  ``FSW'' utilizes gradients of the model’s last layer during computation of the gradient vectors, while ``FSW (all layers)'' utilizes gradients of all model parameters during computation.} 
  \vspace{-0.cm}

  \centering
  \begin{tabular}{l|C{0.9cm}C{0.9cm}C{1.56cm}C{1.45cm}C{1.16cm}}
  \toprule
    {Methods} & {Acc.} & {EO Disp.} & {Gradient Computation} & {CPLEX Computation} & {Model Training} \\
    \midrule
    {\method{}} & .909\tiny{$\pm$.004} & .119\tiny{$\pm$.007} & \textbf{233\tiny{$\pm$017}} (s) & \textbf{507\tiny{$\pm$004}} (s) & \textbf{94\tiny{$\pm$003}} (s) \\
    {\method{} (all layers)} & \textbf{.910\tiny{$\pm$.005}} & \textbf{.117\tiny{$\pm$.009}} & 23,843\tiny{$\pm$882} (s) & 578\tiny{$\pm$003} (s) & 184\tiny{$\pm$010} (s)\\
    \bottomrule
  \end{tabular}
  \label{tbl:gradient_last_layer}
  \vspace{-0.05cm}
\end{table}

\subsection{Alternative Optimization Solver for Linear Programming}
\label{appendix:optimization_solver}

Continuing from Sec.~\ref{subsec:expsettings}, we conduct experiments comparing CPLEX with another optimization solver, HiGHS\,\citep{huangfu2018parallelizing}. Tables~\ref{tbl:optimization_solver_eer},~\ref{tbl:optimization_solver_eo}, and~\ref{tbl:optimization_solver_dp} show accuracy and fairness on \method{} with CPLEX and HiGHS solver with respect to EER, EO, and DP disparity.
We observed that both solvers show comparable accuracy-fairness results, suggesting that FSW is robust to the choice of LP solver in practice. 
While other solvers can be used, we use CPLEX in our paper due to its strong solution stability and high speed in solving large-scale problems with numerous variables and constraints. CPLEX also provides comprehensive tools for in-depth solution analysis, and is thus widely used in the literature \,\citep{sridhar2013approximate, gearhart2013comparison, zhang2023iflipper}.

\begin{table}[H]
  \setlength{\tabcolsep}{5.7pt}
  \small
  \caption{Accuracy and fairness results on the MNIST and FMNIST datasets with respect to EER disparity. ``FSW (using CPLEX)'' uses CPLEX as an optimization solver, while ``FSW (using HiGHS)'' uses HiGHS as an optimization solver.} 
  \vspace{-0.cm}
  \centering
  \begin{tabular}{l|cccc}
  \toprule
    {Methods} & \multicolumn{2}{c}{\sf MNIST} & \multicolumn{2}{c}{\sf FMNIST} \\
    \cmidrule{1-5}
    {} & {Acc.} & {EER Disp.} & {Acc.} & {EER Disp.} \\
    \midrule
    {\method{} (using CPLEX)} & \textbf{.925\tiny{$\pm$.004}} & .032\tiny{$\pm$.005} & .824\tiny{$\pm$.006} & \textbf{.039\tiny{$\pm$.006}} \\
    {\method{} (using HiGHS)} & \textbf{.925\tiny{$\pm$.004}} & \textbf{.031\tiny{$\pm$.006}} & \textbf{.825\tiny{$\pm$.006}} & \textbf{.039\tiny{$\pm$.007}} \\
    \bottomrule
  \end{tabular}
  \label{tbl:optimization_solver_eer}
  \vspace{-0.05cm}
\end{table}

\begin{table}[H]
  \setlength{\tabcolsep}{0.7pt}
  \small
  \caption{Accuracy and fairness results on the Biased MNIST, DRUG, and BiasBios datasets with respect to EO disparity. The other settings are the same as in Table~\ref{tbl:optimization_solver_eer}.} 
  \vspace{-0.cm}
  \centering
  \begin{tabular}{l|cccccc}
  \toprule
    {Methods} & \multicolumn{2}{c}{\sf Biased MNIST} & \multicolumn{2}{c}{\sf DRUG} & \multicolumn{2}{c}{\sf BiasBios} \\
    \cmidrule{1-7}
    {} & {Acc.} & {EO Disp.} & {Acc.} & {EO Disp.} & {Acc.} & {EO Disp.} \\
    \midrule
    {\method{} (using CPLEX)} & \textbf{.909\tiny{$\pm$.004}} & .119\tiny{$\pm$.007} & \textbf{.406\tiny{$\pm$.014}} & \textbf{.077\tiny{$\pm$.010}} & \textbf{.808\tiny{$\pm$.002}} & .072\tiny{$\pm$.001} \\
    {\method{} (using HiGHS)} & .907\tiny{$\pm$.005} & \textbf{.117\tiny{$\pm$.006}} & \textbf{.406\tiny{$\pm$.014}} & \textbf{.077\tiny{$\pm$.010}} & .805\tiny{$\pm$.004} & \textbf{.070\tiny{$\pm$.002}} \\
    \bottomrule
  \end{tabular}
  \label{tbl:optimization_solver_eo}
\vspace{-0.05cm}
\end{table}

\begin{table}[H]
  \setlength{\tabcolsep}{0.7pt}
  \small
  \caption{Accuracy and fairness results on the Biased MNIST, DRUG, and BiasBios datasets with respect to DP disparity. The other settings are the same as in Table~\ref{tbl:optimization_solver_eer}.} 
  \vspace{-0.cm}
  \centering
  \begin{tabular}{l|cccccc}
  \toprule
    {Methods} & \multicolumn{2}{c}{\sf Biased MNIST} & \multicolumn{2}{c}{\sf DRUG} & \multicolumn{2}{c}{\sf BiasBios} \\
    \cmidrule{1-7}
    {} & {Acc.} & {DP Disp.} & {Acc.} & {DP Disp.} & {Acc.} & {DP Disp.} \\
    \midrule
    {\method{} (using CPLEX)} & \textbf{.904\tiny{$\pm$.004}} & \textbf{.008\tiny{$\pm$.001}} & \textbf{.405\tiny{$\pm$.013}} & \textbf{.043\tiny{$\pm$.004}} & \textbf{.809\tiny{$\pm$.003}} & \textbf{.022\tiny{$\pm$.000}} \\
    {\method{} (using HiGHS)} & .902\tiny{$\pm$.006} & .009\tiny{$\pm$.002} & \textbf{.405\tiny{$\pm$.013}} & \textbf{.043\tiny{$\pm$.004}} & .807\tiny{$\pm$.004} & \textbf{.022\tiny{$\pm$.000}}\\
    \bottomrule
  \end{tabular}
  \label{tbl:optimization_solver_dp}
  \vspace{-0.1cm}
\end{table}

\section{More Related Work}
\label{appendix:related_work}

Continuing from Sec.~\ref{sec:relatedwork}, we discuss more related work.

Class-incremental learning is a challenging type of continual learning where a model continuously learns new tasks, each composed of new disjoint classes, and the goal is to minimize catastrophic forgetting\,\citep{DBLP:journals/ijon/MaiLJQKS22, DBLP:journals/pami/MasanaLTMBW23}. Data replay techniques\,\citep{DBLP:conf/nips/Lopez-PazR17, DBLP:conf/cvpr/RebuffiKSL17, DBLP:journals/corr/abs-1902-10486} store a small portion of previous data in a buffer to utilize for training and are widely used with other techniques\,\citep{DBLP:journals/corr/abs-2302-03648} including knowledge distillation\,\citep{DBLP:conf/cvpr/RebuffiKSL17, DBLP:conf/nips/BuzzegaBPAC20}, model rectification\,\citep{DBLP:conf/cvpr/WuCWYLGF19, DBLP:conf/cvpr/ZhaoXGZX20}, and dynamic networks\,\citep{DBLP:conf/cvpr/YanX021, DBLP:conf/eccv/WangZYZ22, DBLP:conf/iclr/0001WYZ23}. Simple buffer sample selection methods such as random or herding-based approaches\,\citep{DBLP:conf/cvpr/RebuffiKSL17} are also commonly used as well. There are also gradient-based sample selection techniques like GSS\,\citep{DBLP:conf/nips/AljundiLGB19} and OCS\,\citep{DBLP:conf/iclr/YoonMYH22} that manage buffer data to have samples with diverse and representative gradient vectors. The core contribution of \method{} is its capability to separately address different groups through sample weighting, and all baseline works do not consider fairness and simply assume that the entire incoming data is used for model training, which may result in unfair forgetting, as we show in our experiments.

Model fairness research mitigates bias by ensuring that a model's performance is equitable across different sensitive groups, thereby preventing discrimination based on race, gender, age, or other sensitive attributes\,\citep{DBLP:journals/csur/MehrabiMSLG21}. Existing model fairness techniques can be categorized as pre-processing\,\citep{DBLP:journals/kais/KamiranC11, DBLP:conf/kdd/FeldmanFMSV15, DBLP:conf/nips/CalmonWVRV17, DBLP:conf/aistats/JiangN20}, in-processing\,\citep{DBLP:conf/icml/AgarwalBD0W18, DBLP:conf/aies/ZhangLM18, DBLP:conf/alt/CotterJS19, DBLP:conf/icml/RohLWS20}, and post-processing\,\citep{DBLP:conf/nips/HardtPNS16, DBLP:conf/nips/PleissRWKW17, DBLP:conf/nips/ChzhenDHOP19}. In addition, there are other techniques that assign adaptive weights for samples to improve fairness\,\citep{chai2022fairness, jung2023re}. However, most of these techniques assume that the training data is given all at once, which may not be realistic. There are techniques for fairness-aware active learning\,\citep{anahideh2022fair, pang2024fairness, tae2024falcon}, in which the training data evolves with the acquisition of samples. However, these techniques store all labeled data and use them for training, which is impractical in continual learning settings.

A recent study addresses model fairness in class-incremental learning, where models may suffer from imbalanced forgetting of previously learned sensitive groups including classes, leading to unfairness across different groups\,\citep{he2024gradient, xu2024defying, DBLP:conf/aaai/ChowdhuryC23, DBLP:conf/nips/TruongNRL23} (see more details in Sec.~\ref{appendix:related_work}). 
A recent study\,\citep{he2024gradient} addresses the dual imbalance problem involving both inter-task and intra-task imbalance by reweighting gradients. However, the bias arises not only from data imbalance, but also from inherent or acquired characteristics of data itself\,\citep{DBLP:journals/csur/MehrabiMSLG21, angwin2022machine}. 
CLAD\,\citep{xu2024defying} first discovers imbalanced forgetting between classes caused by conflicts in representation and proposes a class-aware disentanglement technique to improve accuracy. Among the fairness-aware techniques, 
FaIRL\,\citep{DBLP:conf/aaai/ChowdhuryC23} supports group fairness metrics like demographic parity for continual learning, but proposes a representation learning method that does not directly optimize the given fairness measure and thus has limitations in improving fairness as we show in experiments. 
FairCL\,\citep{DBLP:conf/nips/TruongNRL23} also addresses fairness in a continual learning setup, but only focuses on resolving the imbalanced class distribution based on the number of pixels of each class in an image for semantic segmentation tasks. 
In comparison, we can mitigate bias from both data imbalance and intrinsic or acquired characteristics within the data with satisfying multiple notions of group fairness for sensitive groups including classes.

\section{Future Work}
\label{appendix:future_work}
Continuing from Sec.~\ref{sec:conclusion}, we discuss the future work in more detail.

\subsection{Generalization to Multiple sensitive attributes}
\label{appendix:multiple_sen}
\method{} can be extended to tasks involving multiple sensitive attributes by defining a sensitive group as a combination of sensitive attributes. For instance, recall the loss for EO in a single sensitive attribute is $\frac{1}{|\mathbb{Y}| |\mathbb{Z}|} \sum_{y \in \mathbb{Y}, z \in \mathbb{Z}} |\tilde{\ell}(f_{\theta}, G_{y, z}) - \tilde{\ell}(f_{\theta}, G_{y})|$. This definition can be extended to the case of multiple sensitive attributes as $\frac{1}{|\mathbb{Y}| |\mathbb{Z}_1||\mathbb{Z}_2|} \sum_{y \in \mathbb{Y}, z_1 \in \mathbb{Z}_1, z_2 \in \mathbb{Z}_2} |\tilde{\ell}(f_{\theta}, G_{y, z_1, z_2}) - \tilde{\ell}(f_{\theta}, G_{y})|$. The new definition for multiple sensitive attributes allows the overall optimization problem to optimize both sensitive attributes simultaneously.
The design above can also help prevent `fairness gerrymandering'~\citep{kearns2018preventing}, a situation where fairness is superficially achieved across multiple groups, but specific individuals or subgroups within those groups are systematically disadvantaged. This is achieved by minimizing all combinations of subgroups, thereby disrupting the potential for unfair prediction based on certain attribute combinations. However, having multiple loss functions may increase the complexity of optimization, and a more advanced loss function may need to be designed for multiple sensitive attributes.

\subsection{Further Reduction of Computational Overhead}
\label{appendix:reduce_overhead}
The scalability of \method{} can be further improved for large datasets, particularly when the quadratic computational complexity becomes a bottleneck. Continuing from Sec.~\ref{subsec:runtime}, clustering current data and assigning cluster-wise weight reduce LP variables, lowering computational complexity. Mini-batch K-means~\citep{sculley2010web} is efficient for extremely large datasets where the complexity is O(BKDI) (B: batch size, K: number of clusters, D: data dimension, I: number of iterations).

Beyond clustering, another approach is reducing the complexity of LP using a first-order method, which results in a complexity of O(ND) (N: number of current task samples, D: data dimension). CPLEX solves LP based on the simplex method~\citep{DBLP:journals/ior/Bixby02}, which is accurate, but computationally intensive due to Hessian calculations. First-order methods, by contrast, rely only on gradient information and thus offer better scalability for large datasets~\citep{beck2017first}. Moreover, these two approaches---clustering and first-order optimization---are orthogonal and can be combined to further improve \method{}'s scalability on extremely large datasets.

\subsection{Possible Strategies Under Data Drift}
\label{appendix:data_drift}
Data drift can be categorized into two types: virtual drift and concept drift~\citep{gama2014survey}.
Virtual drift refers to changes in the data distribution that do not alter the decision boundary. In such cases, model performance may slightly degrade, but \method{} can eventually adapt to the new distribution while preserving previously learned knowledge. This adaptation is possible because \method{} adjusts weights based on gradients derived from both the data and the model, allowing it to maintain flexibility and resilience against mild distributional changes.
On the other hand, concept drift occurs when the decision boundary shifts. In class-incremental learning, such shifts can be interpreted as the noise introduced into the true label ($y$) prior to the concept drift, thereby complicating the fairness maintenance. To address this, combining \method{} with robust training methods that are effective under noisy label conditions could be a potential direction. This integration would enhance both fairness and robustness, ensuring the model's performance despite shifts in the decision boundary.}{}

\end{document}